\documentclass[jair,twoside,11pt,theapa]{article}
\usepackage{jair, theapa, rawfonts}

\ShortHeadings{Learning Temporal Causal Sequence Relationships from Real-Time Time-Series}
{Antonio A. Bruto da Costa \& Pallab Dasgupta}
\usepackage{amsfonts}
\usepackage{amssymb}
\usepackage{amsmath}
\usepackage{amsthm}
\usepackage{booktabs}
\usepackage[ruled,vlined,lined,linesnumbered]{algorithm2e} % For algorithms

\SetAlFnt{\small}
\SetAlCapFnt{\small}
\SetAlCapNameFnt{\small}
\SetAlCapHSkip{0pt}
\usepackage{epsfig}
\usepackage{epstopdf}
\usepackage{tcolorbox}
\usepackage{fancyvrb}
\usepackage{dsfont}
\usepackage{array} % For tables and arrays - \arraybackslash
\usepackage{subcaption} % For subfigures
\usepackage{enumitem}
\usepackage{url}
\newtheorem{example}{Example}

\newtheorem{definition}{Definition}
\newtheorem{lemma}{Lemma}
\newtheorem{theorem}{Theorem}
\newtheorem{observation}{Observation}
\newtheorem{proposition}{Proposition}

\newenvironment{mycenter}[1][\topsep]
{\setlength{\topsep}{#1}\par\kern\topsep\centering}% \begin{mycenter}[<len>]
{\par\kern\topsep}% \end{mycenter}

\newcommand{\newchange}[1] {{\color{black}{{#1}}}}
\newcommand{\change}[1] {{\color{black}{{#1}}}}

\begin{document}

\setlength{\abovedisplayskip}{4pt}
\setlength{\belowdisplayskip}{4pt}
\setlength{\abovedisplayshortskip}{4pt}
\setlength{\belowdisplayshortskip}{4pt}

\title{Learning Temporal Causal Sequence Relationships \\from Real-Time Time-Series}

\author{\name Antonio Anastasio Bruto da Costa \email antonio@iitkgp.ac.in\\
		\name Pallab Dasgupta \email pallab@cse.iitkgp.ac.in \\
		\addr Dept. of Computer Science\\Indian Institute of Technology Kharagpur\\
		Kharagpur, West Bengal, India 721302}

% For research notes, remove the comment character in the line below.
% \researchnote

\maketitle

%% Group authors per affiliation:

%% or include affiliations in footnotes:
%\author[mymainaddress,mysecondaryaddress]{Elsevier Inc}
%\ead[url]{www.elsevier.com}

%\author[mysecondaryaddress]{Global Customer Service\corref{mycorrespondingauthor}}
%\cortext[mycorrespondingauthor]{Corresponding author}
%\ead{support@elsevier.com}

%\address[mymainaddress]{1600 John F Kennedy Boulevard, Philadelphia}
%\address[mysecondaryaddress]{360 Park Avenue South, New York}

\begin{abstract}
We aim to mine temporal causal sequences that 
\change{explain observed events (consequents) in time-series traces.} 
% Existing work on real-time time-series has focused on classification, summarization, and pattern extraction. 
Causal explanations of key events in a time-series have applications in design debugging, 
anomaly detection, planning, root-cause analysis and many more. We make use of decision 
trees and interval arithmetic to mine sequences that explain defining events in the 
time-series. We propose modified decision tree construction metrics to handle the 
non-determinism introduced by the temporal dimension. The mined sequences are expressed in a 
readable temporal logic language that is easy to interpret. The application of the
proposed methodology is illustrated through various \change{examples.}%case studies.
\end{abstract}

%\begin{keyword}
%Time-series \sep mining \sep properties \ decision trees \sep interval arithmetic
%\MSC[2010] 68T05 \sep 68T27\sep 
%\end{keyword}

%\linenumbers

\section{Introduction}

%Re-write this

%Position papers
%Generalizing to other domains is a problem. Properties can be used for deduction
%Causal Relationships - learning bayes networks [want the assertions to be such that they don't fail on the data, which causes what is not explicitly indicated, a variety of networks can be constructed for the same data, past influences the future - time is a factor, introduces a partial order on the events] , learning assertions [not discrete time or discrete variable domains].
%We can use  the miner to construct a hierarchy of causal relations
%Product of Time * Event space for constructing a Bayes Network 

%\pd{This is a comment}

This article presents an approach for learning causal sequence relationships, in the form of temporal properties, from data. Most realistic causal relationships are timed sequences of events that affect the truth of a target event. For example:
\textit{``A car crashes into another. The cause was that there was a time-point 7~sec to 
8~sec before the crash at which the two cars had a relative velocity of 70~kmph and a longitudinal distance of 6~m, and the leading car braked sharply."} %If the car is at a speed of 70kmph and brakes, it covers between 50m to 100m within the next 8sec before coming to a complete stop."}
In this relationship, the cause is the relative velocity, the distance between the cars, and the braking of the lead car. Such timing relationships can be expressed in logic languages such as Linear Temporal Logic~\cite{Pnueli77} for discrete event systems and Signal Temporal Logic~\cite{stl2004} for continuous and hybrid systems.  Temporal logic properties are also extensively used in the semiconductor industry, with language standards such as SystemVerilog Assertions (SVA)~\cite{sva} and Property Specification Language (PSL)~\cite{psl}.
The notion of {\em sequence expressions} in SVA, which is very natural for expressing
temporal sequences of events, is one of the primary features responsible for the
popularity of the language in industrial practice. In this article we use a logic language
inspired from SVA, which allows us to express real-valued timing relations between 
predicates.
Our choice is partially influenced by our objective of mining assertions from circuit simulation traces, but also due to the applicability of the semantics of {\em sequence expressions} in other time-series domains. For instance, in our language, the causal expression for a crash may be captured as:
\begin{mycenter}[0pt]
{\tt rspeed >= 70 \&\& brake \&\& ld <= \change{6} |-> \#\#[7:8] crash }
\end{mycenter}
In this expression, {\tt brake} is a proposition, {\tt ld} and {\tt rspeed} are real-valued variables representing the longitudinal distance and relative velocity respectively, and
{\tt rspeed>=70} \change{and {\tt ld<=6} are {\em predicates over real variables} (PORVs).} 

Methods for learning causal relationships from data have been studied extensively and its importance is well established~\cite{pearl1995,Pearl2000,Pearl2009causal,Evans2018,Guo2019,Pearl2019}. Most recently, the case was made for the need to build a framework for learning causal relationships using logic languages, to have explanations for predictions or recommendations and to understand cause-effect patterns from data, the latter being a nec\change{e}ssary ingredient for achieving human level cognition~\cite{Pearl2019}. Causal learning has applications in many areas~\cite{Guo2019} including medical sciences, economics, education, environmental health and \change{epidemiology}.
%epidemology. 
We aim to learn causal relationships as temporal properties, of the form $\alpha \Rightarrow \beta$, having a defined syntax and semantics. $\alpha$ is a sequence of Boolean expressions, with adjacent expressions separated %in time
from one another \change{by time intervals}. $\beta$ is a predicate whose cause {\change is} to be learned. Properties of this nature can be used for deduction. This then enables us to learn complex hierarchical causal relationships. 

Existing learning approaches, such as neural networks, require large amounts of training data, and their results are not easily explainable. Furthermore, it is difficult to generalize the same network structure to a variety of domains. Alternately, using Bayes networks to capture causal relationships, poses ambiguities which are challenging to deal with 
%during
\change{when} reasoning over time. Bayes \change{networks} do not explicitly indicate what events cause what; a variety of networks can be constructed for the same data. 
Also, time, which is a factor in our learning, introduces a natural partial order among the events.
%, that is, past events may influence future events. 
Furthermore, the data is a time-series representing the continuous evolution of variables over real-time, and therefore in our setting time is assumed to be dense by default,
and the variables over which the \change{predicates} are defined are assumed to be continuous in general.

In this article, we use decision trees to achieve our goal of learning causal relationships. 
Traditional decision tree structures, as they exist in standard 
texts~\cite{Quinlan1986,Quinlan1993,Mitchell1997}, deal with enumerated value domains of
the variables. Continuous domains can be partitioned into a finite set of options by using 
\change{predicates}. For example, we may have the \change{predicates}, {\tt speed<30}, 
{\tt 30<=speed<=100},
and {\tt speed>100}, to define {\tt low\_speed}, {\tt moderate\_speed}, and
{\tt high\_speed}. In the temporal setting, {\tt speed}, varies with time and hence the
truths of these predicates change with time. When learning causal sequences which connect events in multiple time-worlds, the following fundamental challenges arise:
\begin{itemize}[topsep=0em]\setlength\itemsep{0em}
\item A data-point is the state of a single time-world. For a time-series over real-time, in theory, there are infinite time-worlds in a finite time window.
\item The influence of a \change{predicate} on the truth of the consequent changes with the time
    separation between them. In other words, the same \change{predicate} may contribute to the
    consequent being true in some time-worlds and false in some other time-worlds. 
    Relative to the consequent, it means that different past time-worlds
    contribute to its truth in potentially conflicting ways.
\item The influence of a time-world state on a future event (namely, the consequent) 
    not only depends on the truth of the predicates in that time-world, but also on the
    truth of the predicates in past time-worlds.
\end{itemize}
The main challenges may therefore be summarized in terms of two questions, namely:
\begin{enumerate}
    \item Finding the predicates which influence the consequent, and
    \item For these predicates, finding the time window separating the predicates and
        the consequent, such that the sequence of predicates guarantee the consequent. 
\end{enumerate}
The challenges above arise due to the non-deterministic characteristic of temporal properties such as the one described earlier. In a temporal property, past time-worlds influence future time-worlds, and due to the dense nature of the real-time domain, these time-world associations are not always an exact association, and non-determinism \textit{in time} plays a crucial role in representing these relationships. 
% As a result of these challenges, standard decision tree learning metrics can not be directly applied.

\begin{figure}[t]
\centering
\includegraphics[scale=1]{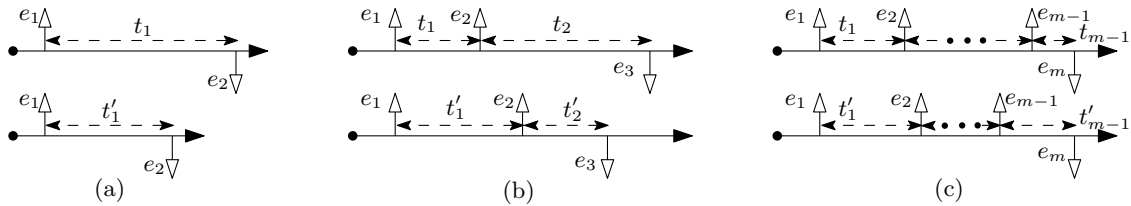}
\caption{Evidence of time-separations between events in different time-worlds.}\label{fig:non-deterministic-association}
\end{figure}

In Fig\change{.}~\ref{fig:non-deterministic-association}, various temporal associations are described. In Fig.~\ref{fig:non-deterministic-association}(a), in one instance, event $e_1$ is separated from event $e_2$ by $t_1$ time units, and in another by $t'_1$ time units. 
The system from which the traces have been taken may potentially admit infinite variations of time separation between $e_1$ and $e_2$ within a dense time interval. For learning meaningful \change{associations} %assertions
we need to generalize from the discrete time separation instances shown in the time-series to time intervals.
%Many, possibly infinite, such evidences of time separation between $e_1$ and $e_2$ may exist in the data. 
Fig.~\ref{fig:non-deterministic-association}(b) depicts a similar situation with three events, where the separation between $e_1$ and $e_2$ can vary, and the separation between $e_2$ and event $e_3$ may also vary. Fig.~\ref{fig:non-deterministic-association}(c) generalizes this.

The primary contributions of this article are as follows:
\begin{itemize}[topsep=0em]\setlength\itemsep{0em}
	\item We discuss the problems associated with using the standard decision tree learning
	framework for learning causal sequence relationships across multiple time-worlds. We show how this is attributed to the metrics used in building the decision tree.
	\item We adapt the measures of entropy, to account for time-worlds that may non-determi\-nistically be classified into multiple classes.
	\item We also adapt the information gain metric to account for time-world states that may be present across multiple sibling nodes in the \change{decision} tree.
	\item We propose a logic language for representing causal sequences. \change{The semantics of the language are compatible with standard ranking measures for assessing learned properties. We use measures of support and correlation to measure the quality of the learned properties. These measures also} give insights into the data.
	\item We propose a decision tree construction that uses the adapted metrics to learn causal sequences across multiple time-worlds. We provide a method to translate the associations learned into properties, in a logic language that can be then used for reasoning.

\end{itemize}
%extends the technique presented therein with the aim of providing higher trace coverage and finer control over the quality of assertions generated. The assertions generated are in a format similar to that of the widely used language SystemVerilog Assertions (SVA)~\cite{sva}, where sequence expressions form the key to expressing the chain of events representing a cause associated with a target effect. Unlike SVA which is famously used to write assertions over Boolean signals for clocked systems, here the artifacts in the sequence expression are predicates over real variables (PORVs)~\cite{STL}, and the delays seperating individual events/PORVs are measured in real-time versus being in clock cycles as in SVA.

%The decision tree learning approach presented in this article is applied to various case studies and the mining process is validated through experiments. On small combinatorial circuits the miner was able to mine Boolean properties that cover the entire specification. 
The theory developed in this article has been implemented in a tool called the \textit{Prefix Sequence Inference Miner} (PSI-Miner), available at \url{https://github.com/antoniobruto/PSIMiner}. 
\change{The article is organized as follows. Section~\ref{sec:motivation} outlines the problem statement with a motivating example and presents the formal language for representing the properties mined. 
Section~\ref{sec:decisionPrelims} presents definitions for various structures and metrics used throughout this article.  
In Section~\ref{sec:concurrent}, we extend the standard decision tree metrics to incorporate time into the learning process and develop an algorithm for mining temporal sequence expressions to derive explanations for a given target event. Section~\ref{sec:ranking} introduces ranking metrics for properties.
Section~\ref{sec:stoppingPruningOverfitting} discusses measures employed to prevent over-fitting using various stopping conditions and pruning methods. Section~\ref{sec:multipleTraces} describes how we extend structures and metrics to operate over multiple time-series simultaneously. In Section~\ref{sec:experiments} we demonstrate the utility of the methodology through select case studies. 
Section~\ref{sec:related} discusses related work.
Section~\ref{sec:summary} concludes and summarizes the article.}

\section{Mining Explanation as Prefix Sequences}\label{sec:motivation}

\begin{figure}[t]
	\centering
	\includegraphics[height=2.8in]{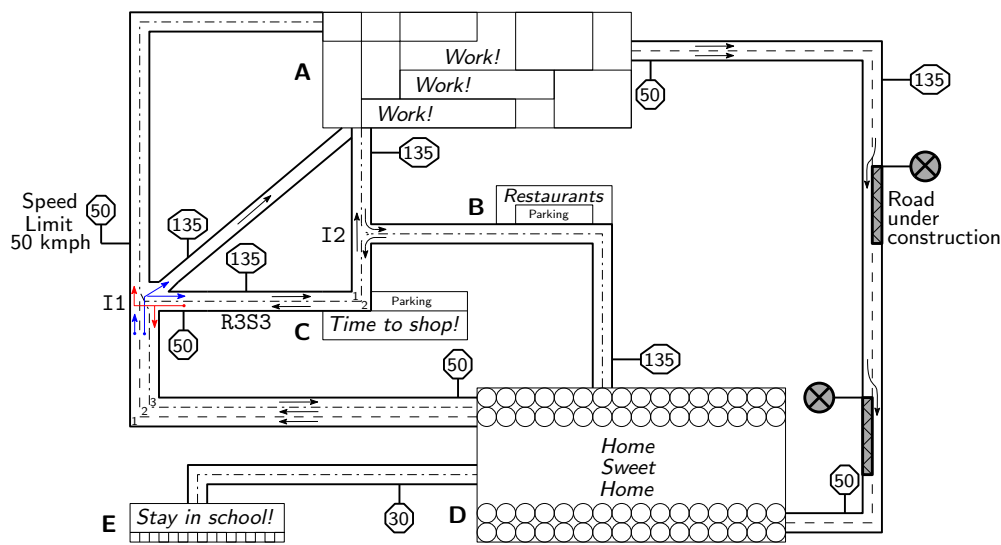}%{traffic-labeled.eps}%[width=2.8in, height=2.2in]{traffic-labeled.eps}
	\caption{Vehicle road-map: Routes with demarcations for direction and speed limits.}
	\label{fig:traffic}
\end{figure}

\change{We start with a motivating example. Figure~\ref{fig:traffic} shows a map of town X, depicting roadways for vehicular movement in two dimensions. Vehicles are tagged with GPS devices to monitor their movements. The data contains patterns describing routes vehicles follow and their speed. Congestion and delays are reported in various parts of the town and one wishes to determine the causes that lead to a delay in reaching the office.}

\change{We label those states as {\em delayed} from which the delay in reaching office is inevitable. Obviously, the cause for a delay is a sequence of movement events that lead to a delayed state. Once in a delayed state, the vehicle is always delayed. 
Some events may be common to all vehicles reaching the office, and such events
need to be separated out from those that contribute to the delay. Also, since
the traffic pattern evolves with the time of the day, the time delays separating 
the relevant events have a significant role in capturing the causal sequence
responsible for the delay.}

% Town-X wishes to use movement data of vehicles, both those involved in delays and those reporting non-delay trajectories to learn why delays occur, to then plan to make moving within the town delay-free.
%Prior to verification, engineers do not always have a complete specification of the system being designed, nor is it easy to determine the cause of failures or bugs when they occur. The problem we wish to address is to mine as many explanation patterns from the traces, to uncover potential reasons that could improve an engineer's understanding of why (or when) some event occurs.
\change{Mining causes from the data leads to the discovery of the potential sequence of events leading to a delay. One such event sequence is as follows:}
\change{
%\begin{tcolorbox}
\begin{mycenter}[0em]
%\begin{Verbatim}
%{\tt 22<=x<=24 \&\& 15<=y<=20 \#\#[0:0.369] v>100 \#\#[2.2:2.3] !(route==1) |-> crash}
{\tt I1 \#\#[0:40] !LANE2 \#\#[0:5] !LANE1 \&\& R3S3 |-> \#\#[0:30] DELAY}
%\end{Verbatim}
\end{mycenter}
%\end{tcolorbox}
}

%The formula reads as ``If the car is in the region $x\in[22:24]$ and $y \in [15:20]$ and if within the following 0h22m8s the velocity is above 100kmph and thereafter in the next 2h12m to 2h18m if the route is not Route-1 a crash occurs''. This is an example of a property mined from the data depicted in Figure~\ref{fig:traffic}. Observe that the formula describes a sequence of events and very finely grained time delays between the events. Mined patterns are always in the form of such sequences. The events may be mined or provided as inputs using domain knowledge. Time-delays are computed from the decision tree generated.

\change{The formula reads as, ``After being at Intersection-1 ({\tt I1}), if the vehicle is not in Lane-2 ({\tt !LANE2}) within the next 40 minutes, and is on road segment {\tt R3S3} but not in Lane-1 ({\tt !LANE1}) within the next 5 minutes, then within the next 30 minutes, the vehicle is delayed.'' On further examination, using the layout of roads, the town discovers that the vehicle was on the wrong lane while making the turn into {\tt R3S3} and ended up in the wrong lane in a high-speed zone. We call the language for describing the above property as
the \textit{Prefix Sequence Inferencing} language (PSI-L).}

\begin{comment}
The authorities also observe a different formula mined from the data:
\begin{mycenter}[0em]
{\tt !I1 \&\& R3S3 \#\#[0:15] LANE1 |-> \#\#[0:60] !DELAY}
\end{mycenter}
\noindent which reads as, ``When on {\tt R3S3}, if within the following 15mins the vehicle moves into Lane-1 then there is no delay experiences within the next 60mins.'' They also are aware that it takes less than 60mins to reach destinations reachable in the direction of Lane-1. The authorities use this information to improve road markers and place appropriate divides to guide vehicles into correct lanes and avoid delays.
\end{comment}

%\section{\change{Problem of Mining Prefixes}}
%Explaining using Sequences: Prefix Sequence Inferencing Language}
\label{sec:prefixMiningProblem}
\label{sec:PSI-L}\label{sec:psi-l_Syntax}\label{sec:psi-l_MatchSemantics}
%We express explanations in the form of a sequence of events or predicates (possibly over real-valued variables). In addition to an ordering between events, the timing between adjacent events is key. An explanation for a target event is observed as a prefix to the target. 

%\subsection{PSI-L Syntax}\label{sec:psi-l_Syntax}
%The language used to describe prefix sequences inferred from time-series data has the following general syntax:
\change{A prefix sequence inference (alternatively, a PSI-L formula or \newchange{PSI-L property}) has the general syntax,}
%\begin{mycenter}
%\begin{tabular}{r@{~~} c@{~~} l}
$\mathcal{S}${\tt |->}$E$;
%& or & $\mathcal{S}$ {\tt |->} $\tau_0 ~ E$\\
%\end{tabular}
%\end{mycenter}
%\noindent
where, $\mathcal{S}$ is a prefix sequence of the form $s_n ~\tau_{n} ~s_{n-1} ~\tau_{n-1}$ $\ldots \tau_{1}~s_0$, also known as a sequence expression.

%\pd{There are two $\tau_1$s here -- one before E, and one before $s_0$. Please correct in a way that is consistent with the rest of the paper.}

\change{A time interval $\tau_i$ is} of the form $[a:b]$, $a,b\in \mathbb{R}^{\geq 0}$, $a\leq b$, and each $s_i$ is a Boolean expression of \change{predicates.}
%PORVs 
%and events. 
The length of the sequence expression is $n$ (having at most $n$ \change{time intervals}).
%non-temporal sub-expressions).
\change{A special case arises when $s_0=$ \textit{true} and $\tau_1\neq \emptyset$. In such a case the prefix sequence inference is treated as the expression $s_n ~\tau_{n} ~s_{n-1} ~\tau_{n-1}$ $\ldots \tau_{2}~s_1~${\tt |->}$~\tau_1~E$.}

%\pd{But $s_0$ to $s_n$ consists of $n+1$ non-temporal sub-expressions, so not sure whether the above is correct.}

The 
\change{consequent}
$E$ in a PSI-L formula is assumed to be 
%known
\change{given}. It is in the context of $E$ that \change{the antecedent} $\mathcal{S}$ is learned. $E$ is called the {\em target} of the PSI-L \newchange{formula}. The notation $\mathcal{S}_i^j$ is used to denote the expression $s_j ~\tau_{j} \ldots \tau_{i+1}~s_i$, $0\leq i \leq j \leq n$. In general, $\mathcal{S} \equiv \mathcal{S}_0^n$.

%Note that here, an increase in the index indicates going backward in time, whereas for traces increasing the index indicates a movement forward in time.

%\subsection{PSI-L Semantics for Traces}\label{sec:psi-l_MatchSemantics}

For variable set $V$, the set $\mathbb{D} = \mathbb{R}^{\geq 0} \times \mathbb{R}^{|V|}$ is the domain of valuations of timestamps and variables in $V$. A data point is a tuple $(t,\eta) \in \mathbb{D}$, $t \in \mathbb{R}^{\geq 0}$ and $\eta \in \mathbb{R}^{|V|}$. The value of a variable $x \in V$ at the data point $(t,\eta)$ is denoted by $\eta[x]$. Boolean and real-valued variables are treated in the same way in the implementation. A Boolean value at a data point is either {\tt 1}  for {\em true} or {\tt 0} for {\em false}, and \{$0,1$\} $\subset \mathbb{R}$. \newchange{We use the notation $\eta \models s$ to denote satisfaction of a Boolean expression $s$ by a valuation $\eta$ at a data point.}

\begin{definition}\label{def:hybridTrace}
	\textbf{Time-Series (Trace):}
	A trace $\mathcal{T}$ is \newchange{a finite} ordered list of tuples $(t_1,\eta_1)$, $(t_2,\eta_2)$, $(t_3,\eta_3)\ldots(t_d,\eta_d)$, $\forall_{i \in \mathbb{N}_{d-1}} t_i < t_{i+1}$. The length of $\mathcal{T}$, the number of tuples in $\mathcal{T}$, expressed as $|\mathcal{T}|$, is $d$. The temporal length of $\mathcal{T}$, denoted $||\mathcal{T}||$, is $t_d-t_1$.
	
	$\mathcal{T}(i)$, $i\in\mathbb{Z}^{>0}$ denotes the $i^{th}$ data point $(t_i,\eta_i)$ in trace $\mathcal{T}$.
	
	A sub-trace $\mathcal{T}_i^j$ of $~\mathcal{T}$ is defined as the ordered list $(t_i,\eta_i)$, $(t_{i+1},\eta_{i+1})$, $\ldots(t_j,\eta_j)$; $i,j \in \mathbb{N}_{d}$ and $i\leq j$.
	\qed
\end{definition}

\begin{definition}\label{def:seqExprMatch}
\change{
\textbf{Match of a Sequence Expression and a PSI-L formula: }
The sequence expression $S_l^m ::= s_m ~\tau_{m}$ $s_{m-1} ~\tau_{m-1}$ $\ldots \tau_{l+1}~s_l$ has a match at $\mathcal{T}(j)$ in sub-trace $\mathcal{T}_i^j$ of trace $\mathcal{T}$, denoted $\mathcal{T}_i^j {\models} S_{l}^{m}$ iff:}
%sub-trace $\mathcal{T}_i^j$, $i\leq j$, of $\mathcal{T}$ matches  at $\mathcal{T}(j)$, denoted $\mathcal{T}_i^j \overset{m}{\vDash} S_{l}^{m}$

%\pd{Previously you defined:$\mathcal{S}_i^j$ is used to denote the expression $s_j ~\tau_{j} \ldots \tau_{i+1}~s_i$, $0\leq i \leq j \leq n$. But now you say $S_l^m ::= s_m ~\tau_{m}$ $s_{m-1} ~\tau_{m-1}$ $\ldots \tau_{l-1}~s_l$. Note the discrepancy in notation. Perhaps you mean:}

%\pd{The sequence expression $S_l^m ::= s_m ~\tau_{m}$ $s_{m-1} ~\tau_{m-1}$ $\ldots \tau_{l+1}~s_l$ has a match at $\mathcal{T}(j)$ in sub-trace $\mathcal{T}_i^j$ of trace $\mathcal{T}$, denoted $\mathcal{T}_i^j {\models} S_{l}^{m}$ iff:}
\change{
\begin{itemize}[topsep=0.1em]\setlength\itemsep{0.1em}
	\item $\eta_i \models s_m$, $\eta_j \models s_l$
	\item $\exists_{i \leq k \leq j} ~\mathcal{T}_{k}^j \models S_{l}^{m-1}$ $\bigwedge$ $t_k - t_i \in \tau_{m}$ if $(m-1)>l$
	%\pd{You need to check this after cleaning up theprevious issues.}
\end{itemize}}

\change{A PSI-L formula can have multiple matches in $\mathcal{T}$. The PSI formula $S${\tt|->}$E$ has a match in trace $\mathcal{T}$ at $\mathcal{T}(j)$ iff $\exists_{1\leq i\leq j}$ $\mathcal{T}_i^j \models S$ and $\eta_j \models E$.}
%$\exists_{i}~ {i \leq k}$, $\mathcal{T}_i^k \models S_0^m$, $t_k - t_j \in \tau_0$, and $\mathcal{T}(k) \models E$. The sub-trace $\mathcal{T}_i^j$ 
%\pd{(or should it be $\mathcal{T}_i^k$ ?)}
%is then a witness to the PSI-L property in trace $\mathcal{T}$.
\qed
\end{definition}
\begin{figure}[t]
    \centering
    \includegraphics[width=\textwidth]{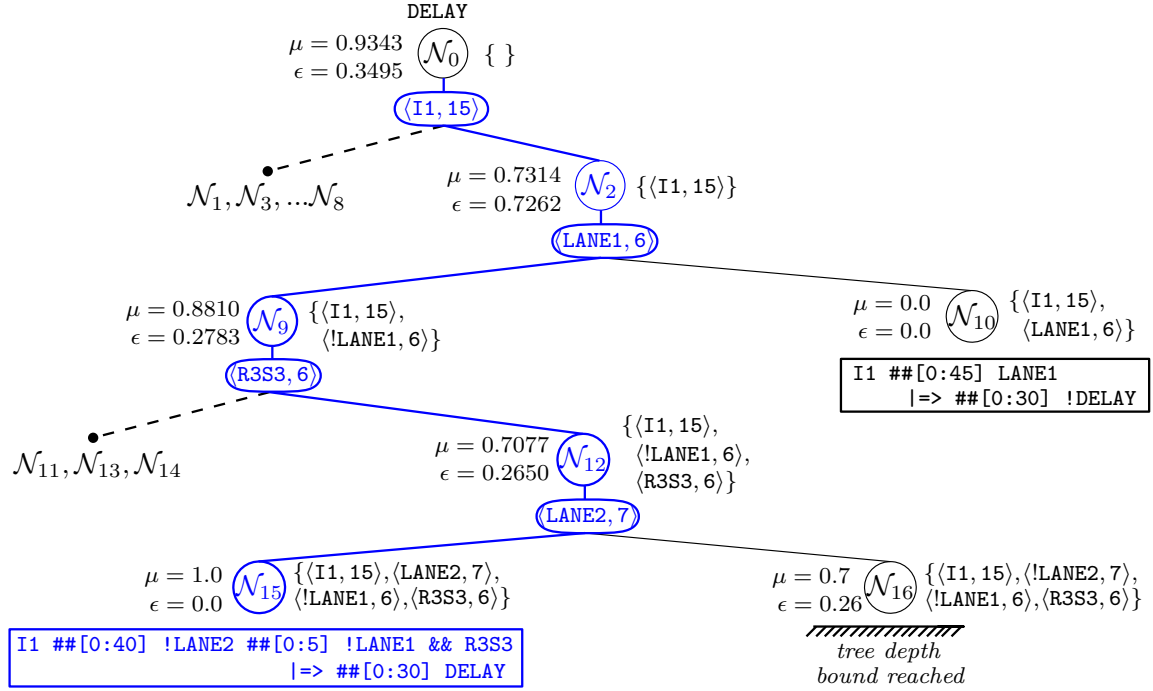}
    \caption{Decision Tree for mining causes of the delayed state {\tt DELAY} for Town-X. A node is represented as a circle, a decision of the form $\langle predicate, bucket \rangle$ is an oval, decisions constraining a node are indicated to the right of each node within braces, while metrics for the node are on its right.}
    \label{fig:decisionTree}
\end{figure}

%\section{Problem Definition}
\change{We consider a given predicate alphabet $\mathbb{P}$ and a predicate, $E \in \mathbb{P}$, that needs explanation, called {\em the target}. %, to be given as a PORV or event, $E$. 
Given a \newchange{finite} set of traces $\mathbb{T}$ %, trace $\mathcal{T}$ 
and a target $E$, we wish to find various prefix sequences that causally 
determine the truth of the target $E$. Each such prefix sequence produces 
a PSI-L formula, which is valid over all traces in $\mathbb{T}$.} %$\mathcal{T}$.
	
We assume a bound $n\in \mathbb{N}$, $n\geq 0$, on the length of the prefix.
\change{We also use a parameter, $k\in \mathbb{R}^{\geq 0}$, called 
{\em delay resolution,} representing an initial upper bound on the time separating 
adjacent predicates in the prefix sequence.}
%We also take as input a resolution $k\in \mathbb{R}$ as a maximum delay between sub-expression in a prefix sequence, that is,
\change{It is assumed that initially every time interval in the prefix sequence, $\tau_i = [0:k]$, $0\leq i\leq n$. The time intervals are refined as the PSI-L property is learned.}

\change{Before developing the theory behind mining PSI-L properties, we convey a high
level intuitive outline of the approach, which we believe will help in the
understanding of the notations and definitions that follow.}

\change{Given the target, $E$ (namely, the consequent), and the values of $n$ and $k$,
we create a template of the following form:
\begin{mycenter}[0em]
{\tt
 $\sqcup_n$ \#\#[0:k]  $\sqcup_{n-1}$ \#\#[0:k] $\ldots$ \#\#[0:k] $\sqcup_1$ \#\#[0:k] $\sqcup_0$ |-> E }
\end{mycenter}
where each $\sqcup_i$ is called a {\em bucket} and represents a placeholder
for a Boolean combination of predicates (denoted $s_i$) from $\mathbb{P}$. \newchange{Initially, therefore, given values for $n$ and $k$, the template describes a sequence of events spread over a time span of, at most, $n \times k$ time units before $E$.} Our algorithm uses  decision trees and works
on this template in two cohesive ways, namely:
\begin{enumerate}[topsep=0.1em]\setlength\itemsep{0.1em}
    \item It uses novel metrics based on information gain to choose the 
        combinations of predicates that go into a bucket. Not all buckets
        may be populated at the end -- the intervals preceding and succeeding 
        empty buckets are merged.
    \item The delay intervals between populated buckets are narrowed to optimize
        the influence of the sequence expression on the target.
\end{enumerate}
We use metrics based on interval arithmetic to compute the influence of predicates
on the target across time intervals. The arithmetic is elucidated by 
hypothetically moving the target backwards in time, so that information gain
metrics can be computed on individual time worlds.}

\change{An example of a decision tree produced by our algorithm for the property described earlier for the delayed state is shown in Figure~\ref{fig:decisionTree}. The  path in the decision tree leading to the node at which a property is found is indicated using bold blue lines. A node in the decision tree is named as $\mathcal{N}$ and given an index. When the error $\epsilon$ at a node is non-zero, a choice of predicate and bucket is made and two child nodes are generated. If the decision tree depth bound is reached, no more decisions can be made. Some nodes, such as $\mathcal{N}_{10}$ and $\mathcal{N}_{15}$ are nodes with zero error. For such nodes, a property can be constructed using predicates and bucket positions labeling the path from the node to the root. A predicate is false on the left branch and true on the right branch. For instance, the property described earlier is generated at $\mathcal{N}_{15}$, and consists of the buckets $\sqcup_{15} = \{{\tt I1}\}$, $\sqcup_{7} = \{{\tt LANE2}\}$ and $\sqcup_{6} = \{{\tt !LANE1, R3S3}\}$. A delay resolution of $5$ is used here. The delays between buckets is computed using this delay resolution and the bucket indexes. Refinement of delays may be possible in some instances, and we explore this later.}

%The user may provide a set of known predicates which form the predicate alphabet $\mathbb{P}$ used by the mining algorithm. The set $\mathbb{P}$ may also be extracted before hand using existing techniques~\cite{Guralnik1999} to learn \textit{interesting} events in the trace. Parameterized predicates may also be learned using parameter optimization techniques, however this is not the focus of this article. We assume, for now, for ease of understanding, that the set $\mathbb{P}$ is given.

\section{PSI-Arithmetic and Preliminary Definitions}\label{sec:decisionPrelims}

A summary of the methodology for mining prefix sequences is depicted in Figure~\ref{fig:toolflow}. 
Initially, \change{an} event $E$ (the consequent) is presented as the target. 
Prefix sequences that appear to {\em cause} $E$ are to be mined (these are the
potential antecedents). The antecedent and the consequent together define the
mined property. It is important to note that the non-existence of a 
counter-example in the data is a necessary but not sufficient condition for
a property to be mined.

The truth intervals of a predicate in a trace define a Boolean trace. The given
trace is initially replaced by the Boolean traces corresponding to 
the predicate alphabet $\mathbb{P}$.

We use interval arithmetic to represent and analyze truths of predicates over 
dense-time. We handle time arithmetically, instead of as a series of samples, 
making the methodology robust to variations in the mechanism used for sampling 
the data. This also allows us to parameterize time delays between sequenced 
\change{events}, and compute the trade-offs involved while varying the temporal 
positions of the events. \change{All definitions are with respect to a single 
trace for ease of explanation, but the methodology easily extends to dealing 
with multiple traces, as discussed later in Section~\ref{sec:multipleTraces}.}

\begin{definition}\label{def:intervalSet}
	\textbf{Interval Set of a predicate $P$ for trace $\mathcal{T}$: }
	The Interval Set of a predicate $P$ for trace $\mathcal{T}$, $\mathcal{I}_{\mathcal{T}}(P)$, is the set of all non-overlapping maximal time intervals, $[a,b); ~a,b\in\mathbb{R}_{\geq 0};~a< b$, in $\mathcal{T}$ where $P$ is true.	
	The interval $[t_i:t_j) \in \mathcal{I}_\mathcal{T}(P)$ iff $\forall_{i\leq k<j} \mathcal{T}(k) \models P$. 
	The length of the interval set $\mathcal{I}_{\mathcal{T}}(P)$, denoted $|\mathcal{I}_{\mathcal{T}}(P)|$ is defined as,\linebreak 
	$|\mathcal{I}_{\mathcal{T}}(P)| = \Sigma_{\forall I=[a,b) \in \mathcal{I}_{\mathcal{T}}(P)} (b-a)$ 
	\newchange{For an interval $I=[a,b)$, the left and right values of the interval are denoted $l(I)$ and $r(I)$, denoting $a$ and $b$ respectively.}
\qed
\end{definition}

\begin{figure}[t]
	\centering
	\includegraphics[scale=1.2]{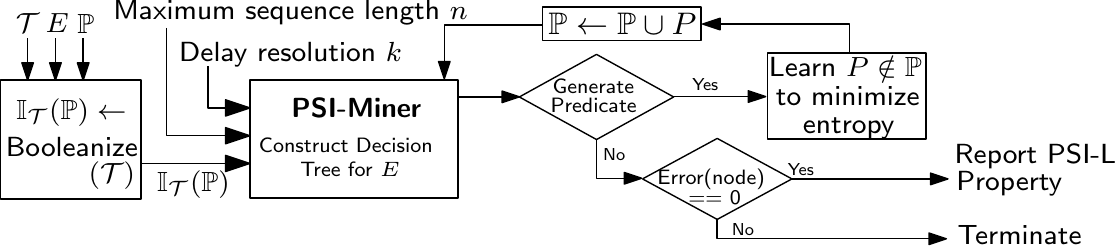}
	\caption{Prefix Sequence Property Mining Workflow}\label{fig:toolflow}
\end{figure}

%{\color{red}
%PD Comment: It is possibly better to avoid the following assumption at this place in the paper. For us the data is the holy grail. In a more appropriate place (perhaps where we actually look at the circuit domain test cases) we may write that the data may be made more accurate with respect to the predicate truths using constructs like cross events of Verilog-AMS.

%\noindent \textbf{\textit{Assumption}:} It is assumed that the trace is the result of a \textit{sufficiently} accurate sampling of the process under observation with respect to the choice of predicates in $\mathbb{P}$. This is possible to achieve during a simulation of a mixed-signal circuit. For more information, the interested reader may refer to Ref.~\cite{subhankar12a}.}

\begin{definition}\label{def:truthSet}
	\textbf{Truth Set for trace $\mathcal{T}$ and Predicate Set $\mathbb{P}$:}
	The Truth Set for $\mathbb{P}$ in trace $\mathcal{T}$, $\mathbb{I}_{\mathcal{T}}(\mathbb{P})$, is the set of all Interval Sets for the trace	$\mathcal{T}$ of all predicates $P \in \mathbb{P}$. $\mathbb{I}_{\mathcal{T}} = \{ \mathcal{I}_{\mathcal{T}}(P) | P \in \mathbb{P}\}$.\qed \label{def:TruthSet}
\end{definition}
\change{
Essentially, the trace $\mathcal{T}$ is translated into a \textit{Truth Set}, 
namely the set of all labeled interval sets for predicates in $\mathbb{P}$. 
The truth set acts a Booleanized abstraction of the trace $\mathcal{T}$ with
respect to $\mathbb{P}$.}

%{\color{red}
\change{Recall the following template of the mined properties as outlined in the previous section:
\begin{mycenter}[0em]
{\tt
 $\sqcup_n$ \#\#[0:k]  $\sqcup_{n-1}$ \#\#[0:k] $\ldots$ \#\#[0:k] $\sqcup_1$ \#\#[0:k] $\sqcup_0$ |-> E }
\end{mycenter}
where each $\sqcup_i$ is called a {\em bucket}.
We propose a decision tree learning methodology for mining prefix sequences.
Every path of the learned decision tree leads
to a true or false decision, representing the truth of the consequent. Each node
of the decision tree corresponds to a pair $\langle P, i\rangle$, namely a
chosen predicate and its position (the {\em bucket}) in the prefix sequence. 
Different branches correspond to different choices of predicates in different buckets. 
The accumulated choices along a path of the decision tree define a partial prefix 
sequence, where some of the buckets have been populated. These accumulated choices 
shall be referred to as a {\em constraint set}.%}
}
\begin{definition}\label{def:constraintSet}
	\textbf{Constraint Set: }
	A constraint is a pair, $\langle P, i\rangle$, consisting of a predicate, $P$, and 
	its position in the prefix-sequence, where $P\in \mathbb{P}$, $i\in[0:n]$, $n\in\mathbb{N}$. A constraint set $\mathcal{C}$ is a set of constraints at a node in 
	the decision tree obtained by accumulating the constraints at its ancestors in the
	tree. 
\qed
\end{definition}

\begin{definition}\label{def:bucket}
\textbf{Prefix-Bucket: } \change{For a constraint set $\mathcal{C}$, the prefix-bucket at position 
$i\in \mathbb{N}$, given as $\mathcal{B}_i(\mathcal{C})$, is the set of all predicates $P$, where
$\langle P, i\rangle \in \mathcal{C}$. The set of all buckets for a constraint set 
$\mathcal{C}$ is written as $\mathbb{B}(\mathcal{C})$ or simply $\mathbb{B}$ if constraint 
set $\mathcal{C}$ is known from context.}

\change{
The term prefix-bucket refers to the set of constraints in a bucket. When the constraint 
set $\mathcal{C}$ is known, we use the notation  $\mathcal{B}_i$ to mean  $\mathcal{B}_i(\mathcal{C})$. 
}

\change{The set of constraints in $\mathcal{C}$ define a partial prefix in PSI-L. In the prefix-sequence $s_n ~\tau_{n} ~s_{n-1} ~\tau_{n-1} \ldots \tau_{1}~s_0$, the sub-expression $s_i$, $0\leq i \leq n$, is formed by the conjunction of predicates in the bucket $\mathcal{B}_i(\mathcal{C})$. The partial prefix sequence formed from the constraint set $\mathcal{C}$ is denoted as $\mathcal{S}_\mathcal{C}$.
For a constraint set $\mathcal{C}$, the interval set for bucket $\mathcal{B}_i$, given as $\mathcal{I}_{\mathcal{T}}(\mathcal{B}_i)$, is the set of truth intervals where the constraints in $\mathcal{B}_i$ are all true.
}\qed
\end{definition}

\begin{comment}   computed as follows:	
\begin{equation}
\mathcal{I}_{\mathcal{T}}(\mathcal{B}_i(\mathcal{C})) = flatten\Biggl(~
\bigcup_{\substack{{I \in \mathcal{I}_{\mathcal{T}}(P), \langle P,i\rangle\in\mathcal{B}_i} \\ {I' \in \mathcal{I}_{\mathcal{T}}(Q), \langle Q,i\rangle\in\mathcal{B}_i}\\ P\neq Q } } \hspace*{0em}(I \cap I') \Biggr)
\label{eq:intervalSet}
\end{equation}	
where, $flatten(\mathcal{I})$ merges intervals that have an overlap. %For instance, let the interval set of $P$ be $\{[1.1:3),[5.2:5.6),[7.1:9.6)\}$, and let the interval set of $Q$ be $\{[0.5:3.4),[6:8.3)]\}
\end{comment}

The learning algorithm must place predicate and event constraints into various buckets. Some buckets may remain empty, resulting in the delays in the sequence that appear before and after it to merge.

\begin{definition}\label{def:intervalWorkSet}
	{\bf Interval Work-Set:}	
	\change{An interval work-set $\mathcal{W}_0^n$ is a set, 
	$\{{I}_{\mathcal{B}_0},I_{\mathcal{B}_1},{I}_{\mathcal{B}_2},...,{I}_{\mathcal{B}_n}|$ $ {I}_{\mathcal{B}_i} \in
	\mathcal{I}_\tau(\mathcal{B}_i), 0\leq i \leq n\}$, 
	of labeled truth intervals for the set of labelled buckets, $\mathbb{B} = \{\mathcal{B}_0$, $\mathcal{B}_1$, $\mathcal{B}_2$,..., $\mathcal{B}_n\}$, of constraint set $\mathcal{C}$.
	We also define $\mathcal{W}_i^k = \{{I}_{\mathcal{B}_j} | i \leq j \leq k\}$.}
	
	\change{For trace $\mathcal{T}$, different combinations of bucket truth intervals, produce unique interval work-sets. The set of all work sets that can be derived from $\mathcal{I}_\tau(\mathcal{B}_i), 0\leq i \leq n$, is given as $\mathbb{W}_\mathcal{C}$ or simply $\mathbb{W}$ when the context $\mathcal{C}$ is known.}
	\qed
\end{definition}

\change{Intuitively, each Interval Work-Set is constructed by choosing
one interval from each prefix bucket. This has been further illustrated in 
Example~\ref{ex:forwardInfluence}. }

%\change{We compute the \textit{Forward Influence} and \textit{Backward Influence} for a sequence expression $\mathcal{S}$. The forward influence is the set of time points looking forward in time, for which the possible outcomes are constrained by the sequence expression. On the other hand, the backward influence gives the time points looking backward in time that are constrained by the sequence expression. \pd{What do you mean by "outcomes" here? Please consider replacing this paragraph with the following one.}}

%{\color{red}
\change{We use the notion of \textit{Forward Influence} to find the set of time
points that qualify as a {\em match} for a prefix sequence following Definition~\ref{def:seqExprMatch}. Since the designated time of the match
is the time at which the match ends, we refer to these time points as
{\em end-match} times. End-match times can be spread over an interval,
and we shall refer to such intervals as {\em end-match intervals}.}
%}

\begin{definition}\label{def:forwardInfluenceMining}
{\bf Forward Influence $\mathds{F}(\mathcal{S},\mathcal{W}_0^n)$}:
% The forward influence is the set of time points for which the possible outcomes are constrained by the seq expression.
% The influences are the time points that affect the match outcomes of the sequence expression.
The forward influence for a prefix sequence expression $\mathcal{S} =  s_n~ \tau_{n}~ s_{n-1}~ ... ~\tau_{1}~ s_0$, 
given the interval work-set $\mathcal{W}_0^n = \{{I}_{\mathcal{B}_0}, ..., {I}_{\mathcal{B}_n}\}$, 
is an interval, recursively defined as follows:
\begin{center}
\begin{tabular}{lll}
$\mathds{F}(\mathcal{S},\mathcal{W}_i^i)$&$=$&${I}_{\mathcal{B}_i}$ \\	
$\mathds{F}(\mathcal{S},\mathcal{W}_i^j)$&$=$&$(\mathds{F}(\mathcal{S},\mathcal{W}_{i+1}^{j}) \oplus \tau_{i})~ \cap ~  {I}_{\mathcal{B}_i}$ , $0\leq i < j \leq n$
%$\mathds{F}(S,\mathcal{W}_1^n)$&$=$&$(\mathds{F}(S,\mathcal{W}_2^{n}) \oplus d_{1})~ \cap ~  {I}_{s_1} $
\end{tabular}
\end{center}
\change{where, $\oplus$ represents the Minkowski sum of intervals: $[\alpha:\beta] \oplus [a:b] = [\alpha+a:\beta+b]$.}
For $\mathcal{S}$, the set of intervals in $\mathds{F}(\mathcal{S},\mathcal{W}_0^n)$ are called the \textbf{end-match} intervals of $\mathcal{S}$. The set of all forward influence intervals over all work sets $\mathbb{W}$, is $\mathds{F}(\mathcal{S},\mathbb{W}) = \bigcup_{\substack{
\mathcal{W}_0^n \in \mathbb{W}
}} 
\mathds{F}(\mathcal{S},\mathcal{W}_0^n)$.
\qed
\end{definition}

\change{
\begin{example} \label{ex:forwardInfluence}
{
Consider the sequence expression \change{$\mathcal{S}$ $\equiv$} {\tt $s_2$ \#\#[1:4] $s_1$ \#\#[2:8] $s_0$}, and \change{bucket} truth interval sets \change{for a trace $\mathcal{T}$}, $\mathcal{I}_\mathcal{T}(\mathcal{B}_2) = \{[2:4]\}$, $\mathcal{I}_\mathcal{T}(\mathcal{B}_1) = \{[3:5], [7:9]\}$ and $\mathcal{I}_\mathcal{T}(\mathcal{B}_0) = \{[4:9], [12:19]\}$. 
There are $1\times2\times2 = 4$ interval work-sets.

The computation of forward influence using Definition~\ref{def:forwardInfluenceMining}, for each possible interval work-set is described in the form of a tree in Figure~\ref{fig:forwardInfluenceMining}. 
An interval work-set is the set of truth intervals of buckets encountered along a path from the root to a 
leaf node in the tree. The tree is rooted at a node corresponding to the truth interval $[2:4]$ for $\mathcal{B}_2$. 
Level $i$ in the tree corresponds to the computation of $\mathds{F}(\mathcal{S},\mathcal{W}_i^2)$.

The sequence expression, therefore, has four end-match intervals, namely
$[5:9]$, $[9:9]$, $[12:13]$, and $[12:16]$. These intervals are contributed
by different combinations of truth intervals of the constituting predicates
(that is, different interval work-sets), and may have overlaps.} \qed
\end{example}
}
\change{Even though a truth interval of a predicate in the work-set may contribute to a match of a sequence expression, {\em not all the points in that truth interval
may participate in the contribution}. In other words, it is quite possible \newchange{that only} some sub-interval of a truth interval actually \newchange{takes part in} the forward influence. Finding such sub-intervals eventually leads us to find the {\em begin-match intervals} for a sequence expression.}

\change{In order to find the sub-intervals of the intervals in the work-set, we use
a {\em backward influence} computation as defined below.
}

\begin{figure}
        \begin{subfigure}[b]{0.5\textwidth}
                \centering
                	\includegraphics[width=0.9\linewidth]{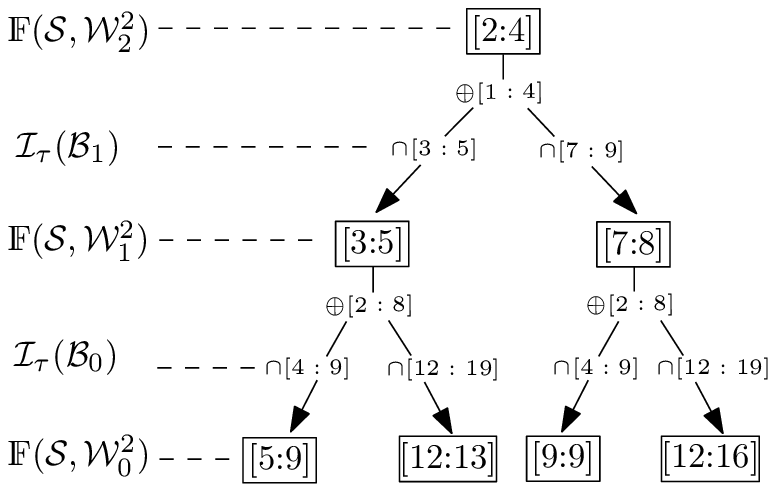}
                	\caption{Forward Influence computation }\label{fig:forwardInfluenceMining}
        \end{subfigure}%
        \begin{subfigure}[b]{0.5\textwidth}
                \centering
                \includegraphics[width=0.9\linewidth]{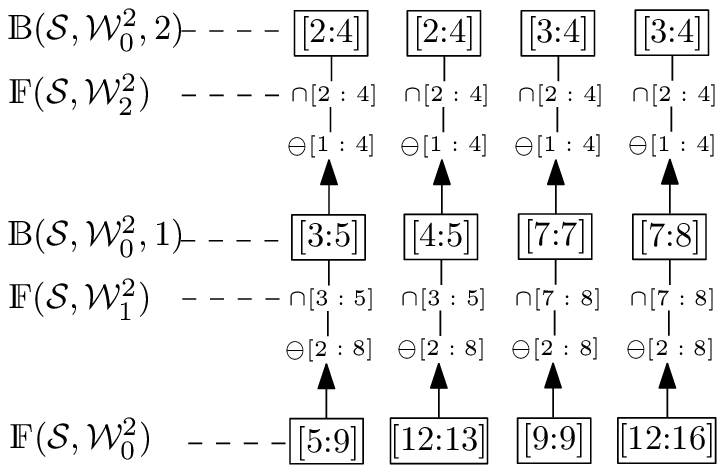}
               	\caption{Backward Influence computation}\label{fig:backwardInfluenceMining}
        \end{subfigure}%
        \caption{Influences computed for $\mathcal{B}_2 \#\# [1:4]~ \mathcal{B}_1 \#\#[2:8] ~\mathcal{B}_0$, with $\mathcal{I}_\tau(\mathcal{B}_2) = \{[2:4]\}$, $\mathcal{I}_\tau(\mathcal{B}_1) = \{[3:5],[7:9]\}$ and $\mathcal{I}_\tau(\mathcal{B}_0) = \{[4:9],[12:19]\}$}\label{fig:animals}
\end{figure}

\begin{definition}\label{def:backwardInfluenceMining}
{\bf Backward Influence $\mathds{B}(\mathcal{S},\mathcal{W}_0^n,i)$}:	
The backward influence, $\mathds{B}(\mathcal{S},\mathcal{W}_0^n,i)$, $i\in[0,n]$, for a sequence expression $S =  s_n~ \tau_{n}~ s_{n-1}~ ... ~\tau_{1}~ s_0$, 
given the interval work-set $\mathcal{W}_0^n = \{{I}_{\mathcal{B}_0},{I}_{\mathcal{B}_1}, ..., {I}_{\mathcal{B}_n}\}$, 
is an interval defined as follows:
\begin{center}
	\begin{tabular}{lll}
		$\mathds{B}(\mathcal{S},\mathcal{W}_0^n,0)$&$=$&$\mathds{F}(\mathcal{S},\mathcal{W}_0^n)$\\
		$\mathds{B}(\mathcal{S},\mathcal{W}_0^{n},i)$&$=$&$(\mathds{B}(\mathcal{S},\mathcal{W}_0^n,i-1) \ominus \tau_{i}) \cap  \mathds{F}(\mathcal{S},\mathcal{W}_i^n)$ , $0<i\leq n$\\ %\mathds{F}(S,\mathcal{W}_1^k)$,\\%{I}_{s_k}
		%&&\hspace{10em}$1\leq k < n$\\
	\end{tabular}
\end{center} 
\change{where, $\ominus$ represents the Minkowski difference of intervals: $[\alpha:\beta] \ominus [a:b] = [\alpha-b:\beta-a]$.} 
For $\mathcal{S}$, the set of intervals in $\mathds{B}(\mathcal{S},\mathcal{W}_0^{n},n)$ are called the \textbf{begin-match} intervals of $\mathcal{S}$. The set of all backward influence intervals at position $i$, over all work sets $\mathbb{W}$, is $\mathds{B}(\mathcal{S},\mathbb{W},i) = \bigcup_{\substack{
\mathcal{W}_0^n \in \mathbb{W}
}} 
\mathds{B}(\mathcal{S},\mathcal{W}_0^n,i)$.
\qed
\end{definition}

For a given prefix sequence expression $\mathcal{S}$ having at most $n$ time intervals, and interval work-set $\mathcal{W}_0^n$, we use the shorthand notation $\mathds{F}_0^i$ to represent $\mathds{F}(\mathcal{S},\mathcal{W}_0^i)$, and $\mathds{B}_0^i$ to represent $\mathds{B}(\mathcal{S},\mathcal{W}_0^n,i)$.  \change{Computing backward influences will, as we see later, also allow us to refine delay intervals between non-empty buckets of a learned property.}

\begin{example}\label{ex:backwardInfluence}
\change{
We continue with the example of Figure~\ref{fig:forwardInfluenceMining}. 
For each interval set, the backward influence is computed, using Definition~\ref{def:backwardInfluenceMining}. 
The computation begins with the leaves of the tree in Figure~\ref{fig:forwardInfluenceMining}, and proceeds backwards 
through the sequence expression to determine the intervals corresponding to each match, indicated as a bottom-up 
computation in Figure~\ref{fig:backwardInfluenceMining}.}

\change{In a sequence expression computing the forward influence is not sufficient to identify the sequence of intervals contributing to a match. We explain this using Figure~\ref{fig:backwardInfluenceMining}. 
Observe the second column in Figure~\ref{fig:backwardInfluenceMining} corresponding 
to the interval work-set $\mathcal{W}_0^2= \{[12:19]_{\mathcal{B}_0},[3:5]_{\mathcal{B}_1},[2:4]_{\mathcal{B}_2}\}$. From Figure~\ref{fig:forwardInfluenceMining}, the 
forward influence computes the influence intervals to be $[2:4]$, $[3:5]$, $[12:13]$. The interval $[3:5]$ 
corresponds to the forward influence match up to $\mathcal{B}_1$. The truth interval of $\mathcal{B}_0$ under consideration for this 
match is the interval $[12:19]$. Of the truth interval $[3:5]$ of $\mathcal{B}_1$, observe that $[3:4) \oplus [2:8] = 
[5:12)$,  $[5:12) \cap [12:19] = \emptyset$, which does not fall within the truth interval $[12:19]$ of 
$\mathcal{B}_0$, and thus truth interval $[3:4)$ cannot contribute to a match. On the other hand, $[4:5] \oplus [2:8] = [6:13]$, and $[6:13] 
\cap [12:19] = [12:13]$. Therefore, of the interval $[3:5]$, only $[4:5]$ contributes to a match.} \qed
\end{example}
%
%\pd{Instead of the following proposition, can we give a single theorem covering 
%what we achieve by using the combination of forward and backward influence?}
%
%\begin{proposition}
%For a potential infinite continuum of prefix sequence expression matches associated with an interval work-set, tight delay intervals between sub-expressions of the sequence can be computed using the backward influence.\qed
%\end{proposition}

%\pd{Please consider calling this an observation instead of a theorem. If you wish to keep it as a theorem, then you must provide a brief proof.}
\begin{observation}
\change{For a potential infinite continuum of prefix sequence expression matches associated with an interval work-set, the forward-influence computes the corresponding end-match time intervals, and the backward influence computes sub-intervals from work-sets that take part in matches.}
%\pd{backward influence computes delay intervals between buckets in the sequence-expression. $\leftarrow$ But backward influence does not seem to change the delays "between" buckets?}
\qed
\end{observation}

\change{Recall, that} we use the shorthand notation $\mathcal{S}_\mathcal{C}$ to represent the prefix sequence constructed from constraint-set $\mathcal{C}$. %It should be remembered that the prefix sequence expression $\mathcal{S}= s_n ~\tau_{n} ~s_{n-1} ~\tau_{n-1} \ldots \tau_{1}~s_0$ is constructed from the constraint set $\mathcal{C}$. The delay terms between sub-expressions in $\mathcal{S}$ are multiples of the resolution $k$ (as described in Section~\ref{sec:prefixMiningProblem}). Hence, we use $\mathcal{S}_\mathcal{C}$ to represent the prefix constructed from $\mathcal{C}$. If $\mathcal{C}$ is not known, we use $\mathcal{S}$ for the prefix.s \pd{$\leftarrow$ Please check this statement.}

\begin{definition}{\bf Influence Set} \label{def:InfluenceSet}
The influence set for a sequence expression for constraint set $\mathcal{C}$, \change{$S_\mathcal{C} =  s_n~ \tau_{n}~ s_{n-1}~ ... ~\tau_{1}~ s_0$},  given interval sets for each bucket, $\mathcal{I}_\mathcal{T}(\mathcal{B}_i)$, $0\leq i\leq n$, is defined as the union of end-match intervals over all possible work sets $\mathcal{W}_0^n \in \mathbb{W}$, defined  as follows:
$\mathcal{I}_\mathcal{T}(\mathcal{S}_\mathcal{C}) = \bigcup_{\substack{
\mathcal{W}_0^n \in \mathbb{W}
}} 
\mathds{F}(\mathcal{S}_\mathcal{C},\mathcal{W}_0^n)
$
\qed
\end{definition}

\begin{definition}
	\textbf{Length of a Truth Set:}
	The length of a Truth Set $\mathbb{I}_{\mathcal{T}_\mathcal{C}}(\mathbb{P})$, under constraint set $\mathcal{C}$, represented by $|\mathbb{I}_{\mathcal{T}_\mathcal{C}}(\mathbb{P})|$ is the length of the influence set, given as $|\mathcal{I}_\mathcal{T}(\mathcal{S}_\mathcal{C})|$.
	\qed 
	\label{def:TruthSetlength}
\end{definition}

%\pd{Use the same example to illustrate the above two definitions.}
\begin{example}
\change{
We continue with the example of Figure~\ref{fig:forwardInfluenceMining}. 
The set of all possible work sets, $\mathbb{W}$, in the figure, is $\mathbb{W}=\{
\{[4:9]_{\mathcal{B}_0},[3:5]_{\mathcal{B}_1},[2:4]_{\mathcal{B}_2}\},
\{[4:9]_{\mathcal{B}_0},[7:9]_{\mathcal{B}_1},[2:4]_{\mathcal{B}_2}\},
\{[12:19]_{\mathcal{B}_0},[3:5]_{\mathcal{B}_1},[2:4]_{\mathcal{B}_2}\},
\{[12:19]_{\mathcal{B}_0},[7:9]_{\mathcal{B}_1},[2:4]_{\mathcal{B}_2}\}\}$. Let $\mathcal{C}$ be the constraint set forming the buckets in the sequence-expression. The influence set for $\mathcal{S}_\mathcal{C}$, is computed in Example~\ref{ex:forwardInfluence} using forward influence. The influence set is $\mathcal{I}_\mathcal{T}(\mathcal{S}_\mathcal{C}) = \{[5:9], [9:9], [12:13], [12:16]\}$. The length of the truth set for $\mathcal{S}_\mathcal{C}$ is $|\mathcal{I}_\mathcal{T}(\mathcal{S}_\mathcal{C})| = 9-5+9-9+13-12+16-12 = 9$.} \qed
\end{example}

This section may be summarized as follows. Properties over dense real-time may match over a continuum of time points. We use time intervals to represent truth points for predicates (Definition~\ref{def:intervalSet}) and sets of predicates (Definition~\ref{def:truthSet}). \change{A prefix sequence is built from a set predicate constraints (Definition~\ref{def:constraintSet}), where each predicate occupies a fixed position (bucket) in the prefix sequence. Multiple predicates sharing the same position in the sequence form a prefix-bucket (Definition~\ref{def:bucket}). A predicate can be true over multiple disjoint time intervals. All predicates in the same bucket are conjuncted together, and hence a bucket of predicates may be true over a set of intervals. Choosing a combination of truth intervals, one truth interval from each bucket position, forms an interval work-set (Definition~\ref{def:intervalWorkSet}). For the prefix-sequence and a given work-set, the forward influence (Definition~\ref{def:forwardInfluenceMining}) provides a mechanism for computing the end time-points associated with the match of the prefix-sequence , while the backward influence (Definition~\ref{def:backwardInfluenceMining}) provides a mechanism for computing the begin time-points associated with the matching end time-points of the forward influence. The set of all end time-points for all work-sets of a prefix-sequence form a set of intervals where the prefix-sequence has influence, the influence set (Definition~\ref{def:InfluenceSet}). The choice of constraints $\mathcal{C}$ limits the length of the truth set (Definition~\ref{def:TruthSetlength}), and determines the decisions made for mining additional constraints.}

\newpage\section{Learning Decision Trees for Sequence Expressions}
\label{sec:concurrent}
{
\label{sec:immediate}
\change{
This section develops the methodology for learning decision trees treating the
target as the decision variable. Each path of the decision tree from the root to
a leaf will represent a prefix sequence for the target or its negation. }

The semantics of sequence expressions allow for both, {\em immediate causality} and {\em future causality} to be asserted. Immediate causality, expressed as $\mathcal{S}$ {\tt |->} $E$, is observed when the truth of sequence $\mathcal{S}$ at time $t$ causes the consequent $E$ to be true at time $t$. Future causality, expressed by the assertion {\tt $\mathcal{S}$ |-> \#\#[a:b] $E$} relates the truth of the sequence $\mathcal{S}$ at time $t$ with the truth of $E$ at time \change{$t' \in [t+a:t+b]$}. \change{In both these forms, $E$ is Boolean, and $\mathcal{S}$ is, by default, a temporal sequence expression.}

\change{An important special case of {\em immediate causal} relations consists of 
relations where $\mathcal{S}$ is strictly Boolean and does not contain any
delay interval. We first present decision making metrics available in standard texts~\cite{Mitchell1997} that we have adapted to interval sets for mining immediate causal relations of this type. Later, we demonstrate through examples that these metrics can lead to incorrect decisions when $\mathcal{S}$ is temporal. At the end, we present metrics and a decision tree algorithm to learn sequence expressions to express {\em future causality} relations.}

\subsection{Decision Metrics for learning Immediate Relations}\label{sec:NodeMetrics}
At each node of the decision tree, statistical measures of \textit{Mean} and \textit{Error} are used to evaluate the node. Standard decision tree algorithms use measures of \textit{Entropy} or \textit{Ginni-Index}~\cite{Aggarwal2015} to measure the disorder and chaos in the data. %The aim is to choose or design decision making metrics that the metric unambiguously distinguishes when a property has counter-examples verses when no counter-examples exist in the data.

For immediate relations of the form $S$ {\tt |->} $E$, where $\mathcal{S}$ and $E$ are Boolean expressions, \change{the property template is {\tt $\sqcup_0$ |-> E}}, and standard Shannon Entropy is used as a measure of disorder. Information gain is used to evaluate decisions at internal nodes of the decision tree. 

\begin{definition}\label{def:mean}
\textbf{Mean$_{\mathcal{T}_\mathcal{C}}$(E)~\cite{Ott2006}:} 
For the target $E$, the proportion of time in trace $\mathcal{T}$ that $E$ is true in the trace constrained by $\mathcal{C}$:
\[
Mean_{\mathcal{T}_\mathcal{C}}(E) =\frac{|\mathcal{I}_{\mathcal{T}}(E) \cap \mathcal{I}(\mathcal{S}_\mathcal{C})|} {|\mathcal{I}(\mathcal{S}_\mathcal{C})|}
\]
The mean represents the conditional probability of $E$ being true under the influence of the constraints in $\mathcal{C}$. 
For convenience, we use $\mu_{\mathcal{T}_\mathcal{C}}(E)$ to refer to $Mean_{\mathcal{T}_\mathcal{C}}(E)$. The mean is not defined when $|\mathcal{I}(\mathcal{S}_\mathcal{C})| = 0$.
\qed
\end{definition}

\change{Note that for this special case of immediate causality, the mean represents the conditional 
probability of $E$ being true under the {\em immediate} influence of the constraints in $\mathcal{C}$. Also, there is only one immediate bucket,
$\sqcup_0$, and therefore all members of $\mathcal{C}$ correspond to
the same bucket. This means $\mathcal{I}(\mathcal{S}_\mathcal{C})$ 
corresponds to those regions of the trace which satisfy all predicates in 
$\mathcal{C}$.}

\begin{lemma}
The mean, $\mu_{\mathcal{T}_\mathcal{C}}(E)$ is 1 if and only if $E$ is true wherever $\mathcal{S}_\mathcal{C}$ is true, and 0 if and only if $E$ is false wherever $\mathcal{S}_\mathcal{C}$ is true.
\end{lemma}
\begin{proof}
If $E$ is true wherever $\mathcal{S}_\mathcal{C}$ is true, then $\mathcal{I}_{\mathcal{T}}(E) \cap \mathcal{I}(\mathcal{S}_\mathcal{C}) = \mathcal{I}(\mathcal{S}_\mathcal{C})$, therefore $\mu_{\mathcal{T}_\mathcal{C}}(E) = 1$. Conversely, if $\mu_{\mathcal{T}_\mathcal{C}}(E) = 1$, then $ |\mathcal{I}_{\mathcal{T}}(E) \cap \mathcal{I}(\mathcal{S}_\mathcal{C})| =  |\mathcal{I}(\mathcal{S}_\mathcal{C})|$, which is possible only in the situation where  $ \mathcal{I}_{\mathcal{T}}(E) \cap \mathcal{I}(\mathcal{S}_\mathcal{C}) =  \mathcal{I}(\mathcal{S}_\mathcal{C})$.
Similarly, it can be shown that E is false whenever $\mu_{\mathcal{T}_\mathcal{C}}(E) = 0$.
\end{proof}

%\noindent\textbf{Error:}\label{subsec:error}
%For a target $E$, the error is a measure of entropy in the data with regards to the classes $E$ and $\neg E$. It is a measure of how well the set of constraints $\mathcal{C}$ explain $E$. An error value of zero indicates that there is no disagreement in the class ($E$ or $\neg E$) under the constraint set $\mathcal{C}$. 

\begin{definition}\label{def:error}
\textbf{Error$_{\mathcal{T}_\mathcal{C}}(E)$\cite{Aggarwal2015}:} For the target class $E$, the error for the trace $\mathcal{T}$ constrained by $\mathcal{C}$ is defined as follows:
	\begin{align*}
	Error_{\mathcal{T}_\mathcal{C}}(E)
	& =	-\mu_{\mathcal{T}_{\mathcal{C}}}(E) \times log_2(\mu_{\mathcal{T}_{\mathcal{C}}}(E)) 
	- 
	\mu_{\mathcal{T}_{\mathcal{C}}}(\neg E) \times log_2(\mu_{\mathcal{T}_{\mathcal{C}}}(\neg E))
	\end{align*}
%The average deviation of truth of $E$ from Mean$_{\mathcal{T}_\mathcal{C}}$(E) is the error in  trace $\mathcal{T}$ under the constraints in $\mathcal{C}$. Error is a measure of the disorder at the node. Since the domain of values for a $E$ is $\{\top, \bot\}$, the error function may be expressed as follows:
	%\par \setlength{\parskip}{0pt plus 0pt minus 5pt}
	%\allowdisplaybreaks\begingroup
	%\addtolength{\jot}{0.5em}	
\noindent For convenience, we use $\epsilon_{\mathcal{T}_{\mathcal{C}}}(E)$ to refer to  $Error_{\mathcal{T}_\mathcal{C}}(E)$. %This measure may be simplified to the following:
%\begin{equation}
%\epsilon_{\mathcal{T}_{\mathcal{C}}}(E) = 2\times \mu_{\mathcal{T}_{\mathcal{C}}}(E) \times (1 - \mu_{\mathcal{T}_{\mathcal{C}}}(E))
%\end{equation}
\qed
\end{definition}

\begin{lemma}
%For an immediate property, the error, \textit{$\epsilon_{\mathcal{T}_\mathcal{C}}$(E)} for constraint set $\mathcal{C}$ and consequent $E$ is zero if and only if there are no counter-examples for $\mathcal{S}_{\mathcal{C}} ${\tt |->}$ [\neg]E$.
\change{For an immediate property, the error,
\textit{$\epsilon_{\mathcal{T}_\mathcal{C}}$(E)} for constraint set $\mathcal{C}$ and consequent $E$ is zero if and only if $\mathcal{S}_{\mathcal{C}}$ decides the
truth of $E$ or $\neg E$, that is, either there are no counter-examples for $\mathcal{S}_{\mathcal{C}} ${\tt |->}$ E$ or there are no counter-examples for
$\mathcal{S}_{\mathcal{C}} ${\tt |->}$ \neg E$. }
\end{lemma}
\begin{proof}
We assume that $|\mathcal{I}(\mathcal{S}_\mathcal{C})| > 0$, the sum of lengths of truth intervals for the expression describing $\mathcal{C}$ is non-zero. 

\noindent\textbf{Part A:}
Consider the property $\mathcal{S}_{\mathcal{C}} ${\tt |->}$ E$, where $|\mathcal{I}_{\mathcal{T}}(E) \cap \mathcal{I}(\mathcal{S}_\mathcal{C})| > 0$. 
Assume that there are counter-examples for   $\mathcal{S}_{\mathcal{C}} ${\tt |->}$ E$. 
The counter-examples would introduce a non-zero time interval when $\mathcal{S}_{\mathcal{C}}$ is true and $E$ is false.
Let one of the counter-example time-intervals be the interval $[a:b)$, where $b>a$. 
Therefore, $|\mathcal{I}_{\mathcal{T}}(\neg E) \cap \mathcal{I}(\mathcal{S}_\mathcal{C})|>0$, hence $\mu_{\mathcal{T}_\mathcal{C}}(\neg E) > 0$, and $log_2(\mu_{\mathcal{T}_\mathcal{C}}(\neg E)))<0$ (the mean is bounded between 0 and 1). Similarly the term, $\mu_{\mathcal{T}_\mathcal{C}}(E) \times log_2(\mu_{\mathcal{T}_\mathcal{C}}(E))) < 0$, since $|\mathcal{I}_{\mathcal{T}}(\neg E) \cap \mathcal{I}(\mathcal{S}_\mathcal{C})| > 0$ . Hence, from Definition~\ref{def:error}, $\epsilon_{\mathcal{T}_\mathcal{C}}(E) > 0$. 

Consider the property $\mathcal{S}_{\mathcal{C}} ${\tt |->}$ E$, where $|\mathcal{I}_{\mathcal{T}}(E) \cap \mathcal{I}(\mathcal{S}_\mathcal{C})| = |\mathcal{I}(\mathcal{S}_\mathcal{C})|$. Hence, $|\mathcal{I}_{\mathcal{T}}(\neg E) \cap \mathcal{I}(\mathcal{S}_\mathcal{C})| = 0$. Hence  $\mu_{\mathcal{T}_\mathcal{C}}(E) = 1$,  $\mu_{\mathcal{T}_\mathcal{C}}(\neg E) = 0$ and hence  $Error_\mathcal{T}(E) = 0$. 

The proof for when the property is of the form $\mathcal{S}_{\mathcal{C}} ${\tt |->}$ \neg E$ is identical.

\noindent\textbf{Part B:}
Conversely, if the error is non-zero, then both terms of Definition~\ref{def:error} are non-zero (the intervals for $E$ and $\neg E$ are compliments of each other). From Definition~\ref{def:mean}, $\mu_{\mathcal{T}_\mathcal{C}}(E) > 0$,  $\mu_{\mathcal{T}_\mathcal{C}}(\neg E) > 0$, and hence  $\mathcal{I}_{\mathcal{T}}(E) \cap \mathcal{I}(\mathcal{S}_\mathcal{C}) \neq \mathcal{I}(\mathcal{S}_\mathcal{C})$ and  $\mathcal{I}_{\mathcal{T}}(\neg E) \cap \mathcal{I}(\mathcal{S}_\mathcal{C}) \neq \mathcal{I}(\mathcal{S}_\mathcal{C})$. Therefore, there exists a non-empty interval $[a_1:b_1)$, $b_1>a_1$ where $\mathcal{S}_\mathcal{C}$ is true and $E$ is false, and there exists a non-empty interval $[a_2:b_2)$, $b_2>a_2$ where $\mathcal{S}_\mathcal{C}$ is true and $E$ is true, respectively representing counter-example intervals for $\mathcal{S}_{\mathcal{C}} ${\tt |->}$ E$ and $\mathcal{S}_{\mathcal{C}} ${\tt |->}$ \neg E$.
\end{proof}
\change{
We aim to construct a decision tree where the target, $E$, is the decision variable. Each node of the decision tree represents a predicate, which is chosen on the basis of its utility in separating the cases where $E$ is true from the cases where $E$ is false.
Each branch out of a node represents one of the truth values of the predicate representing that node.
The set of cases at an intermediate node of the decision tree are the time points at which the predicates from the root to that node have precisely the values representing the edges along that path
(recall the definition of a {\em constraint set}). 
}

This utility metric is called the \textit{gain}. While many gain metrics exist~\cite{Aggarwal2015}, we find that Information Gain works best for our two class application. The standard definition of Information Gain is given below.

\begin{definition}\label{def:gain}
\textbf{Gain~\cite{Quinlan1986}:} The gain (improvement in error) of choosing $P \in \mathbb{P}$ to add to constraint set $\mathcal{C}$, at a node having error $\epsilon_{\mathcal{T}_\mathcal{C}}(E)$ is as follows:
\[\pushQED{\qed} Gain = \epsilon_{\mathcal{T}_\mathcal{C}}(E) -  
\frac{
|\mathcal{I}(\mathcal{S}_{\mathcal{C} \cup \{P\}})|}{|\mathcal{I}(\mathcal{S}_\mathcal{C})|
} 
\times 
\epsilon_{\mathcal{T}_{\mathcal{C}\cup\{P\}}}(E)
- 
\frac{
|\mathcal{I}(\mathcal{S}_{\mathcal{C} \cup \{\neg P\}})|}{|\mathcal{I}(\mathcal{S}_\mathcal{C})|
} 
\times 
\epsilon_{\mathcal{T}_{\mathcal{C}\cup\{\neg P\}}}(E)\qedhere
\popQED\]

\end{definition}
}

Applying traditional decision tree learning using Definitions~\ref{def:error} and~\ref{def:gain} on the truth-set is suitable for mining those immediate properties where the antecedent, $\mathcal{S}$, is Boolean. 

To mine prefix sequences with delays, $E$'s truth must be tested with the truth of other predicates over past time points. In Section~\ref{sec:pseudoTargets}, we introduce the notion of {\em pseudo-targets} that allow us to evaluate constraints that have influence on $E$ over a time-span. 
%We show that standard measures used for decision making are not suitable for evaluating the {\em goodness} of a decision involving pseudo-targets. The classification task we deal with is aimed at classifying time points describing behaviours that explain when the target is true or when it is false. While this is similar to a standard two class classification task, it is not. 
%Due to the non-deterministic semantics of prefix-sequences, the two classes may share end-match time points\footnote{This is further elaborated in Example~\ref{ex:sharedTimePointsPseudoTargetTruths}.}, requiring an adapted definition of Shannon Entropy. Similarly, the decision made at each node of the decision tree splits the data-set in a manner that allows sharing of time-points between the branches of a split, thus requiring a special handling of gain. We therefore introduce variations of these measures enabling the best greedy decision be taken. The proposed measures are evaluated using a correlation-coverage metric to measure the proportion of the target's truth covered by the mined prefixes.
In Section~\ref{sec:newMetrics} we provide metrics for evaluating decisions involving pseudo-targets. Section~\ref{sec:gainVtime} provides insights into the proposed metrics and Section~\ref{sec:psi-miner-algo} incorporates them into an algorithm for learning prefixes. In Section~\ref{sec:psi-l-to-formulae} we explain how the learned decision tree is translated into PSI-L temporal logic properties.

\subsection{Pseudo-Targets for Sequence Expressions}\label{sec:pseudoTargets}

% Observe the truth set of predicates $P$ and $E$ in Figure~\ref{fig:pseudo-targets}. The horizontal green bands indicate time intervals where the predicate is true, while the red bands indicate intervals where the predicate is false. It is clear that the truths of $P$ and $E$ do not align. Our intention is to learn a temporal relation between $P$ and $E$ (if such a relation exists). When learning relations in the presence of an alphabet of predicates, it is expensive to examine the relationships between every pair of predicates in the alphabet and their relationship in turn with the target. Additionally, it is important to have domain knowledge about the system to determine the quantum of time in the past of $E$ that a predicate would be expected to have an influence on it, if at all. In order to scale such an analysis to large predicate alphabets we use \textit{pseudo-targets}. Pseudo-targets allow us to compute summary statistical measures that give clear indications of the existence or absence of such relations.  

\change{
We use the notion of {\em pseudo-targets} to compute summary statistical measures for the influence of predicates on the target across time. Recall from 
Section~\ref{sec:motivation} the template form:
\begin{mycenter}[0em]
{\tt
 $\sqcup_n$ \#\#[0:k]  $\sqcup_{n-1}$ \#\#[0:k] $\ldots$ \#\#[0:k] $\sqcup_1$ \#\#[0:k] $\sqcup_0$ |-> \#\#[0:k] E }
\end{mycenter}
The parameters used in the above are the {\em delay resolution}, $k$, and 
the length, $n$, of the sequence expression. Pseudo-targets are generated by 
stretching the target back in time across each of the $n$ delay intervals.
}

\begin{figure}[t]
\centering
\includegraphics[scale=0.8]{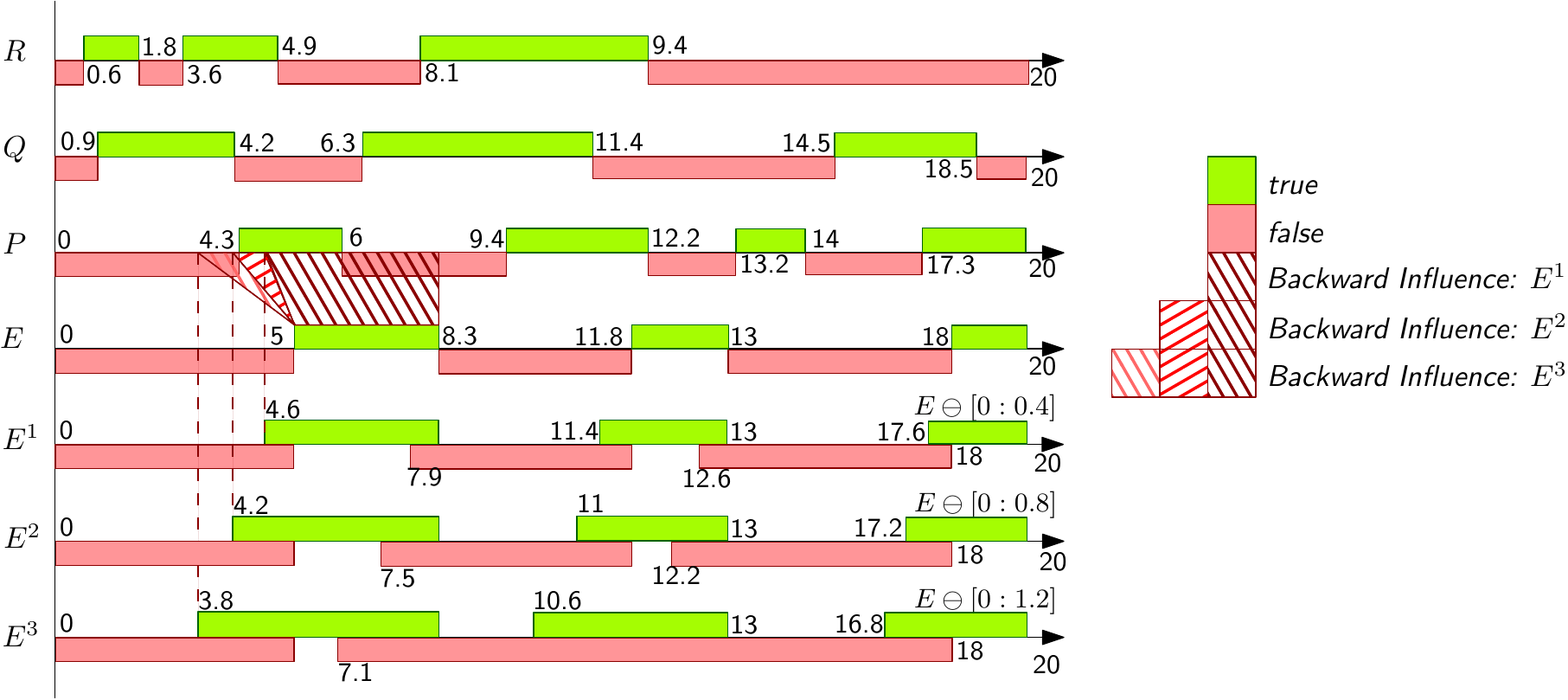}
\caption{Truths of Predicates=\{P,Q,R,E\} and Pseudo-Targets (n=3, k=0.4) for $E$.}\label{fig:pseudo-targets}
\end{figure}

\begin{definition}\label{def:pseudoTarget}
\textbf{Pseudo-Target:}
A pseudo-target is an artificially created target computed by stretching the truth of the target's interval set back in time by a multiple of the delay resolution $k$. The target $E$ stretched back in time by an amount $i\times k$ is denoted as $E^i$. The interval \change{set} for the pseudo-target $E^i$ in trace $\mathcal{T}$ is computed as follows:
\begin{equation}\label{eq:pseudo-target}
\mathcal{I}_\mathcal{T}(E^i) = \mathcal{I}_\mathcal{T}(E) \ominus [0:i\times k] = \bigcup_{I \in \mathcal{I}_\mathcal{T}(E)} I \ominus [0:i\times k]
\end{equation}
%and,
%\begin{equation}
%\mathcal{I}_\mathcal{T}(P) \ominus [a:b] = \bigcup_{I \in %\mathcal{I}_\mathcal{T}(P)} I \ominus [a:b]
%\end{equation}  
\noindent where, the $\ominus$ represents the Minkowski difference between intervals: $[\alpha:\beta] \ominus [a:b] = [\alpha-b:\beta-a]$. The truth intervals for the negated pseudo-targets are computed similarly.\qed
\end{definition}

For a prefix sequence with at-most $n$ \change{delay intervals} and a delay resolution $k$, we augment the truth set, $\mathbb{I}_\mathcal{T}(\mathbb{P})$ to  $\widehat{\mathbb{I}}_\mathcal{T}(\mathbb{P})$   as follows:

\begin{equation}
\widehat{\mathbb{I}}_\mathcal{T}(\mathbb{P}) = \mathbb{I}_\mathcal{T}(\mathbb{P}) \cup (\bigcup_{1\leq i \leq n} \mathcal{I}_\mathcal{T}(E^i))
\cup (\bigcup_{1\leq i \leq n} \mathcal{I}_\mathcal{T}(\neg E^i))
\label{eq:transformedTruthSet}
\end{equation}

\begin{example}\label{ex:sharedTimePointsPseudoTargetTruths}
In Figure~\ref{fig:pseudo-targets}, $n=3$ and $k=0.4$, the truth intervals of three pseudo-targets of predicate $E$ are shown. Pseudo-target $E^i$ is computed according to Equation~\ref{eq:pseudo-target}. For instance, given $\mathcal{I}_\tau(E) = \{[5:8.3),[11.8:13),[18:20)\}$ and $\mathcal{I}_\tau(\neg E) = \{[0:5),[8.3:11.8),[13:18)\}$, $\mathcal{I}_\tau(E^2)$ and $\mathcal{I}_\tau(\neg E^2)$ are computed to be as follows:
\begin{align*}
\pushQED{\qed}
\mathcal{I}_\tau(E^2) & =  \{[5:8.3),[11.8:13),[18:20)\} \ominus [0:0.8]\\
& = \{[4.2:8.3),[11:13),[17.2:20)\} \\
\mathcal{I}_\tau(\neg E^2) & =  \{[0:5),[8.3:11.8),[13:18)\} \ominus [0:0.8]\\
&= \{[0:5),[7.5:11.8),[12.2:18)\}\qedhere
\popQED
\end{align*}
\end{example}

\noindent Observe that in Figure~\ref{fig:pseudo-targets}, the true and false intervals for pseudo-targets are not complementary. We wish to mine prefixes to explain both $E$ and $\neg E$. Hence while generating pseudo-targets, \change{Equation~\ref{eq:pseudo-target}} %\pd{perhaps you mean Equation~\ref{eq:pseudo-target}??}
is also used to generate the pseudo-target truth intervals for when $E$ is false. 
% Due to the non-deterministic match semantics of PSI-L, described in Section~\ref{sec:psi-l_MatchSemantics}, the stretched portions of the targets intervals for its true state and false state overlap. This overlap creates confusion in computing the metrics for decision making at query nodes of the decision tree. We elaborate this in the following section.

\subsection{Effect of Pseudo-Targets on Decision Making}\label{sec:newMetrics}
%In the building of a traditional decision tree (as in Algorithm~\ref{algo:miner-immediate-relations}), at each decision node, statistical measures, $Mean_{\mathcal{T}_\mathcal{C}}(E)$ and $Error_{\mathcal{T}_\mathcal{C}}(E)$, are  computed consistently with respect to a single target, in this case $E$. 

%Here, our aim is to build {\em temporal sequences} of Boolean formulas that explain the target's truth.

At each decision node, we must decide which predicate best reduces the error in the resulting split, while simultaneously choosing a temporal position for the predicate in the $n$-length prefix sequence. We achieve the latter by choosing to test a predicate with each pseudo-target, to identify which pseudo-target (and therefore which position), given a possibly non-empty partial prefix, is most correlated with the predicate under test. The choice of predicate and position that gives the best correlation is then chosen. %We must still decide what measure is best to determine correlation. In the rest of this section, we describe two methods for computing error, the challenges involved, and demonstrate why one is superior to the other. 

At each query node of the decision tree, with constraint set $\mathcal{C}$, for target $E$, the following steps are carried out:
\begin{tcolorbox}
\begin{enumerate}[topsep=0.1em]\setlength\itemsep{0.1em}

\item Compute $Mean_{\mathcal{T}_\mathcal{C}}(\hat{E})$ and $Error_{\mathcal{T}_\mathcal{C}}(\hat{E})$, where $\hat{E}$ is the pseudo-target \textit{applicable} for constraint set $\mathcal{C}$. \label{step:computeNodeMeanError}

\item \label{step:computeGain}\textbf{For each $\langle P,i\rangle$, where $P\in \mathbb{P}, 1\leq i \leq n$, and $\langle P,i\rangle \notin \newchange{\mathcal{C}}$, $\langle \neg P,i\rangle \notin \newchange{\mathcal{C}}$}
\begin{enumerate}[topsep=0.1em]\setlength\itemsep{0.1em}

\item $\mathcal{C}_1 = \mathcal{C} \cup \langle P,i\rangle$, $\mathcal{C}_0 = \mathcal{C} \cup \langle \neg P,i\rangle$.

\item Compute $Mean_{\mathcal{T}_{\mathcal{C}_1}}(\tilde{E})$, $Error_{\mathcal{T}_{\mathcal{C}_1}}(\tilde{E})$, $Mean_{\mathcal{T}_{\mathcal{C}_0}}(\tilde{E})$, $Error_{\mathcal{T}_{\mathcal{C}_0}}(\tilde{E})$. Here, $\tilde{E}$ is a pseudo-target that may differ from $\hat{E}$.

\item Compute the gain for the choice $\langle P,i\rangle$, for constraint sets $\mathcal{C}_1$ and $\mathcal{C}_0$, with respect to $Error_{\mathcal{T}_C}(\hat{E})$ computed in Step 1.
\end{enumerate}

\item Report the arguments $\langle P^*,i^*\rangle$ that contribute the best gain from Step~\ref{step:computeGain}.
\end{enumerate}
\end{tcolorbox}

In the core steps of the above procedure, namely, Step~\ref{step:computeNodeMeanError} and Step~\ref{step:computeGain}, we compute the statistical measures of mean and error for pseudo-targets $\hat{E}$ and $\tilde{E}$. The definition of $Mean_{\mathcal{T}_\mathcal{C}}(\hat{E})$ and $Error_{\mathcal{T}_\mathcal{C}}(\hat{E})$ in Section~\ref{sec:NodeMetrics} assume that the true and false interval lists of $E$ are compliments of each other, however, for a pseudo-target $\hat{E}$ this is not true, hence the metrics cannot be directly applied. Furthermore, even though in each iteration of Step~\ref{step:computeGain}, a choice of $\langle P,i\rangle$ is made, it is unclear for which pseudo-target the measures must be computed. We first resolve the later and then address the computation of the statistical measures.

Note that in Step 2 (a), for a predicate $P$, once a pseudo-target position is determined, the split considers $P$ being true in one
% node
\change{branch}
and $P$ being false in the other, while the temporal position for $P$ and $\neg P$ remains the same for both child nodes of the split. 

\subsubsection{Choosing a Pseudo-Target for a Partial Prefix}
% The reader may recall that an assertion in PSI-L has the following syntax:
% \begin{mycenter}
% \begin{tabular}{r@{~~} c@{~~} l}
% $\mathcal{S}$ {\tt |->} $E$ & or & $\mathcal{S}$ {\tt |->} $\tau_1 ~ E$\\
% \end{tabular}
% \end{mycenter}
% where, the prefix sequence $\mathcal{S}$ of length $n$ is of the form $s_n ~\tau_{n} ~s_{n-1} ~\tau_{n-1} \ldots \tau_{1}~s_0$. In the prefix sequence each $s_i$ is a bucket at position $i$, a Boolean expression of PORVs and events. For a prefix sequence of maximum length $n$, at a query node with constraint set $\mathcal{C}$, some bucket positions may be empty (interpreted as the Boolean expression \textit{true}). 
% 
% In the case when $\mathcal{S}$ {\tt |->} $E$, the last sub-expression is $s_0$. The alternate syntax $\mathcal{S}$ {\tt |->} $\tau_1 ~ E$ is developed to take into account the cases wherein the last sub-expression is $s_i$, $i>0$. In such cases, the forward match of the sequence (Definition~\ref{def:forwardInfluenceMining}) must be further stretched by $i$ delay intervals to match with $E$. These delay intervals are coalesced into $\tau_1$. 

\change{
A path of a decision tree in the making is expanded by adding a constraint 
$\langle P, i \rangle$ into the constraint set $\mathcal{C}$ corresponding to the
path. The choice of the constraint, namely the predicate $P$ and its bucket $\sqcup_i$, is made by comparing the influence (on the target, $E$) of the prefix
sequences obtained by adding to $\mathcal{C}$ each such pair $\langle P, i \rangle$. 
At the heart of this approach are the metrics used to determine the goodness of
a partial prefix sequence in terms of its influence on the target $E$. Our
metrics are computed by stretching the target back in time as pseudo-targets. It is,
therefore, an important step to determine which of the pseudo-targets is to be used
for a given partial prefix sequence.}

\begin{proposition} \label{prop1}
The pseudo-target applicable for a given constraint set $\mathcal{C}$ is the smallest bucket position, $i$, among all non-empty buckets, $0\leq i\leq n$. \qed
\end{proposition}

\begin{example}\label{ex:pseudo-target-choice}
\change{For $n=3$, and delay resolution $k=0.4$, we consider the template of the form:
$\sqcup_3$ {\tt \#\#[0:0.4]} $\sqcup_2$ {\tt \#\#[0:0.4]} $\sqcup_1$ {\tt \#\#[0:0.4]} $\sqcup_0$ {\tt |->} {\tt E}. Suppose the given constraint set is $\mathcal{C}=\{\langle Q,3\rangle, \langle P,1\rangle\}$. The prefix-buckets are therefore, 
$\mathcal{B}_3(\mathcal{C})=\{Q\}$, $\mathcal{B}_2(\mathcal{C})=\{\}$, $\mathcal{B}_1(\mathcal{C})=\{ P\}$ and $\mathcal{B}_0(\mathcal{C})=\{\}$. 
%$\mathcal{I}_\tau(\mathcal{B}_3) = \{[0.9:4.2),[6.3:11.4),[14.5:18.5)\}$, and $\mathcal{I}_\tau(\mathcal{B}_2) = \{[4.3:6),[9.4:12.2),[13.2:14),[17.3:20)\}$. 
The partial prefix sequence that results from $\mathcal{C}$ is $\mathcal{S} \equiv$ $Q$ {\tt \#\#[0:0.8]} $P$. This then asserts $Q$ {\tt \#\#[0:0.8]} $P$ {\tt |-> \#\#[0:0.4]} $E$. }

\change{Note that the delays separating placeholders $\sqcup_3$, $\sqcup_2$, and $\sqcup_1$ have merged into the interval {\tt \#\#[0:0.8]} because $\sqcup_2$ is
empty. Also, since $\sqcup_0$ is empty, the delay between $\sqcup_1$ and $\sqcup_0$
separates $\mathcal{B}_1(\mathcal{C})$ from $E$.
In this case, therefore, evaluating $\mathcal{C}$ requires using pseudo-target $E^1$ as shown in Figure~\ref{fig:pseudo-targets}. In other words, $E^1$ represents the expression {\tt \#\#[0:0.4]}$E$.} \qed
\end{example}

\subsubsection{Adapting Statistical Measures for Pseudo-Targets}
Standard decision tree algorithms assume that classes are independent and that no data point in the data set belongs to more than one class. However, for a pseudo-target, this is not the case, because, the pseudo-target can (non-deterministically \change{in time}) be
true \change{and} false at some times (see Figure~\ref{fig:pseudo-targets}). Hence a traditional error computation, such as in Definition~\ref{def:error}, misrepresents the relationships that exist. 

Furthermore, at each decision node, we make two decisions, deciding which predicate to pick, and deciding which temporal position (pseudo-target) gives the best gain for the chosen predicate. The two decisions are dependent and must be made together. Pseudo-targets help us to learn the most influential temporal position for a predicate. 
% We do not treat all pseudo-targets as part of the set of classes, otherwise, for $n$ pseudo-targets there would then be $2\times (n+1)$ classes. 

Example~\ref{eg:meanErrorBroken} demonstrates the lacunae of using the traditional
definitions of mean and error from Section~\ref{sec:NodeMetrics}.

\begin{figure}[t]
\centering
\includegraphics[scale=0.8]{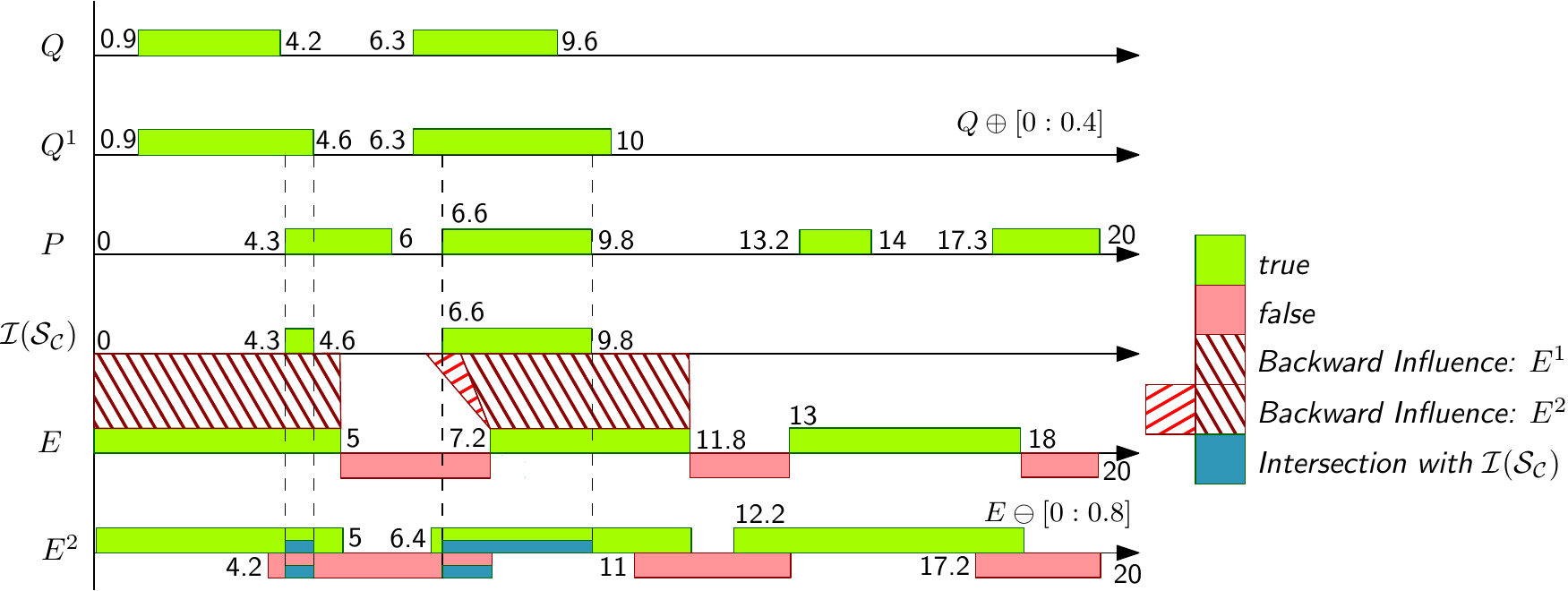}
\caption{\change{Predicates=\{Q,P,E\} and Pseudo-Targets (n=3, k=0.4) for $E$.  $\mathcal{C}=\{\langle Q,3\rangle, \langle P,2\rangle\}$ is a constraint set; while $\mathcal{I}(\mathcal{S}_{\mathcal{C}})$ is the set of end match influence time intervals.}} \label{fig:temporalSemantics}
\end{figure}

%\pd{COMMENT: The following example is good but somewhat convoluted. I suggest that you modify the example such that a property where $E$ is the consequent is missed (in this example, the one missed has $\neg E$ as the consequent. This may be possible by exchanging true and false intervals for $E$. Also it may not be necessary to keep both $H$ and $Q$. We can keep only the offending one and explain why the assertion is missed.Similarly for Example 7.}

\begin{example}\label{eg:meanErrorBroken}
\change{Consider the constraint set $\mathcal{C}=\{\langle Q,3\rangle, \langle P,2\rangle\}$, where the truth intervals of the predicates are shown in Figure~\ref{fig:temporalSemantics}. We demonstrate the computation of $Mean_{\mathcal{T}_{\mathcal{C}}}({E^2})$ and  $Error_{\mathcal{T}_{\mathcal{C}}}({E^2})$ for $\mathcal{C}$, following their standard definitions over time intervals, as given in Section~\ref{sec:NodeMetrics}.}
\begin{itemize}

\item \change{$\mathcal{S}_\mathcal{C}$ $\equiv$ $Q$ {\tt \#\#[0:0.4]} $P$. The influence set is computed following Definition~\ref{def:InfluenceSet}:
\begin{align*}
\mathcal{I}(\mathcal{S}_\mathcal{{C}}) & = \{[4.3:4.6),[6.6:9.8)\} \\ |\mathcal{I}(\mathcal{S}_{{\mathcal{C}}})| & = (4.6-4.3)+(9.8-6.6) = 3.5
\end{align*}}
\end{itemize}

\change{Since the smallest non-empty bucket is $\sqcup_2$ (containing $P$), we use pseudo-target $E^2$ in our computations. The mean and error with respect to $E^2$ are computed as follows:}
\change{\small
\begin{align*}
%----------------------FOR {C}---------------------------
\mu_{\mathcal{T}_{{\mathcal{C}}}}({E^2}) & = \frac{|\mathcal{I}_{{\mathcal{T}_{{\mathcal{C}}}}}(E^2) \cap \mathcal{I}(\mathcal{S}_{\hat{\mathcal{C}}})|}{|\mathcal{I}(\mathcal{S}_{{\mathcal{C}}})|} &
%-----------------------------------------------------
\mu_{\mathcal{T}_{{\mathcal{C}}}}({\neg E^2}) & = \frac{|\mathcal{I}_{{\mathcal{T}_{{\mathcal{C}}}}}(\neg E^2) \cap \mathcal{I}(\mathcal{S}_{{\mathcal{C}}})| }{|\mathcal{I}(\mathcal{S}_{{\mathcal{C}}})|}\\
%=====================================================
& = \frac{(0.3+3.2)}{(3.5)}  = \frac{(3.5)}{(3.5)}&
%----------------------------------------------
& = \frac{(0.3+0.8)}{(3.5)}  = \frac{(1.1)}{(3.5)}\\
%==============================================
& = {1.0} &
%-------------
& = {0.3143}\\
%=============
\epsilon_{\mathcal{T}_{{\mathcal{C}}}}({E^2}) & = -(0)-(-0.5238)& = 0.5238
\end{align*}
}
\change{Observe that for the prefix $\mathcal{S}_{\mathcal{C}}$, visually one observes no error with respect to \newchange{$E^2$} (indicated by the blue band overlayed on $E^2$). Hence, intuitively, the property $Q$ {\tt \#\#[0:0.4]} $P$ {\tt |-> \#\#[0:0.8]} $E$, having no error, should be reported.  However, since error exists with respect to {$\neg E^2$}, according to Definition~\ref{def:error} the property is would be set aside as requiring refinement, by adding predicates in some bucket.}\qed
\end{example}

From Example~\ref{eg:meanErrorBroken} it is clear that using the measures of mean and error from Section~\ref{sec:NodeMetrics}, it is possible to miss potential relations that may exist in the data presented in Figure~\ref{fig:temporalSemantics}. This is primarily due to the fact that the measures ignore that \change{there can be} an overlap of the truth intervals of a pseudo-target's true and false state. The measure of error counts the entropy contributed by the overlapping states twice. We introduce the definition of \textit{unified error} that takes this into account when computing the error.

\begin{definition}\label{def:jointError}
\textbf{Unified-Error$_{\mathcal{T}}(E^i,C)$:} For the target class $E^i$, the unified error, denoted as $UE_{\mathcal{T}_{\mathcal{C}}}(E^i)$, for the trace $\mathcal{T}$ constrained by $\mathcal{C}$ is defined as follows:
\[\pushQED{\qed}
	UE_{\mathcal{T}_{\mathcal{C}}}(E^i) =	\epsilon_{\mathcal{T}_{\mathcal{C}}}(E^i) 
	+
	\mu_{\mathcal{T}_{\mathcal{C}}}(E^i \wedge \neg E^i) \times log_2(\mu_{\mathcal{T}_{\mathcal{C}}}(E^i \wedge \neg E^i)) 
\qedhere\popQED\]
\end{definition}

While computing the unified error, the term $\mu_{\mathcal{T}_{\mathcal{C}}}(E^i \wedge \neg E^i) \times log_2(\mu_{\mathcal{T}_{\mathcal{C}}}(E^i \wedge \neg E^i)$ represents the entropy in the region of the overlap. By adding this term, we effectively ensure that the entropy from the overlap is considered only once when computing the unified error.

\begin{lemma}
For a \newchange{PSI-L property}, the unified error, $UE_{\mathcal{T}_{\mathcal{C}}}(E^i)$, for constraint set $\mathcal{C}$ and consequent pseudo-target $E^i$ is zero if and only if 
\change{$\mathcal{C}$ decides $E^i$ or $\neg E^i$, that is, there are no counter-examples for $\mathcal{S}_{\mathcal{C}} ${\tt |->}$ E^i$ or there are no counter-examples for $\mathcal{S}_{\mathcal{C}} ${\tt |->}$ \neg E^i$.}
\end{lemma}
\begin{proof}
We make the assumption that $|\mathcal{I}(\mathcal{S}_\mathcal{C})| > 0$, the sum of lengths of truth intervals for the expression describing $\mathcal{C}$ is non-zero. 
For the case where $i = 0$, since $\mu_{\mathcal{T}_{\mathcal{C}}}(E \wedge \neg E)= 0$,  $UE_{\mathcal{T}_{\mathcal{C}}}(E^0) =	\epsilon_{\mathcal{T}_{\mathcal{C}}}(E)$. We therefore, only consider the case when $i>0$.

\noindent\textbf{Part A:}
Consider the property $\mathcal{S}_{\mathcal{C}} ${\tt |->}$ E^i$, where $|\mathcal{I}_{\mathcal{T}}(E^i) \cap \mathcal{I}(\mathcal{S}_\mathcal{C})| > 0$. 
Assume that there are counter-examples for   $\mathcal{S}_{\mathcal{C}} ${\tt |->}$ E^i$. 
The counter-examples would introduce a non-zero time interval when $\mathcal{S}_{\mathcal{C}}$ is true and $E^i$ is false.
Let one of the counter-example time-intervals be the interval $[a:b)$, where $b>a$. 
Therefore, $|\mathcal{I}_{\mathcal{T}}(\neg E^i) \cap \mathcal{I}(\mathcal{S}_\mathcal{C})|>0$, hence $\mu_{\mathcal{T}_\mathcal{C}}(\neg E^i,C) > 0$, and $log_2(\mu_{\mathcal{T}_\mathcal{C}}(\neg E^i,\mathcal{C})))<0$ (the mean is bounded between 0 and 1). Similarly the term, $\mu_{\mathcal{T}_\mathcal{C}}(E^i,\mathcal{C}) \times log_2(\mu_{\mathcal{T}_\mathcal{C}}(E^i,\mathcal{C})) < 0$, since $|\mathcal{I}_{\mathcal{T}}(\neg E^i) \cap \mathcal{I}(\mathcal{S}_\mathcal{C})| > 0$ . Hence, from Definition~\ref{def:error}, $Error_\mathcal{T}(E^i,\mathcal{C}) > 0$. For $i>0$, $\mu_{\mathcal{T}_\mathcal{C}}(E^i \wedge \neg E^i)>0$, however $\mu_{\mathcal{T}_\mathcal{C}}(E^i \wedge \neg E^i)< \mu_{\mathcal{T}_\mathcal{C}}(\neg E^i)$. Therefore, $\epsilon_{\mathcal{T}_{\mathcal{C}}}(E^i) > \mu_{\mathcal{T}_{\mathcal{C}}}(E^i \wedge \neg E^i) \times log_2(\mu_{\mathcal{T}_{\mathcal{C}}}(E^i \wedge \neg E^i))$, and hence $UE_{\mathcal{T}_{\mathcal{C}}}(E^i) > 0$ . 

Consider the property $\mathcal{S}_{\mathcal{C}} ${\tt |->}$ E^i$, where $|\mathcal{I}_{\mathcal{T}}(E^i) \cap \mathcal{I}(\mathcal{S}_\mathcal{C})| = |\mathcal{I}(\mathcal{S}_\mathcal{C})|$. Let $\Upsilon = \mathcal{I}_{\mathcal{T}}(E^i) \cap \mathcal{I}_{\mathcal{T}}(\neg E^i)$ represent the overlapping intervals of $E^i$ and $\neg E^i$.
Since,  $|\mathcal{I}_{\mathcal{T}}(E^i) \cap \mathcal{I}(\mathcal{S}_\mathcal{C})| = |\mathcal{I}(\mathcal{S}_\mathcal{C})|$, $\mathcal{I}(\mathcal{S}_\mathcal{C}) \subseteq \mathcal{I}_{\mathcal{T}}(E^i)$. Therefore, $\mathcal{I}_{\mathcal{T}}(\neg E^i) \cap \mathcal{I}(\mathcal{S}_\mathcal{C}) \subseteq \mathcal{I}_{\mathcal{T}}(E^i)$ and hence $\mathcal{I}_{\mathcal{T}}(\neg E^i) \cap \mathcal{I}(\mathcal{S}_\mathcal{C}) = \Upsilon \cap \mathcal{I}(\mathcal{S}_\mathcal{C})$. 
For $i>0$, $\Upsilon \neq \emptyset$, and $0 \leq |\Upsilon \cap \mathcal{I}(\mathcal{S}_\mathcal{C})| < |\mathcal{I}(\mathcal{S}_\mathcal{C})|$. Hence, for the case where $|\Upsilon \cap \mathcal{I}(\mathcal{S}_\mathcal{C})|=0$, $\mu_{\mathcal{T}_\mathcal{C}}(\neg E^i,C) = \mu_{\mathcal{T}_\mathcal{C}}(E^i \wedge \neg E^i,C) =  0$. On the other hand, when $|\Upsilon \cap \mathcal{I}(\mathcal{S}_\mathcal{C})|>0$, $\mu_{\mathcal{T}_\mathcal{C}}(\neg E^i,C) = \mu_{\mathcal{T}_\mathcal{C}}(E^i \wedge \neg E^i,C)$ and the term for $\neg E$ in $\epsilon_{\mathcal{T}_{\mathcal{C}}}(E^i)$ balances out the term for $\Upsilon$ from $UE_{\mathcal{T}_{\mathcal{C}}}(E^i)$, yielding a \textit{unified error} of zero.

The proof for when the property if of the form $\mathcal{S}_{\mathcal{C}} ${\tt |->}$ \neg E^i$ is identical.

\noindent\textbf{Part B:}
Conversely, if the unified error is non-zero, then $0<\mu_{\mathcal{T}_\mathcal{C}}(E^i,\mathcal{C})<1$ and $0<\mu_{\mathcal{T}_\mathcal{C}}(\neg E^i,\mathcal{C})<1$. If any one term was zero or one, the same argument, using overlapping truth intervals, $\Upsilon$, from \textit{Part A} of the proof would render the error zero. From Definition~\ref{def:mean}, $0<\mu_{\mathcal{T}_\mathcal{C}}(E^i,\mathcal{C})<1$,  $0<\mu_{\mathcal{T}_\mathcal{C}}(\neg E^i,\mathcal{C})<1$, and hence  $\mathcal{I}_{\mathcal{T}}(E^i) \cap \mathcal{I}(\mathcal{S}_\mathcal{C}) \neq \mathcal{I}(\mathcal{S}_\mathcal{C})$ and  $\mathcal{I}_{\mathcal{T}}(\neg E^i) \cap \mathcal{I}(\mathcal{S}_\mathcal{C}) \neq \mathcal{I}(\mathcal{S}_\mathcal{C})$. Therefore, there exists a non-empty interval $[a_1:b_1)$, $b_1>a_1$ where $\mathcal{S}_\mathcal{C}$ is true and $E^i$ is false, and there exists a non-empty interval $[a_2:b_2)$, $b_2>a_2$ where $\mathcal{S}_\mathcal{C}$ is true and $E^i$ is true, respectively representing counter-example intervals for $\mathcal{S}_{\mathcal{C}} ${\tt |->}$ E^i$ and $\mathcal{S}_{\mathcal{C}} ${\tt |->}$ \neg E^i$.
\end{proof}

\begin{example}\label{ex:jointEntropy}
\change{In Example~\ref{eg:meanErrorBroken} on considering the overlap component and computing the Unified Error for the constraint set $\mathcal{C}$, we get the following:
{\small
\begin{align*}
\mu_{\mathcal{T}_{{\mathcal{C}}}}({E^2 \wedge \neg E^2}) & = \frac{|\mathcal{I}_{{\mathcal{T}_\mathcal{C}}}(E^2\wedge\neg E^2) \cap \mathcal{I}(\mathcal{S}_\mathcal{C})| }{|\mathcal{I}(\mathcal{S}_\mathcal{C})|} &
%-------------------------------------------------
UE_{\mathcal{T}_{\mathcal{C}}}(E^2) & = \epsilon_{\mathcal{T}_{{\mathcal{C}}}}({E^2}) + 0.3143\times(log_2(0.3143))\\
%==============================================
& = \frac{(0.3+0.8)}{(3.5)} = \frac{(1.1)}{(3.5)} ~~ & 
%-------------------------------------------------
& = 0.5238~+~(-0.5238)\\  
%=============
& = {0.3143} &  
%-------------
& = {0}
\end{align*}}
From this example, we see that the use of Unified Error reveals the association in the data that was earlier set aside when using the error metric of Definition~\ref{def:error}.
}
\qed
\end{example}

\change{
The non-determinism resulting out of reasoning with time intervals presents another fundamental difference with the traditional decision tree learning algorithm. Each
node in the traditional decision tree splits the data points into disjoint subsets,
which is leveraged by the information gain metric. As the following example shows,
some time points in the time-series data may be relevant for both branches of a node
in the decision tree.
}

%Similar to the challenges posed by the overlapping of classes, when considering a predicate and temporal position for splitting a node, the two nodes that result from such a split can have overlaps in the data points they represent. 

\begin{example}\label{ex:jointGain}
\change{Consider the split of the data-set resulting from considering $Q$ at bucket $\sqcup_3$ to augment the constraint set $\{\langle P,2\rangle\}$ to obtain ${\mathcal{C}}_0=\{\langle \neg Q,3\rangle, \langle P,2\rangle\}$ and ${\mathcal{C}}_1=\{\langle Q,3\rangle, \langle P,2\rangle\}$. Due to the non-deterministic semantics of the temporal operator {\tt \#\#[$a:b$]}, the forward influence (end-matches) of ${\mathcal{C}}_0$ and ${\mathcal{C}}_1$ overlap.
\begin{flalign*}
\mathbb{F}(\mathcal{S}_{{\mathcal{C}}_0},\mathcal{W}_2^3) & =  \{[4.3:6),[6.6:6.7),[9.6:9.8),[13.2:14),[17.3:20)\}\\
\mathbb{F}(\mathcal{S}_{{\mathcal{C}}_1},\mathcal{W}_2^3) & =  \{[4.3:4.6),[6.6:9.8)\}
\end{flalign*}
The two sets overlap at all time-points in the intervals $[4.3:4.6)$, $[6.6:6.7)$ and $[9.6:9.8)$. 
%Also note that the union of the two sets is the set $\mathcal{I}(\mathcal{S}_{\{\langle P,2\rangle\}})$. 
Hence, the sum of the weights in the definition of Gain in Definition~\ref{def:gain} exceeds $1$, that is:  
\[ \frac{|\mathbb{F}(\mathcal{S}_{{\mathcal{C}}_0},\mathcal{W}_2^3)| }{ |\mathcal{I}(\mathcal{S}_{\{\langle P,2\rangle\}})|} +
\frac{|\mathbb{F}(\mathcal{S}_{{\mathcal{C}}_1},\mathcal{W}_2^3)| }{ |\mathcal{I}(\mathcal{S}_{\{\langle P,2\rangle\}})|} >1 \]
Furthermore, the influence set of intervals for the constraint set $\{\langle P,2\rangle\}$ is: 
\[ \mathcal{I}(\mathcal{S}_{\{\langle P,2\rangle\}})= \{[4.3:6),[6.6:9.8),[13.2:14),[17.3:20)\} \] 
Since the temporal position of $Q$ is earlier than that of $P$, $\mathbb{F}(\mathcal{S}_{{\mathcal{C}}_0},\mathcal{W}_2^3) \cup \mathbb{F}(\mathcal{S}_{{\mathcal{C}}_1},\mathcal{W}_2^3) = \mathcal{I}(\mathcal{S}_{\{\langle P,2\rangle\}})$, however, this need not always be the case. If $Q$ is placed later than $P$ in the sequence, the forward influence list can change substantially, since it would now be contained in the interval list of $Q$.}
\qed 
\end{example}

Due to both, potential overlapping data-points between the child nodes of a split and the data-points after the split being potentially different from those of the parent node, we present a revised definition of gain, called \textit{unified gain}.

\change{
\begin{definition}\label{def:jointGain}
\textbf{Unified-Gain:} The gain (improvement in error) of choosing to split on $P \in \mathbb{P}$, at bucket position $b$, given the existing constraint set $\mathcal{C}$, at a node having error $\epsilon$, is as follows:
\[ UG(\epsilon,\mathcal{C},P,b) = \epsilon -  
\alpha(\mathcal{C},P,b)
\times 
\epsilon_{\mathcal{T}_{\mathcal{C}\cup\{\langle P,b\rangle\}}}(E)
- 
\alpha(\mathcal{C},\neg P,b)
\times 
\epsilon_{\mathcal{T}_{\mathcal{C}\cup\{\langle \neg P,b\rangle\}}}(E)\]
\noindent where, \nolinebreak
\[\pushQED{\qed}\alpha(\mathcal{C},P,b) = %{ |\mathbb{F}(\mathcal{S}_{\mathcal{C}\cup\{\langle P,b\rangle\}},\mathcal{W}_i^j)|
\frac{
|\mathcal{I}_{\mathcal{T}}(\mathcal{S}_{\mathcal{C}\cup\{\langle P,b\rangle\} })|
}
{ 
|\mathcal{I}_{\mathcal{T}}
(\mathcal{S}_{\mathcal{C}\cup\{\langle P,b\rangle\}})|
+
|\mathcal{I}_{\mathcal{T}}
(\mathcal{S}_{\mathcal{C}\cup\{\langle \neg P,b\rangle\}})|
}
\qedhere\popQED\]
%}{ |\mathbb{F}(\mathcal{S}_{\mathcal{C}\cup\{\langle P,b\rangle\}},\mathcal{W}_i^j)| + |\mathbb{F}(\mathcal{S}_{\mathcal{C}\cup\{\langle \neg P,b\rangle\}},\mathcal{W}_i^j)|}\]
%and $i$ and $j$ are respectively the smallest and largest non-empty bucket indexes in $\mathcal{C}\cup\{\langle P,b\rangle\}$.\pd{$\leftarrow$ There is something odd here.
%I dont see any $i$ and $j$ in the definition.}
\end{definition}
}
\subsection{Variations of \newchange{Unified} Gain with Temporal Positions}\label{sec:gainVtime}
Given an existing constraint set $\mathcal{C}$ (possibly empty - at the root of the tree), a candidate predicate $P$, and its temporal position $i$ in the sequence, the Gain is dependent on the resulting pseudo-target association. 

\begin{theorem}\label{thm:gainVtemporalPosition}
The Unified Gain, of placing predicate $P$ in bucket index $i$, \change{is monotonically non-decreasing with increasing values of $i$}. %\pd{$\leftarrow$ Since you have the plateaus, do you mean "monotonically non-decreasing"?}
\end{theorem}
\begin{proof}
\change{Following Proposition~\ref{prop1}}, the pseudo-target association is dependent on the smallest index among the non-empty buckets in the constraint lists resulting from the split. The split results in two nodes, one with the constraint list $\mathcal{C} \cup \{\langle P,i\rangle\}$  and the other with the constraint list $\mathcal{C} \cup \{\langle \neg P,i\rangle\}$. The smallest non-empty bucket is, therefore, the minimum of $i$ and the index of the smallest index non-empty bucket in $\mathcal{C}$. Let the index of the smallest index non-empty bucket in $\mathcal{C}$ be $b$. Let $\hat{b}$ be the lesser of $i$ and $b$. The length of interval list for the pseudo-target, $|\mathcal{I}(E^{\hat{b}})|$, becomes larger with larger values of $\hat{b}$ (From Definition~\ref{def:pseudoTarget}). 

\begin{figure}[t]
	\centering
	\includegraphics[scale=1]{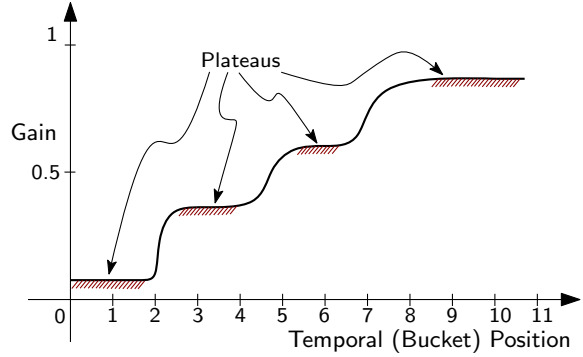}
	\caption{Gain for a candidate predicate, with varying temporal distances from the target predicate.}\label{fig:gainVtarget}
\end{figure}

Initially, at the root of the decision tree, the constraint set, $\mathcal{C}$, is empty. Hence, as $\hat{b}$ increases, a larger fraction of the truth of $P$ would be covered by the pseudo-target $E^{\hat{b}}$. This would cause the quantum of counter-examples for both true and false states of $P$ and it's association with $E$ to reduce, leading to a reduced entropy, and therefore an increase in Gain. Hence as $\hat{b}$ increases, the Gain would monotonically increase \change{or} remain stagnant at a plateau. This is depicted in Figure~\ref{fig:gainVtarget}.

When the constraint set, $\mathcal{C}$, is non-empty, i.e. there is at least one element $\langle Q,j \rangle \in \mathcal{C}$, the following cases arise:
\begin{enumerate}

\item \textbf{[$b>i$]} (Figure~\ref{fig:movingPredicatePosition_bGTi}) \textbf{:} The end-match of the sequence $C \cup \langle P,i\rangle$ is computed by adding $[0:(b-i)\times k]$ to the end-match of $\mathcal{C}$. While maintaining $i<b$, as $i$ increases, i.e. $i$ is a bucket further from the target, but closer to the end-match of $\mathcal{C}$ (depicted in Figure~\ref{fig:movingPredicatePosition_bGTi}), the length of the end-match of the resulting constraint-set $C \cup \langle P,i\rangle$, monotonically decreases. For larger differences between $b$ and $i$ (smaller values of $i$), the end-match is wider, hence the potential for counter-examples (with respect to the target) is higher. Therefore, as $i$ increases, the entropy monotonically decreases. 

\begin{figure}[t]
\centering
\includegraphics[scale=1]{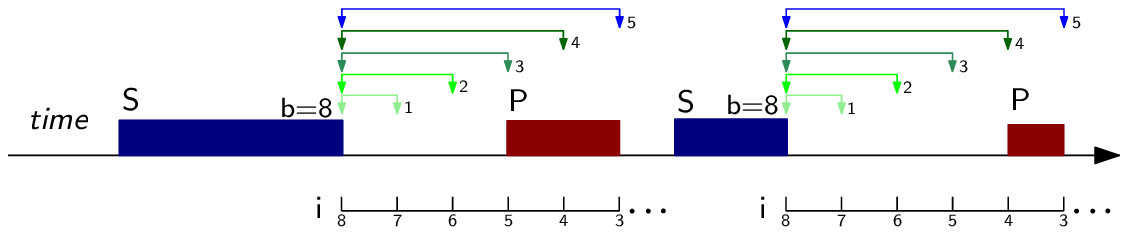}
\caption{Relative Position of Predicate $P$ with respect to the end-match interval list for a constraint set, with the end-match represented as $\mathcal{S}$. The index of the minimum index non-empty bucket, $b$, is 8. The index of the bucket where $P$ may be placed in the sequence is $i$.}\label{fig:movingPredicatePosition_bGTi}
\end{figure}

\begin{figure}[t]
\centering
\includegraphics[scale=1]{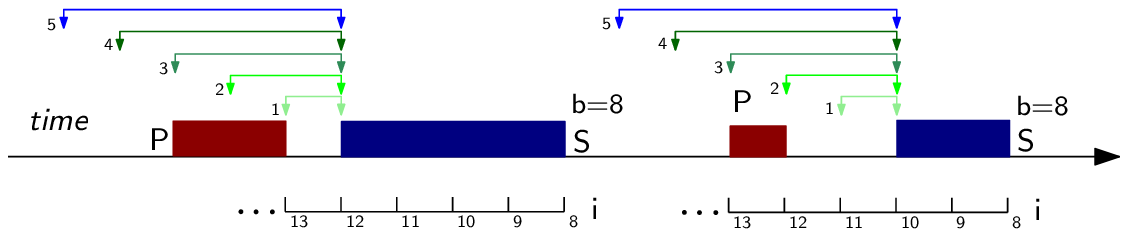}
\caption{Relative Position of Predicate $P$ with respect to the end-match interval list for a constraint set, with the end-match represented as $\mathcal{S}$. The index of the minimum index non-empty bucket, $b$, is 8. The index of the bucket where $P$ may be placed in the sequence is $i$.}\label{fig:movingPredicatePosition_bLEQi}
\end{figure}

\item \textbf{[${b} \leq i$]} (Figure~\ref{fig:movingPredicatePosition_bLEQi}) \textbf{:} For constraint-set $\mathcal{C}$, either $\mathcal{B}_i$ is empty or non-empty. When $P$ is placed in bucket $i \geq b$, the new bucket $\mathcal{B}'_{i} = \mathcal{B}_i \cup \{P\}$. We have the following two cases:

\begin{enumerate}
\item \textbf{$\mathcal{B}_i$ is non-empty:} On adding $P$ to bucket $i$, let The interval list of $P$ is intersected with the interval list of $\mathcal{B}_i$ and therefore the resulting bucket has the interval list $\mathcal{I}(\mathcal{B}'_i) \subseteq \mathcal{I}(\mathcal{B}_i)$.

\item \textbf{$\mathcal{B}_i$ is empty:} Let $h$ be the smallest bucket index, $h>i$, such that $\mathcal{B}_h \neq \emptyset$, and let $j$ be the largest bucket index, $j<i$, such that $\mathcal{B}_j \neq \emptyset$. If such a index $h$ does not exist, then $i$ is the largest index non-empty bucket in the prefix, while the case that $j$ does not exist is not possible under the present case ($b\leq i$).

If $\mathcal{B}_h \neq \emptyset$, then the forward influence of $\mathcal{B}_h$ on $\mathcal{B}_j$ is computed as follows:
\[\Theta = (\mathcal{I}(\mathcal{B}_h) \oplus [0:(h-j)\times k] )\cap \mathcal{I}(\mathcal{B}_j)\] 
However, on adding $P$ in bucket $i$, $h>i>j$, the forward influence of  $\mathcal{B}_h$ on $\mathcal{B}_j$ is computed as follows: \[
\Omega= ((\mathcal{I}(\mathcal{B}_h) \oplus [0:(h-i)\times k] )\cap \mathcal{I}(\mathcal{B}_i)) \oplus [0:(i-j)\times k] )\cap \mathcal{I}(\mathcal{B}_j)\]

Hence, $\Omega \subseteq \Theta$. The potentially reduced forward influence on $\mathcal{B}_j$ similarly propagates toward reducing the end-match for the sequence having $P$ in bucket $i$. A reduced end-match has lesser potential for entropy, and therefore  yields a higher gain, or leaves the gain unchanged. \qedhere
\end{enumerate}
\end{enumerate}\end{proof}

\begin{comment}

\begin{figure}[t]
    \centering
    \begin{subfigure}[t]{0.45\textwidth}
		\centering
		\includegraphics[scale=0.45]{sridhar/xgtec16.eps}
		\caption{$x\geq 16$ }
		\label{fig:temporal:p1}
	\end{subfigure}%
	%\\~\\
	~~~~~~~~
	\begin{subfigure}[t]{0.45\textwidth}
		\centering
		\includegraphics[scale=0.45]{sridhar/xgtec23.eps}
		\caption{$x\geq 23$ }
		\label{fig:temporal:p2}
	\end{subfigure}
	\\~~~~~ %\hspace{-5em}
	\hspace{-2em}\begin{subfigure}[t]{0.45\textwidth}
		\centering
		\includegraphics[scale=0.45]{sridhar/ygtec22.eps}
		\caption{$y\geq 22$}
		\label{fig:temporal:p3}
	\end{subfigure}%
	%\\~\\
	~~~~~~~~
	\begin{subfigure}[t]{0.45\textwidth}
		\centering
		\includegraphics[scale=0.45]{sridhar/ygtec71.eps}
		\caption{$y \geq 71$}
		\label{fig:temporal:p4}
	\end{subfigure}
    
    %\frame{\includegraphics[scale = 0.4]{sridhar/xgtec16.eps}}
    %\frame{\includegraphics[scale = 0.4]{sridhar/xgtec23.eps}}
    %\frame{\includegraphics[scale = 0.4]{sridhar/ygtec22.eps}}
    %\frame{\includegraphics[scale = 0.4]{sridhar/ygtec71.eps}}
    \caption{Variation in gain (vertical axis) with respect to increasing bucket positions (horizontal axis) for a vehicular position dataset shown in Figure~\ref{fig:traffic}.}
    \label{fig:bucket_vs_gain}
\end{figure}
\end{comment}

%We experimentally demonstrate Theorem~\ref{thm:gainVtemporalPosition} using four predicates, each chosen to depict four varieties of gain curves. The curves are shown in Figure~\ref{fig:bucket_vs_gain}.

\subsection{A Miner for Prefix Sequences}\label{sec:psi-miner-algo}

The prefix sequence inference mining algorithm is presented as Algorithm~\ref{algo:nk-psi-miner-temporal}. The length of the sequence $n$, and the delay resolution $k$ are meta-parameters of the algorithm. Every choice of $n$ and $k$ yields a different instance of the algorithm. The choice of values of the meta-parameters must come from the domain. The algorithm learns a  decision tree for PSI properties for the truth set $\widehat{\mathbb{I}}_\mathcal{T}(\mathbb{P})$ for a choice of $n$ and $k$. 

\begin{algorithm}[t!]
	{\small
		\KwIn{Truth Set $\mathbb{I}_{\mathcal{T}}(\mathbb{P})$ for trace $\mathcal{T}$, Predicate List $\mathbb{P}$, Target $E$, Constraint List $\mathcal{C}$}%, Temporal split position $sp$.}
		\KwOut{PSI-L properties $\mathbb{A}$, structured as the decision tree.}
		
		\lIf{stoppingCondition($\mathbb{I}_{\mathcal{T}}(\mathbb{P}),\mathbb{P}, E, \mathcal{C}$)}{  
		%\lIf{stoppingCondition($\mathbb{I}_{\mathcal{T}}(\mathbb{P}),\mathbb{P}, E, \mathcal{C},sp$)}{  
		return\label{line:stop}}
		
		$b \leftarrow$ Smallest non-empty bucket position in $\mathcal{C}$;\label{line:getSmallestBucket}\\
		
		$P_{best} \leftarrow \emptyset$; 
		$i_{best} \leftarrow -1$; 
		$g_{best} \leftarrow 0$;\\
		
		\For{$0\leq i < n$}{\label{line:loopTarget}
			\ForAll{$P \in \mathbb{P}, \langle P,i\rangle \notin \mathcal{C}$}{\label{line:loopPredicate}
				%$\hat{b} = min(b,i)$;\label{line:getNewSmallestBucket}\\
				$g \leftarrow UG( UE_{\mathcal{T}_{\mathcal{C}}}(E^b),\mathcal{C},P,i)$;\label{line:computeGain}
				%- 
				%\epsilon_{\mathcal{T}_{\mathcal{C}\cup{\langle P,i \rangle}}}(E^{\hat{b}}) -
				%\epsilon_{\mathcal{T}_{\mathcal{C}\cup{\langle \neg P,i \rangle}}}(E^{\hat{b}})$;\label{line:computeGain}
				
				\lIf{$g > g_{best}$}{\label{line:computeBestGain}
					$P_{best} \leftarrow P$;
					$i_{best} \leftarrow i$;
					$g_{best} \leftarrow g$
				}
			}
		}	
	%{\tt nk-PSI-Miner}{$(\mathbb{I}_{\mathcal{T}}(\mathbb{P}),\mathbb{P}, E, \mathcal{C}\cup\{\langle P_{best},i_{best}\rangle \},i_{best})$};\label{line:leftTree}\\
	\lIf{$i_{best}<0$}{  
		return\label{line:noGain}}
	{\tt nk-PSI-Miner}{$(\mathbb{I}_{\mathcal{T}}(\mathbb{P}),\mathbb{P}, E, \mathcal{C}\cup\{\langle P_{best},i_{best}\rangle \})$};\label{line:leftTree}\\
	{\tt nk-PSI-Miner}{$(\mathbb{I}_{\mathcal{T}}(\mathbb{P}),\mathbb{P}, E, \mathcal{C}\cup\{\langle \neg P_{best},i_{best}\rangle \})$};\label{line:rightTree}
	%{\tt nk-PSI-Miner}{$(\mathbb{I}_{\mathcal{T}}(\mathbb{P}),\mathbb{P}, E, \mathcal{C}\cup\{\langle \neg P_{best},sp_{best}\rangle \},sp_{best})$};\label{line:rightTree}
	
	\caption{{\tt nk-PSI-Miner}: Mining n-length, k-resolution Prefix Sequences}
	\label{algo:nk-psi-miner-temporal}
}\end{algorithm}

In Algorithm~\ref{algo:nk-psi-miner-temporal}, Line~\ref{line:stop} tests the current node for termination. One of the criteria for termination is that the node is homogeneous with respect to $E$, that is, the error at the node is zero. Other stopping conditions are described later in Section~\ref{sec:stoppingPruningOverfitting}. 

In Line~\ref{line:getSmallestBucket} the smallest non-empty bucket index is identified from the constraint set $\mathcal{C}$, and is used in Line~\ref{line:computeGain} to compute the error for $\mathcal{C}$. The loop at Line~\ref{line:loopTarget} iterates over every pseudo-target position, while Line~\ref{line:loopPredicate} chooses a predicate from the predicate alphabet $\mathbb{P}$. Any predicate and pseudo-target position combination already present in $\mathcal{C}$ are ignored in Line~\ref{line:loopPredicate}. The Unified Gain for the choice of predicate and pseudo-target is computed in Line~\ref{line:computeGain} and the best gain, and its associated arguments are determined in Line~\ref{line:computeBestGain}. The computation of Unified Gain uses the smallest non-empty bucket index for the choice of pseudo-target, with respect to which the Unified Entropy and Unified Gain are computed. The loop beginning at Line~\ref{line:loopTarget} implicitly chooses the smallest bucket position for the \textit{``best predicate"} that has the maximum gain. As shown in Section~\ref{sec:gainVtime}, the gain plateaus at various points and we choose the earliest point on the plateau that has the highest gain. Line~\ref{line:noGain} terminates exploration if for the current node, no remaining predicate and bucket position pair can improve the solution. If a best predicate and bucket position pair is found, lines~\ref{line:leftTree} and~\ref{line:rightTree} branch on new constraint sets. In the following section, we describe an approach for translating a decision tree constructed using Algorithm~\ref{algo:nk-psi-miner-temporal} into properties in PSI-L.

\subsection{From Decision Trees to PSI-L Formulae}\label{sec:psi-l-to-formulae}
The nodes in the decision tree at which the error is zero are the leaf nodes of the tree and have homogeneous data with respect to the truth of $E$. We call these nodes \textit{PSI nodes} and the labels along the path from a PSI node to the root (the constraint set $\mathcal{C}$) form a \newchange{PSI-L property} template. A template consists of predicates and their relative sequence position from the target. Concretizing the PSI template involves computing the relative positions of predicates from each other. We do this by grouping predicates that fall in the same relative temporal position into a \textit{bucket}, and then compute tight time delays that separate buckets in order of their temporal distance from the target $E$. The computation of tight separating intervals between buckets assumes an {\em any-match} (may) semantic.

At a PSI node, the set of constraints $\mathcal{C}$ is known. Recall, that a PSI node is a node at which the entropy is zero. \change{We use the notation $\mathcal{B}_i$, to denote the set of predicates having influence on the target with a step size of $[0:i\times k]$, while $s_i$ is the conjunction of predicates in $\mathcal{B}_i$. The \newchange{PSI-L property} has one of the following forms:}

\begin{center}
	\begin{tabular}{ll}	
		\change{$s_n ~ \tau_n ~s_{n-1} ~\tau_{n-1} \ldots \tau_1~ s_0 ${\tt |->} $E$} & {, when $\mathcal{B}_0 \neq \emptyset$}\\
		\change{$s_n ~ \tau_n ~s_{n-1} ~\tau_{n-1} \ldots \tau_2~ s_1 ${\tt |-> }$ \tau_1 ~E$} &, otherwise
	\end{tabular}
\end{center}

We wish to compute \textit{tight} intervals, $\tau_i$, $1 \leq i \leq n$. 

\begin{definition}
\textbf{Tight delay separation:} For trace $\mathcal{T}$ and constraint set $\mathcal{C}$, a delay separation $[a:b]$, $a\leq b$, between buckets $\mathcal{B}_i$ and $\mathcal{B}_j$, $i>j$, is tight with respect to the match semantics of Definition~\ref{def:seqExprMatch} iff the following  conditions hold,
\begin{itemize}\pushQED{\qed}
\item Maximality: $\exists t_i,t'_i \in \mathcal{I}_{\mathcal{T}_\mathcal{C}}(\mathcal{B}_{i}),$ $t_j, t'_j \in \mathcal{I}_{\mathcal{T}_\mathcal{C}}(\mathcal{B}_{j})$: $t_i+a = t_j$ and $t'_i+b = t'_j$
\item Left-Tight: If $a>0$, $\forall t\in [0:a);~ t_i \in \mathcal{I}_{\mathcal{T}_\mathcal{C}}(\mathcal{B}_i);~t_j\in \mathcal{I}_{\mathcal{T}_\mathcal{C}}(\mathcal{B}_j)$: $t_i+t \neq t_j$ 
\item Right-Tight: If b$<$(i$-$j)$\times$k, $\forall t\in($b,(i$-$j)$\times$k$]$; $t_i \in \mathcal{I}_{\mathcal{T}_\mathcal{C}}(\mathcal{B}_i);~t_j\in \mathcal{I}_{\mathcal{T}_\mathcal{C}}(\mathcal{B}_j)$: $t_i+t \neq t_j$  \qedhere \popQED
\end{itemize}
\end{definition}

We use the standard interval widening operation for a set of intervals $\mathcal{I}$.
\begin{definition}
\textbf{Widening over a set of intervals}: For a set of intervals $\mathcal{I}$, the widening over the intervals in $\mathcal{I}$ is defined to be the following interval:
\[\pushQED{\qed}
	\mathcal{W}(\mathcal{I}) = [ \min_{I\in\mathcal{I}}l(I)~:~ \max_{I\in\mathcal{I}} r(I) ]\qedhere \popQED
\]
\end{definition}

In the remainder of this section we describe how tight delay separations, or simply separations, are computed.
\change{
For the property $\mathcal{S} ${\tt |->}$ E^i$ we compute separations between buckets. The interval set for the $j^{th}$ bucket, $\mathcal{B}_j$, for the constraint set $\mathcal{C}$, of intervals that take part in the prefix $\mathcal{S}_\mathcal{C}$, is as follows:
\begin{mycenter}[0.1em]
	$\dot{\mathcal{I}}_{\mathcal{T}_\mathcal{C}}(\mathcal{B}_j) = \mathds{B}(\mathcal{S},\mathbb{W},j)$%\{I~|~ I= \bigcap_{P_j \in \mathcal{B}_i} I_j$ for some $I_j \in \mathcal{I}_{\mathcal{T}_\mathcal{C}}(P_j)  \}$
\end{mycenter}
Note, that for the bucket with the smallest index (closest to the target), the intervals taking part in the prefix $\mathcal{S}_\mathcal{C}$ are the end-match intervals.

We compute the separation interval, $Sep_\mathcal{C}(\mathcal{B}_j,E)$, for $j>0$, of $\mathcal{B}_j$ from $E$ as follows:
\begin{itemize}[topsep=0em]\setlength\itemsep{0em}
\item For each interval $\dot{I}_{\mathcal{B}_j}$ in $\dot{\mathcal{I}}_{\mathcal{T}_\mathcal{C}}(\mathcal{B}_j)$, and each truth interval $I_{E^i}$ of $E$, compute the influence of $\dot{I}_{\mathcal{B}_j}$ on $I_{E}$. For some $\dot{I}_{\mathcal{B}_j}$ and $I_{E}$, this is given as, $\mathcal{F}^+ =  (\dot{I}_{\mathcal{B}_j}\oplus[0:j\times k])\cap I_{E}$. 
\item  For a given $\dot{I}_{\mathcal{B}_j}$ and $I_{E}$, and therefore a value of $\mathcal{F}^+$, we compute the separation of $\mathcal{F}^+$ from $\dot{I}_{\mathcal{B}_j}$ as, $\mathcal{D} = \mathcal{F}^+ \ominus \dot{I}_{\mathcal{B}_j}$.
\item We compute $\mathcal{D}$ over all combinations of $\dot{\mathcal{I}}_{\mathcal{B}_{j}}$ and $I_{E}$, and widen over the resulting set of intervals.
\item It is possible for $\mathcal{D}$ to be larger than $[0:j\times k]$ due to the semantics of the Minkowski operators. We, therefore, bound the widened set by $[0:j\times k]$. In our implementation, intervals in an interval set are sorted according to their timestamps. For any two intervals $[\alpha:\beta]$ and $[a:b]$, we compute the Minkowsky difference between them only if $\beta>=a$.
\end{itemize}

The separation, $Sep_\mathcal{C}(\mathcal{B}_j,E)$, for $j>0$, of $\mathcal{B}_j$ from $E$ as follows:
\begin{center}
	$Sep_\mathcal{C}(\mathcal{B}_j,E) = \mathcal{W}(I~| I = ((\dot{I}_{\mathcal{B}_j}\oplus[0:j\times k])\cap I_E) \ominus \dot{I}_{\mathcal{B}_j}) \cap [0:j\times k]$\\ 
	\flushright{for all $\dot{I}_{\mathcal{B}_j} \in \dot{\mathcal{I}}_{\mathcal{T}_\mathcal{C}}(\mathcal{B}_j)$ and $I_E \in \mathcal{I}_{{\mathcal{T}}}(E))$}  
\end{center}

Note that $Sep_\mathcal{C}(\mathcal{B}_j,E)$ computes the separation between a bucket $\mathcal{B}_j$ and the target $E^i$. To form a prefix sequence, delay intervals separating adjacent buckets must be computed.
The separation $\tau_j$ between $\mathcal{B}_j$ and $\mathcal{B}_{j-1}$ is iteratively computed as follows:
\[
\tau_j = 
	\begin{cases}
	%Sep_\mathcal{C}(\mathcal{B}_j,B) \ominus Sep_\mathcal{C}(\mathcal{B}_{j-1},E) & 0 < j \leq n\\
	Sep_\mathcal{C}(\mathcal{B}_i,E) & j=i\\
	Sep_\mathcal{C}(\mathcal{B}_j,B_{j-1}) & 0 \leq i < j \leq n\\
	%Sep_\mathcal{C}(\mathcal{B}_j,E) \ominus Sep_\mathcal{C}(\mathcal{B}_{j-1},E) & i=1, \mathcal{B}_0 \neq \emptyset\\
	%Sep_\mathcal{C}(\mathcal{B}_i,E) & i=1, \mathcal{B}_0 = \emptyset
	\end{cases}
\]

Recall that $i$ is the smallest index of a non-empty bucket in $\mathcal{S}_\mathcal{C}$, and that some buckets may be empty, and the separation between adjacent non-empty buckets $\mathcal{B}_j$ and $\mathcal{B}_l$ is computed in a similar manner. 

\begin{proposition}
Two non-empty buckets $\mathcal{B}_j$ and $\mathcal{B}_l$ are adjacent iff $\forall m \in (j:l),~\mathcal{B}_m = \emptyset$. We define the predicate Adj(j,l) to be true iff $\mathcal{B}_j$ and $\mathcal{B}_l$ are adjacent.
\end{proposition}

\noindent The separation \newchange{in terms of} adjacent buckets is then computed as follows:
\[
\tau_j = 
	\begin{cases}
	%Sep_\mathcal{C}(\mathcal{B}_j,E) \ominus Sep_\mathcal{C}(\mathcal{B}_{l},E) & Adj(j,l),~ 0\leq l < j \leq n\\
	Sep_\mathcal{C}(\mathcal{B}_j,B_l) & Adj(j,l),~ 0\leq l < j \leq n\\
	Sep_\mathcal{C}(\mathcal{B}_j,E) & j=i%\forall_{ 0\leq j <l} ~\neg Adj(j,l)
	\end{cases}
\]

The first statement computes the separation between two non-empty adjacent buckets, and the second indicates that the $j^{th}$ bucket is the non-empty bucket in the prefix sequence having the smallest index. Therefore, we use the separation between the $j^{th}$ bucket and the target as-is. The separation between the last non-empty bucket and the target would appear as a delay constraint in the consequent. For a prefix with the largest index bucket being $h$, the second to the smallest being $m$, and the smallest being $i$, the prefix sequence would be of the form, $s_h$ $\tau_h$ $\ldots$ $\tau_{m}$ $s_i$ {\tt |->} $\tau_i$ $E$.

\begin{example}
Consider the set of intervals for two non-empty buckets, $\mathcal{B}_2$ and $\mathcal{B}_3$, that take part in the match of the prefix, and the interval set for target $E$, to be given as follows:
\begin{align*}
\dot{\mathcal{I}}_{\mathcal{T}_\mathcal{C}}(\mathcal{B}_2) & = \{[4.3:4.6),[6.6:9.8)\} \\ \dot{\mathcal{I}}_{\mathcal{T}_\mathcal{C}}(\mathcal{B}_3) & = \{[3.9:4.2),[6.3:6.4)\} \\
\mathcal{I}_{\mathcal{T}}(E) & = \{[4.6,5),[6.6:6.9),[13:18)\}%\\
%\mathcal{I}_{\mathcal{T}}(E^2) & = \{[3.8:.5),[5.8:11.8),[12.2:18)\} 
\end{align*}

The delay resolution $k=0.4$. The separation intervals are computed as follows:
\begin{itemize}[topsep=0em]\setlength\itemsep{0em}
    \item For $\mathcal{B}_2$ we compute $Sep_\mathcal{C}(\mathcal{B}_2,E)$ as follows:
    \begin{align*}
    \mathcal{F}^+ & = \dot{\mathcal{I}}_{\mathcal{T}_\mathcal{C}}(\mathcal{B}_2)\oplus[0:2\times 0.4]\cap \mathcal{I}_{\mathcal{T}}(E)\\
    & = \{[4.3:5.4),[6.6:10.6)\} \cap {\mathcal{I}}_{\mathcal{T}}(E)\\ 
    & = \{[4.6:5),[6.6:6.9)\} \\
    (\mathcal{F}^+ \ominus I_{\mathcal{B}_2}) \cap \{[0:0.8]\} & = \{[0:0.7),[0:0.3)\}\\
    \mathcal{W}(\{[0:0.7),[0:0.3)\}) & = [0:0.7]
    \end{align*}
    \item For $\mathcal{B}_3$ we compute $Sep_\mathcal{C}(\mathcal{B}_2,\mathcal{B}_3)$ as follows:
    \begin{align*}
    \mathcal{F}^+ & = (\dot{\mathcal{I}}_{\mathcal{T}_\mathcal{C}}(\mathcal{B}_3)\oplus[0:1\times 0.4])\cap \dot{\mathcal{I}}_{\mathcal{T}_\mathcal{C}}(\mathcal{B}_2)\\
    & = \{[3.9:4.2),[6.3:6.4)\} \oplus [0:0.4] \cap \dot{\mathcal{I}}_{\mathcal{T}_\mathcal{C}}(\mathcal{B}_2)\\ 
    & = \{[4.3:4.6),[6.6:6.8)\} \\
    (\mathcal{F}^+ \ominus I_{\mathcal{B}_3}) \cap \{[0:0.4]\} & = \{[0.1:0.4),[0.2:0.4)\}\\
    \mathcal{W}(\{[0.1:0.4),[0.2:0.4)\}) & = [0.1:0.4]
    \end{align*}
\end{itemize}

The property obtained, on integrating the delay separation time intervals, is as follows:\linebreak {\tt $s_3$ \#\#[0.1:0.4] $s_2$ |-> \#\#[0:0.7] E}. \qed
\end{example}
}
%\pd{Please consider giving an example here. It will be great if one of the paths of the decision tree from Section 2 can be used here, and then the time separation computation is shown on the property.} 

\section{Quality of Mined PSI-L Properties} \label{sec:ranking}
The decision tree learned by Algorithm~\ref{algo:nk-psi-miner-temporal} can compute several prefixes. It is important to rank these in terms of those that are likely to be causal relations and those that are not. Furthermore, it is also important to understand how mined prefixes are related to the trace. 

We measure the goodness of a \newchange{PSI-L property} $\mathcal{S} ${\tt |->}$ E^i$, where $i$ is the smallest non-empty bucket index in $\mathcal{S}$, using heuristic metrics of Support, and Correlation. We also measure how much of the trace is covered for the set of PSI properties generated.

\change{Recall that $E^i$ is the $i^{th}$ pseudo-target, that is the target's truth stretched back in time by an amount of $i\times k$, where $k$ is the delay resolution. In the definitions, while the support only considers $\mathcal{S}$, the correlation and coverage deal with $E^i$. To be consistent, we re-enforce in all definitions that we relate the truth of the prefix $\mathcal{S}$ forward in time with target $E$, and hence write $\mathcal{S} ${\tt |->}$ E^i$.
%\pd{You need to recall what $E^i$ means, and explain why we define the goodness of $\mathcal{S} ${\tt |->}$ E^i$ instead of $\mathcal{S} ${\tt |->}$ E$, and whether they are the same.} 

%When using data-mining techniques, one of the major problems most techniques suffer from is their inability to judge

\begin{definition}
\textbf{Support:} For a property $\mathcal{S} ${\tt |->}$ {E^i}$, the quantum of time for which $\mathcal{S}$ is true in the trace $\mathcal{T}$ is the support of $\mathcal{S} ${\tt |->}$ E^i$.
\[Support(\mathcal{S} \text{\tt |->} E^i) = \frac{|\mathcal{I}_\mathcal{T}(\mathcal{S})|}{||\mathcal{T}||}\]
\noindent where $\mathcal{I}_\mathcal{T}(\mathcal{S})$ is the influence interval list for the sequence $\mathcal{S}$ computed according to Definitions~\ref{def:InfluenceSet} and~\ref{def:forwardInfluenceMining}, while $||\mathcal{T}||$ is the length of the trace given in Definition~\ref{def:hybridTrace}.
\qed
\end{definition}
A high support for \newchange{PSI-L property} $\mathcal{S} ${\tt |->}$ {E^i}$ is indicative of $\mathcal{S}$ being frequently true in the trace. However, a low support need not indicate that the property is incorrect, and could indicate a corner case behaviour. 

\begin{definition}
\textbf{Correlation:} For the assertion $\mathcal{S} ${\tt |->}$ {E^i}$, correlation indicates how much of $E^i$'s truth is associated with $\mathcal{S}$, that is the quantum of the consequent, $E^i$'s truth, that the antecedent $\mathcal{S}$ contributes to. 
\[\pushQED{\qed}Correlation(\mathcal{S} \text{\tt |->} {E^i}) = \frac{
|(\mathcal{I}_\mathcal{T}(\mathcal{S}) \oplus [0:i\times k]) \cap \mathcal{I}_\mathcal{T}(E^i)|}{|\mathcal{I}_\mathcal{T}(E^i)|}\qedhere\popQED\] 

%\pd{In the above, it will be great if add an intermediate term using $\mathcal{I}(E^i)$ in the numerator instead of $\mathcal{I}(E)$ and then have the RHS as shown.}
\end{definition}

\begin{definition}
\textbf{Trace Coverage:}\label{sec:coverage}
Trace coverage quantifies the fraction of the trace that is explained by the properties generated by the miner.

Given a trace $\mathcal{T}$ of length $\mathcal{L}$, and the mined property set $\mathbb{A}$, the coverage interval list of $\mathcal{T}$ by $\mathbb{A}$, denoted $Cov(\mathcal{T},\mathbb{A})$, is computed as follows:
\begin{center}
	$Cov(\mathcal{T},\mathbb{A}) = \bigcup_{
	(\mathcal{S}\text{{\tt |->}} \tilde{E}^i)\in\mathbb{A}
	}	
	(\mathcal{I}_\mathcal{T}(\mathcal{S})\oplus [0:i\times k]) \cap \mathcal{I}_\mathcal{T}(\tilde{E})$
\end{center}
where $\tilde{E}\in\{E,\neg E\}$. The percentage of coverage is then given by $\frac{|Cov(\mathcal{T},\mathbb{A})|}{\mathcal{L}}\times 100$.
\qed
\end{definition}
}
%\pd{Can we give an example here demonstrating these metrics?}
\change{
\begin{example}
We compute the three metrics introduced in this section for  the interval sets from the prefix and target from Figure~\ref{fig:temporalSemantics}. 
\begin{align*}\pushQED{\qed} 
Support(\mathcal{I}_\mathcal{T}(\mathcal{S}) \text{\tt |->} E^2) &= \frac{3.5}{20} = 17.5\%\\
\\
Correlation(\mathcal{I}_\mathcal{T}(\mathcal{S}) \text{\tt |->} E^2) &= \frac{|\{[4.3:5.4),[6.6:10.6)\}\cap\{ [0:5),[6.4:11.8),[12.2:18)\}|}{|\{ [0:5),[6.4:11.8),[12.2:18)\}|} \\
& = \frac{|\{[4.3:5),[6.6:10.6)\}|}{|\{ [0:5),[6.4:11.8),[12.2:18)\}|}\\
& = \frac{4.7}{16.2} = 29.01\%\\
\\
Cov(\mathcal{T},\mathbb{A}) & = \mathcal{I}_\mathcal{T}(\mathcal{S}) \oplus [0:2\times 0.8] \cap \mathcal{I}_\mathcal{T}(E)\\
& = \{[4.3:5),[6.6:10.6)\}\\
\\
Percentage~of~coverage & =  \frac{|\{[4.3:5),[6.6:10.6)\}|}{20}\\
& = \frac{4.7}{20} = 23.5\%\qedhere
\popQED
\end{align*}
%\qed
\end{example}
}
%\section{Mining Postfix Sequences}

\section{Stopping Conditions, Over-fitting and Pruning}\label{sec:stoppingPruningOverfitting}

%\subsection{Stopping Criteria}
%\textbf{Stopping Criteria:}
While building the decision tree, for prefixes, there are two {\em stopping conditions} that we employ to terminate the growth of the tree. 
\begin{enumerate}
\item \textbf{Purity of a Node:} When the constraint set $\mathcal{C}$ completely determines the truth of the target $E$, the decision tree node is 100\% pure and further growth is terminated. A node with constraint set $\mathcal{C}$, and minimum bucket position $b$, is considered pure if the unified error at the node is zero, that is $UE_{\mathcal{T}_{\mathcal{C}}}(E^b) = 0$.
\item \textbf{Depth Constraints:} It is also possible to define a depth threshold, $\alpha_d$, and stop the tree from growing if the length of the current exploration path crosses $\alpha_d$.
\end{enumerate}

%\subsection{Over-fitting and Pruning}

Decision trees are known to suffer from problems of {\em over-fitting}. Over-fitting involves fitting a learned model so closely to the data, that the model becomes specific in explaining the peculiarities in the data, making it overly defined and specific to the data over which learning is performed. This is exceptionally problematic with data that is discrete. In this case, when dealing with dense real-time data, and arbitrary resolution of time, depending on the time precision in the data, over-fitting can lead to a large number of properties being generated, leaving the designer with an overwhelmingly large number of prefix properties and rendering the mining as an ineffective aid to designers. 

To prevent over-fitting, we employ pruning and abstraction mechanisms controlled via meta-parameters. The measures employed are as follows: 
\begin{enumerate}[topsep=0em]\setlength\itemsep{0em}
\item \textbf{Using a support threshold:} A threshold $\alpha_s$ is defined to indicate the minimum support below which a prefix is being over-fit to the data. If a node with constraint set $\mathcal{C}$ has a support below $\alpha_s$, further splitting of the node is terminated. 
\change{Note that, due to the dense interpretation of time, in some situations, a low support is expected when explaining targets concerning corner case behaviours.} 
\item \textbf{Using a correlation threshold:} A threshold $\alpha_c$ is defined to indicate the minimum correlation below which a prefix is being over-fit to the target. If a node with constraint set $\mathcal{C}$ has a correlation below $\alpha_c$ with its associated target, further splitting of the node is terminated. 
\item \textbf{Prefix Grouping:} We use interval arithmetic to mine prefix sequences. This allows overly constrained prefixes to be grouped into sequences that share common event orderings. Hence a prefix mined by  Algorithm~\ref{algo:nk-psi-miner-temporal} may be representative of an infinite number of distinct prefixes that have similar event orderings but dissimilar in delay intervals separating every adjacent pair of events.
\item \textbf{Constraint Set Limits:} Limits on the depth of the decision tree impose an implicit upper bound on the number of constraints used to build a prefix. 
\end{enumerate}

%\indent When multiple time-series are used for training, the prefixes learned from one can be used to perform pruning of the decision tree learned from another if the prefix sequences infer opposing target truths. Beyond this, for a single time-series pruning is performed a-priori during the growth of the tree, using the stopping conditions.  

The \newchange{constraints on tree-depth} $\alpha_d$, \newchange{support} $\alpha_s$ and \newchange{correlation} $\alpha_c$, are treated as meta-parameters of the decision tree learning algorithm. 
\change{
\section{Handling Multiple Traces}\label{sec:multipleTraces}
In this article, for simplicity, all metrics have been defined for a single time-series. When considering a set $\mathbb{T}$ of time-series, any pair of time-series may share a common time-line. For instance, consider two autonomous vehicles {\tt V1} and {tt V2} that have their state recorded over time. The recording for {\tt V1} begins at 09:48:07AM and is recorded up to 04:02:23PM. The recording for {\tt V2} begins at 08:32:30AM and is recorded up to 04:32:43PM. Hence given a time $t$ in the duration from 09:48:07AM to 04:02:23PM, it is possible that at $t$, we observe two conflicting state entries for a vehicle. Considering truth intervals over predicates, therefore, at any given time, due to differences in the state between {\tt V1} and {\tt V2}, it is possible that a predicate has conflicting values of truth. To resolve this, it is not possible to simply use a timestamp offset and concatenate the two time-series. The concatenation in time could introduce unintended temporal relationships in the data. It is therefore important to treat each time-series as being independent. In this section, we redefine core metrics used in our learning framework for dealing with multiple time-series. %\pd{Give an example -- say the time-series for each car travelling to office at 9:00 AM are sharing a common time line. Also since properties are dealing with relative time intervals and not absolute time, is it equivalent to concatenate the traces into a single trace as far as the miner is concerned?} 

In our framework, the definition of $Mean(.)$ is a fundamental metric, and it is sufficient to re-define this metric to ensure that evidence is considered from all time-series. To handle a set of time-series, the definition of mean is extended as follows:
\begin{definition}\label{def:mean:MultipleTraces}
\textbf{Mean$_{\mathbb{T}_\mathcal{C}}$(E):} For the target class $E$, the proportion of time that $E$ is true in  the set of traces $\mathbb{T}$  constrained by $\mathcal{C}$ :
\[
Mean_{\mathbb{T}}(E,\mathcal{C}) = \frac{ \sum_{\mathcal{T}\in\mathbb{T}}|\mathcal{I}_{\mathcal{T}}(E) \cap \mathcal{I}(\mathcal{S}_\mathcal{C})| }{ \sum_{\mathcal{T}\in\mathbb{T}} |\mathcal{I}(\mathcal{S}_\mathcal{C})|}
\]
For convenience, $\mu_{\mathbb{T}_\mathcal{C}}(E)$ may be used in place of $Mean_{\mathbb{T}}(E,\mathcal{C})$. The mean is not defined when $\sum_{\mathcal{T}\in\mathbb{T}} |\mathcal{I}(\mathcal{S}_\mathcal{C})| = 0$. 
\qed
\end{definition}

Similarly, the normalization factor $\alpha$ used in Definition~\ref{def:jointGain} is extended to consider multiple traces as follows:
\[\alpha(\mathcal{C},P,b) = %{ |\mathbb{F}(\mathcal{S}_{\mathcal{C}\cup\{\langle P,b\rangle\}},\mathcal{W}_i^j)|
\frac{
\sum_{\mathcal{T}\in\mathbb{T}}|\mathcal{I}_{\mathcal{T}}
(\mathcal{S}_{\mathcal{C}\cup\{\langle P,b\rangle\}})|
}{ \sum_{\mathcal{T}\in\mathbb{T}}(|\mathcal{I}_{\mathcal{T}}(\mathcal{S}_{\mathcal{C}\cup\{\langle P,b\rangle\}})|+|\mathcal{I}_{\mathcal{T}}(\mathcal{S}_{\mathcal{C}\cup\{\langle \neg P,b\rangle\}})|)
}
\]

Using the measures in this section to compute unified entropy and unified gain and associated metrics allows us to incorporate information from multiple time-series.
}

\section{Experimental Results}\label{sec:experiments}
We use a selection of examples to demonstrate the utility of PSI-Miner. The miner was used on a standard laptop with a 2.40GHz Intel Core i7-5500U CPU % with two cores each supporting two parallel threads 
with 8GB of RAM.%\pd{$\leftarrow$ Why do you mention the number of threads. Does your miner use multiple threads?} 

For each example, we choose meta-parameters $n$, the number of \change{intervals in the antecedent}, and $k$, the initial time delay between buckets. The target event, being explained, and the predicate alphabet used for mining are also known. The prefixes learned are validated against the data \change{and prior knowledge that was not available to the miner}.

\change{
\begin{example}
%\pd{In this example, you need to clarify that separate sets of traces are provided to PSI-Miner, one route at a time, and that the first three properties are learned from those sets, and the fourth one from the union of those sets. It is important to mention that sets of traces can be used as a conjoined single trace. Also, why does it choose only D as the antecedent, and not the other way point predicates mentioned in Section 2? Do we have other properties over events like delays?}
This example focuses on PSI-Miner's ability to learn and generalize timing intervals across multiple traces. We use position information from multiple vehicles in Town-X (from Section~\ref{sec:motivation}, Figure~\ref{fig:traffic}) to learn which of the routes from location {\tt D} to {\tt A} is the fastest. 

There are three routes, namely $DA_1$, $DA_2$ and $DA_3$, in the direction from {\tt D} to {\tt A}. We pick time-series of nine vehicles, three from each route. We then run PSI-Miner separately for each route, on its set of three time-series, providing predicates for all way-points.  A delay-resolution of $k=2min$ and a maximum sequence length of $n=15$ are used in the experiments. Since a vehicle may spend very little time at the source/destination in relation to total time, we allow small support percentages.

\begin{table}[th]
	\centering
	\small
	\begin{tabular}{clcc}
		\toprule
		\multicolumn{4}{c}{\textbf{Time to travel from D to A}} \\
		\midrule
		\textbf{Route} &\textbf{Property} & \textbf{Support (\%)} & \textbf{Correlation (\%)}	\\
		\midrule
		%$DA_1$ & {\tt D|-> \#\#[15.73:16] E} & 33{\small$\times 10^{-5}$} & 100 \\
		$DA_1$ &  {\tt D|-> \#\#[26.90:28] A} & 78.56{\small$\times 10^{-2}$} & 100 \\
		$DA_2$ & {\tt D|-> \#\#[14.21:16] A} & 26.42{\small$\times 10^{-2}$} & 100 \\
		$DA_3$ & {\tt D|-> \#\#[18.82:20] A} & 37.66{\small$\times 10^{-2}$} & 100 \\
		All & {\tt D|-> \#\#[14.21:28] A} & 51.2{\small$\times 10^{-2}$} & 100 \\
		
		\bottomrule
	\end{tabular}
	\caption{PSI-L Properties describing the time to travel from D to A, along three paths and a summary property for all paths. Time intervals are measured in minutes, with decimals interpreted as a fraction of a minute. }\label{table:traffic-D-A}
\end{table}

The first three properties in Table~\ref{table:traffic-D-A} describe the time to travel along the three different routes. Note that PSI-Miner picked only one predicate, the one for location {\tt D} in the antecedent. This is because, for instance, for route $DA_1$, {\tt D}, at the bucket position of 14 (14min$\times$2 = 28min), was sufficient to construct a property. We observe that route $DA_2$ is the fastest available route; vehicles leaving {\tt D} reach {\tt A} from between 14min:12sec to 16min, while route $DA_1$ takes the longest time, that is between 26min:54sec to 28min. 
%for  We first learn how long it takes for vehicles to move from location {\tt D} to {\tt A} in Figure~\ref{fig:traffic} for individual routes and use learned timings to determine which one is the fastest.  PSI-Miner evaluates time-series from multiple vehicles to determine movement times on each route separately. 

We also evaluate if PSI-Miner is able to correctly generalize this timing information when given time-series from all routes. We provide PSI-Miner traces from all three routes simultaneously and use the same values of meta-parameters $k$ and $n$. We observe from the fourth property that it is correctly able to generalize across all traces and learn the minimum and maximum times for travel correctly.
\qed
\end{example}
}

\change{
\begin{example}\label{ex:disease}
Disease has been reported among passengers arriving by ship. The origin point of all passengers is the same, however, the routes the ships take may differ. Routes may share common way-points and paths. A map of these movements is shown in Fig.~\ref{fig:diseaseControl}.  In the map, passengers arrive at two locations, one is the labelled {\tt SOURCE}, and the other is a labelled {\tt HARBOUR}. We hope to discover locations that could be disease hot-spots. For each passenger, their movement data is tagged as risky or non-risky. A non-risky tag is used on passenger movements not carrying disease, while movements of passengers reporting disease symptoms are tagged as risky. 
\begin{figure}
    \centering
    \includegraphics[width=4in]{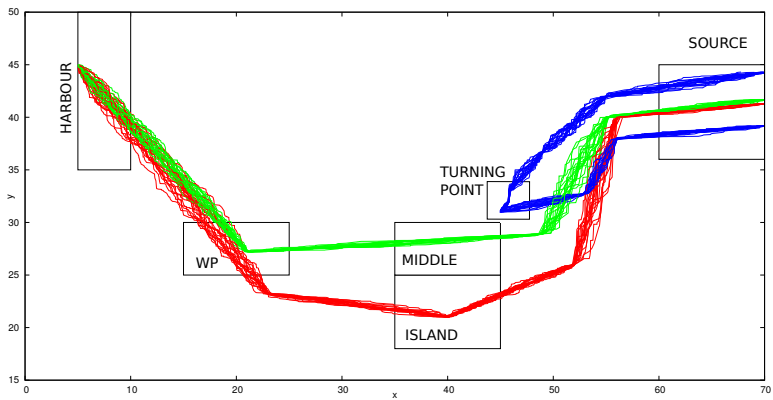}
    \caption{Passenger Travel Surveillance}
    \label{fig:diseaseControl}
\end{figure}

We assume that we are given as predicates locations for way-points ships pass through or stop at along their route. In Fig.~\ref{fig:diseaseControl}, predicates for way-points are marked with rectangles and labelled. Position information from 100 passengers is analyzed using PSI-Miner. It is known that a ship passes through at most five intermediate way-points. We, therefore, use a sequence length of $n=5$. On average, the time to move between way-points is known to be 70mins. We use a time delay of 70mins between events in the sequence. We examine using the target {\tt RISKY}, representing passengers reported having the disease. We use predicates for way-points and the target. The following prefix sequences are learned:

\begin{mycenter}[0em]
{\tt \small ISLAND \#\#[0:140] !TURNING-POINT |-> \#\#[0:70] RISKY}
\end{mycenter}
\hspace*{\fill}{Support = 17.63\%
Correlation = 52.5\%}

\begin{mycenter}[0em]
{\tt \small !ISLAND \&\& MIDDLE \#\#[0:135.23] !TURNING-POINT |-> \#\#[0:70] !RISKY}
\end{mycenter}
\hspace*{\fill}{Support =  7.53\%, Correlation = 11.33\%}

From the prefixes learned, visiting the island has a 52\% correlation with passengers marked risky (that are diagnosed with the disease). Those not visiting the island but passing through the middle are non-risky, with a correlation of 11.33\%. 

The predicate {\tt !TURNING-POINT} is true for all positions other than within the rectangular region marked as {\tt TURNING-POINT}, which is why it appears in all prefixes. However, we also want to know what happens to ships turning back? We explore further using the predicates {\tt TURNING-POINT} and {\tt RISKY} and learn the following prefixes:

\begin{mycenter}[0em]
{\tt \small TURNING-POINT |-> \#\#[0:70] !RISKY}
\end{mycenter}
\hspace*{\fill}{Support = 18.98\%, Correlation = 28.58\%}

In a time-series, for a risky (or non-risky) passenger, every time-position record is labelled as risky (or non-risky). Each prefix can explain (correlate with) only time-points after the prefix is satisfied, possibly a small portion of the time-series, that is labelled risky (or non-risky). Therefore the prefixes reported do not have a high correlation. Since, we never report a property unless it has 100\% confidence, in this case, the correlation values are acceptable. The time-delay of {\tt \#\#[0:70]} appears in the consequent of the PSI-L properties for the same reason.

From the mined properties, we learn that passengers visiting the island carry the disease, whereas passengers in ships turning around or passing through the middle region without visiting the island do not have disease symptoms. Authorities may then curtail visits to the island and inform travellers accordingly.
\qed
\end{example}
}

\begin{comment}
\begin{figure}[t]
	\centering
	\includegraphics[scale=0.3]{traffic-markup.png}
	\caption{Multi-Route 2-dimensional traffic: Blue lines indicate object movements, paths indicated in blue, green and pink, sources/destinations are marked as black boxes, way-points are indicated as red circles.}\label{fig:trafficMovement}
\end{figure}
\end{comment}

%Two industrial transistor netlist circuit models were simulated and the traces produced were analyzed by PSI-Miner. We explore the application of {\tt PSI-Miner} on these traces in Examples~\ref{ex:ldo1} and~\ref{ex:ldo2}. In the examples, {\tt PSI-Miner} is used with meta-parameter values for sequence length $N=15$, and delay resolution $K=10^{-6}s$.
\change{
\begin{example}\label{ex:ldo1}
	A low-dropout voltage regulator (LDO), is a voltage regulator used  to accurately maintain stable voltages for devices containing micro-electronics, such as processing units with varying power levels. During the design of the regulator, simulations produce behaviours that can be difficult to understand. We use PSI-Miner on the simulation of an LDO circuit to understand its operation. We use PSI-Miner to analyze the simulation depicted in Figure~\ref{fig:ldo}. Since the occurrence of a circuit event can span short periods of time, we use a low support threshold ($10^{-4}\%$). Time intervals in the prefixes are measured in seconds.
	
	\begin{figure}
		\centering
		\includegraphics[scale=0.18]{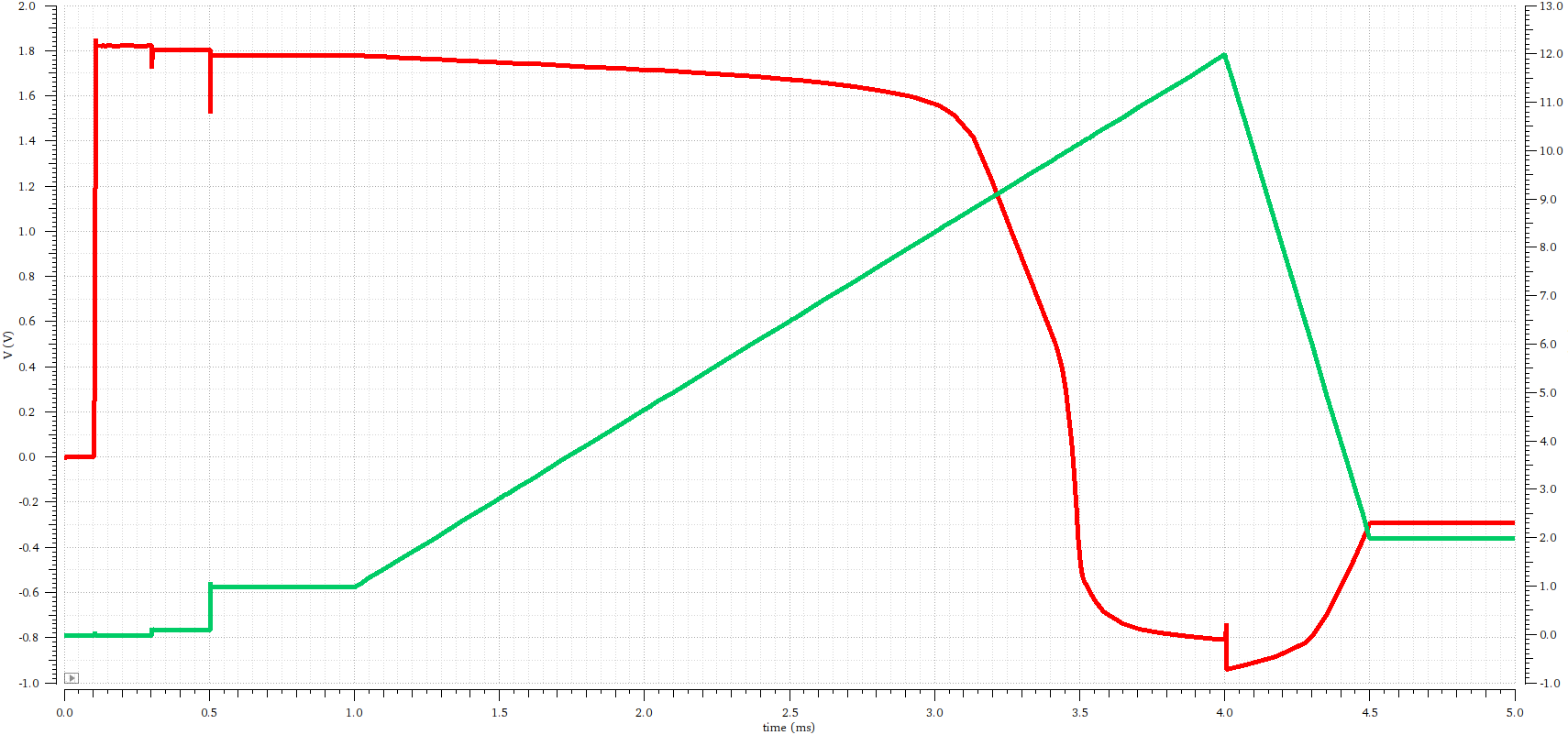}
		\caption{Waveform of voltage (volts) in red, and current (mA) in green, versus time (ms) depicting the behaviour of a low-dropout voltage regulator circuit.}\label{fig:ldo}
	\end{figure}
	
	%mputing resources such as explanations for rated voltage levels and various voltage threshold crossings (such as the 90\% crossing) are mined for the trace depicted in Figure~\ref{fig:ldo1}. In the figure, the red plot describes the voltage at the output of the LDO. It initially rises to its rated level of 1.8V and slowly falls. As the voltage drops, the current, the green plot, starts rising at a sharper pace. Beyond a point the voltage starts falling as the current continues rising. The voltage never rises back to its rated value again and both voltage and current are maintained at low threshold levels. The predicate alphabet (expert inputs) used by PSI-Miner along with aliases for Boolean expressions over predicates is shown in Table~\ref{table:predAlphabetLDO}. Properties mined are presented using the aliases in the table.

	The following are a selection of the properties that are mined:	
	
	\noindent``When at 10\% of the rated voltage ({\tt\small VLowerBand}: {\tt \small 0.18<=v<=0.19}), while not in a state of short circuit ({\tt \small InShrtCkt}: {\tt \small i>=0.0085}), then sometime in the next 4.06$\mu s$ to 4.09$\mu s$ the terminal voltage \textbf{rises to} 90\% of its rated value ({\tt\small VUpperBand}: {\tt \small 1.62<=v<=1.63})."
	\begin{mycenter}[0em]
	\noindent {\tt \small !InShrtCkt \&\& VLowerBand 
		|-> \#\#[4.06e-06:4.09e-06] VUpperBand}
	\end{mycenter}
	\hspace*{\fill}(Support 0.0002\%, Correlation: 0.023\%)\\
	The property has a very low correlation because the prefix containing the lower band (10\%) crossing lasts for a very small amount of time. However, this property is significant, and provides insight into the rise-time of the LDO.\hspace*{\fill}	
	
	\noindent``When not in a state of short circuit, if the terminal voltage is at 90\% of its rated value, then for some time between 0.4$\mu s$ to 0.63$\mu s$ it will not reach the rated value ({\tt\small VStable}: {\tt\small 1.5<=v<=1.85})."
	\begin{mycenter}[0em]
	{\tt \small !InShrtCkt \&\& VUpperBand |-> \#\#[ 4.0e-7:6.3e-7 ] !VStable}
	\end{mycenter} 
	\hspace*{\fill}(Support: 1\%, Correlation: 0.02\%)\\
	%\noindent``When the LDO's terminal voltage is at its rated value and a short circuit event takes place within 20$\mu s$, then the circuit current must be above the short circuit threshold of 8.5mA."	
	%\begin{mycenter}[0em]
	%	{\tt \small VStable \#\#[0:2.0e-05 ] ShrtCktEvent |-> InShrtCkt}
	%\end{mycenter}
	%\hspace*{\fill}(Support: 1.08\%, Correlation: 4.76\%)\\	
	\noindent``When the short circuit event ({\tt\small ShrtCktEvent}: {\tt\small 0.0085<=i<=0.009}) isn't in play, but the circuit is in a state of short circuit (current is above the 8.5mA threshold), then sometime in the next 20$\mu s$ the terminal voltage is not at its rated value."
	\begin{mycenter}[0em]
	\noindent {\tt \small InShrtCkt \&\& !ShrtCktEvent |-> \#\#[0.0:2.0e-05] !VStable}
	\end{mycenter} 
	\hspace*{\fill}(Support: 23.19\%, Correlation: 82.29\%)\\
\qed
\end{example}

\begin{example}\label{ex:ldo2}
A comparator suffers from a glitch that occurs on a digital port. We use  PSI-Miner to mine correlations that may exist, looking for glitches that may have occurred on other lines within a short window of time. It is likely that if these occurred before the observed glitch on the digital port, they could have carried a disturbance across. {\tt PSI-Miner} is provided with  a list of Boolean expressions over a predicate alphabet (Table~\ref{table:predAlphabetComp}) and asked to explain the expression {\tt \small nrst1V8Band}. The alias {\tt \small nrst1V8Band} represents valid voltages for the digital port which may have been violated.

\begin{table}[th]
	\centering
	\footnotesize
	\begin{tabular}{cc}
		\toprule
		\textbf{Predicate/Boolean Expression} & \textbf{Alias}	\\
		\midrule
		%{\tt v >= 1.5 \&\& v <= 1.85} & {\tt VStable}\\
		%\tt v >= 0.18 \&\& v <= 0.19} & {\tt VLowerBand}\\
		%{\tt v >= 1.62 \&\& v <= 1.63} & {\tt VUpperBand}\\
		%{\tt i >= 0.0085} & {\tt IShrtCkt}\\
		%{\tt i >= 0.0085 \&\& i <= 0.009} & {\tt ShrtCktEvent} \\
		%\midrule
		{\tt nrst1V8 <= 0.05 || nrst1V8 >= 1.5 } & {\tt nrst1V8Band}\\
		{\tt uvlo<=0.0041 || uvlo>=0.740} & {\tt uvloBand}\\
		{\tt ovlo<=0.0041 || ovlo>=1.5} & {\tt ovloBand}\\	
		{\tt enOsc <= 0.019} & {\tt enOscBand }\\		
		{\tt vreflow <=0.0270} & {\tt vreflowBand }\\	
		{\tt vrefhigh >=1.21} & {\tt vrefhighBand}\\		
		{\tt supptri<=0.00499} & {\tt suppBand }\\		
		{\tt clkrun<=-0.04} & {\tt clkBand }\\
		\bottomrule
	\end{tabular}
	\caption{Predicate Alphabet used for PSI-Miner for the Comparator}\label{table:predAlphabetComp}
\end{table}

\begin{figure}
	\centering
	\includegraphics[width=\textwidth]{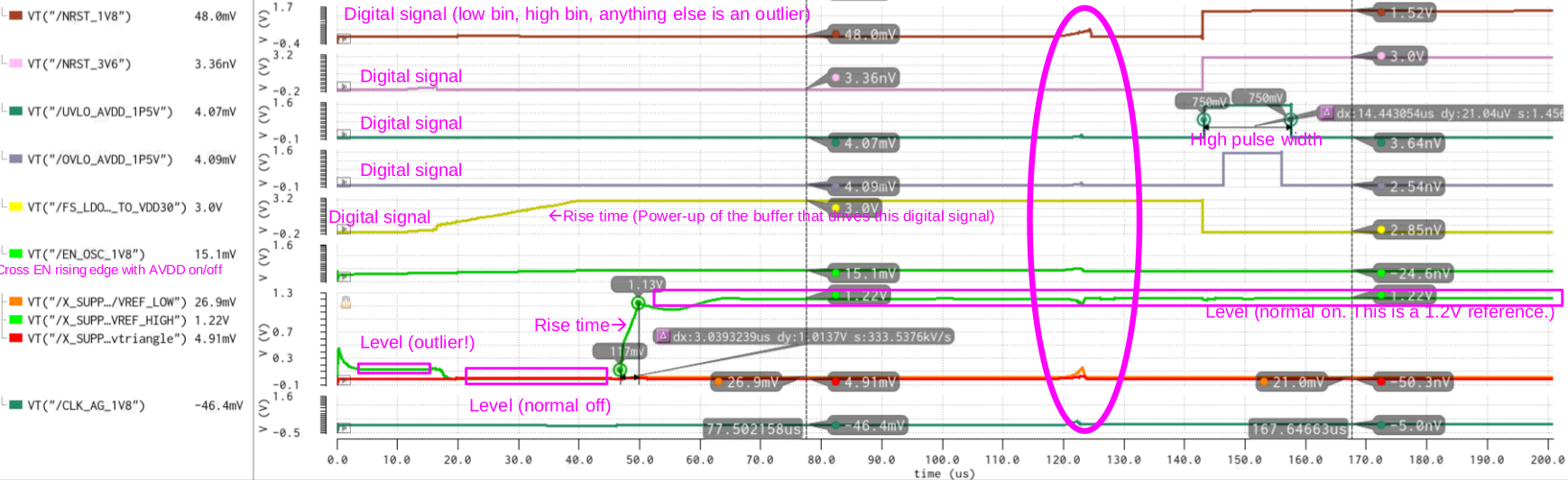}
	\caption{Glitch behaviour in an Comparator Circuit}\label{fig:glitchyLDO}
\end{figure}

\noindent``When the clock level and the UVLO are not within acceptable thresholds, and if within 6$\mu s$ the OVLO is also outside acceptable thresholds, then within 9$\mu s$ the {\tt nstr1V8Band} also falls outside acceptable thresholds."
\begin{mycenter}[0em]
{\tt \small!clkBand \&\& !uvloBand \#\#[0.0:6.0e-06] !ovloBand 
		|-> \#\#[0.0:9.0e-06] !nrst1V8Band}
\end{mycenter}
\hspace*{\fill}(Support: 3.01\%, Correlation: 24.95\%)\\
\noindent``When the clock and oscillator enable are not within acceptable thresholds and the UVLO is within its acceptable threshold, within the next 4.8$\mu s$  the {\tt nstr1V8Band} falls outside acceptable thresholds."
\begin{mycenter}[0em]
{\tt \small !clkBand \&\& uvloBand \&\& !enOscBand 
		|-> \#\# [0.0:4.80e-06] !nrst1V8Band
}
\end{mycenter}
\hspace*{\fill}(Support: 0.14\%, Correlation: 12.85\%)\\
\noindent The prefix sequences mined help identify that it is possible that the when the {\tt clkBand} expression becomes false (there is jitter on the clock port), this leads to a cascading of jitters on other ports. In one case, the UVLO and OVLO both fall out of acceptable bands of operation, while in another the oscillator enable falls outside its acceptable band of operation. Such information may be useful in a root cause analysis and for monitoring for similar anomalies. 
\qed
\end{example}
}
\change{PSI-Miner was able to compute the decision tree in under a second in our experiments. The computation time varies with the size of the interval set. The trend for CPU-Time as $n$ and $k$ vary is depicted in Figure~\ref{fig:cpu-time}. For every increasing incremental change in $n$, the number of pseudo-targets increases accordingly. This then increases the number of pseudo-targets against which unified gain is computed for every predicate. For a fixed predicate alphabet $\mathbb{P}$, for an increase in $n$ by one, a fixed number, $|\mathbb{P}|$, of new computations of unified gain are introduced. 
To demonstrate this, we use the data-set of the LDO of Example~\ref{ex:ldo1}. We choose this example since this analog circuit has behaviours that occur in a scale of micro-seconds, while the trace itself is $5ms$ long, orders of magnitude longer than mined behaviours. We also use predicates having small truth intervals (a few micro-seconds long). In Figure~\ref{fig:cpu-time}, we record the CPU-Time to process the trace and generate the pseudo-targets (Input Processing), and to generate the decision tree (Tree Generation). Other constraints on the decision tree are the same as in Example~\ref{ex:ldo1}.

We first vary $n$ in increments of 10, from 10 to 100, using a fixed delay resolution of $10^{-6}$ and use all the predicates from Example~\ref{ex:ldo1} as part of the predicate alphabet. We use the predicate {\tt VUpperBand} as the target. We observe that as $n$ increases, the time to generate pseudo-targets varies around only marginally around $1.5ms$. On the other hand, we see a clear trend in the time to generate the decision tree. In general, we observe that the CPU-Time increases linearly with $n$. 

Next, we vary $k$ in multiples of 2, starting with $k=10^{-6}$, while using a fixed prefix length of $n=10$, and all the predicates from Example~\ref{ex:ldo1}. The target used is the same as earlier. We observe that as $k$ increases, the time to generate pseudo-targets varies marginally, around $1.5ms$, while in general the CPU-Time decreases with an increase in $k$. There seems to be no clear relationship between $k$ and CPU-time, beyond what is already stated. The decreasing trend of CPU-time with increasing values of $k$ is explained by the fact that although $k$ increases, the number of pseudo-targets remains constant. However, the number of truth intervals for the $i^{th}$ pseudo target may decrease. This is owed to the use of Minkowski difference with larger intervals, leading to the merging of truth intervals for the target. A decrease in the size of the interval set for a pseudo-target results in fewer set operations, resulting in a decrease in CPU-Time.

\begin{figure}[t]
        \begin{subfigure}[b]{0.49\textwidth}
            \includegraphics[width=0.96\linewidth]{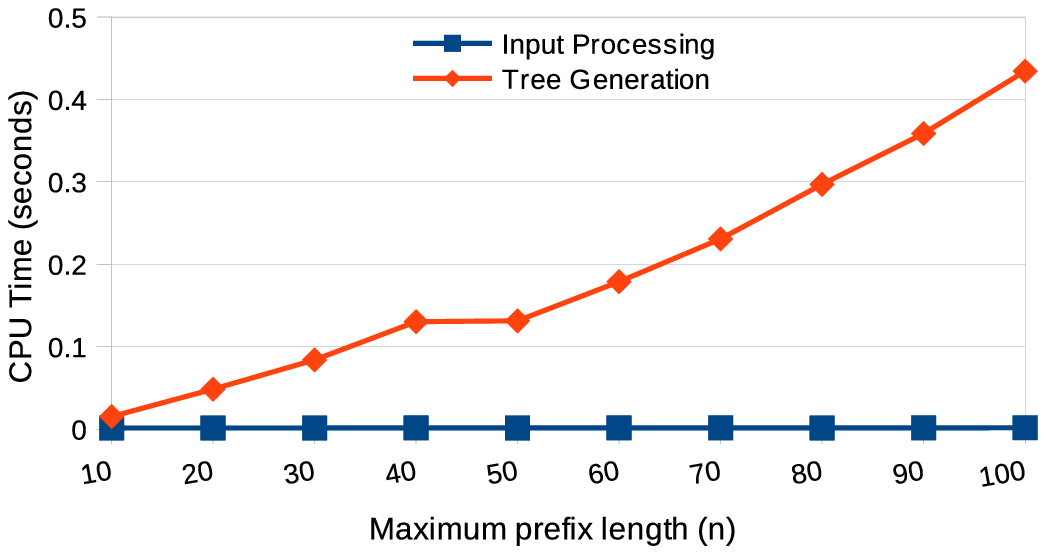}
            \caption{CPU-Time versus prefix length (n)}\label{fig:varying-n}
        \end{subfigure}%
        ~~
        \begin{subfigure}[b]{0.49\textwidth}
            \centering
            \includegraphics[width=0.96\linewidth]{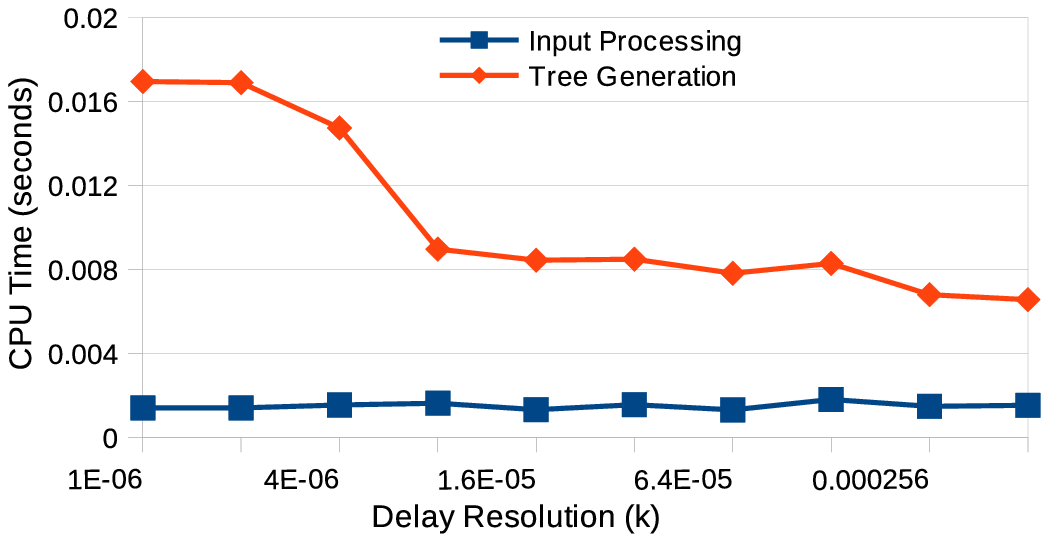}
        	\caption{CPU-Time versus delay resolution (k)}\label{fig:varying-k}
        \end{subfigure}%
        \caption{CPU-Time for input processing and decision tree generation as $n$ and $k$ vary}\label{fig:cpu-time}
\end{figure}

}

\change{
\section{Related Work}\label{sec:related}
Mining information from time-series data has been a topic of study for 
decades~\cite{Esling2012,Ralanamahatana2005}. Learning problems include but are not limited to the following:

% \subsection{Types of Time-Series Learning Problems}
% Given time-series $\mathcal{S}$ and $G$, and dissimilarity metric $D(\mathcal{S},G)$, existing methods for learning from time-series data can be broadly classified into the following types:

\begin{enumerate}

	\item \textbf{Querying}: Finding a time-series similar to a given one from a database. Methods include Singular Value Decomposition (SVD), Discrete Fourier transform (DFT), Discrete Wavelet Transforms (DWT) or Adaptive Piecewise Constant Approximations (APCA) and the dissimilarity metric to index the 
	time-series~\cite{Chakrabarti2002,Faloutsos1994}.
	
	\item \textbf{Clustering}: Clustering time-series data using a similarity metric between time-series. For instance, in Debregeas \& Hebrail, 1998, 
	%~\cite{Debregeas1998} 
	time-series recorded from the electrical power-grid are clustered together using Kohonen Maps, while in 
	Kalpakis et al., 2001,
    %~\cite{Kalpakis2001} 
	distance measures are explored for clustering\label{type:clustering}.
	
	\item \textbf{Summarization}: Summarize a time-series with a representative approximation~\cite{Indyk2000}. Similarly, Van Wijk  \&  Van  Selow,  1999,  %~\cite{VanWijk1999}
	discusses other methods for visualizing trends in a univariate time-series dataset.

	\item \textbf{Separation Features}: Given time-series $\mathcal{S}$ and $\hat{\mathcal{S}}$, find interesting features that separate the two time-series. In Guralnik  \&  Srivastava,  1999, 
	%~\cite{Guralnik1999} 
	time-series data is analyzed to determine \textit{interesting episodes} in the time-series. In Keogh et al., 2002,%~\cite{Keogh2002} 
	the authors suggest discretizing the time-series and finding sub-strings that occur most frequently in the time-series.
	
	\item \textbf{Prediction:} Given a time-series $\mathcal{S}$ over time points $(t_1,...,t_n)$, predict the behaviour of $\mathcal{S}$ as if observed over time points $(t_{n+1},...,t_{n+k})$. Studies in this area have been reported in Brockwell \& Davis, 2002; Harris \& Sollis, 2003; Tsay, 2005; and Brockwell \& Davis, 1986.%~\cite{BrockwellDavis,HarrisSollis,Tsay,Brockwell1986}.
	
	\item \textbf{Anomaly detection:} The problem of detecting a pattern that deviates from a nominal behaviour is strongly linked to the problem of prediction. Like prediction, it also relies on having a sufficiently accurate model of the time-series to be able to identify deviations~\cite{Ypma97noveltydetection,Zhong2007,Ma2003}.
	
	\item \textbf{Motif Discovery:} A sub-sequence that is observed frequently in a time-series is called a \textit{Motif}. A detailed review of existing literature in this area can be found in Esling \& Agon, 2012.%~\cite{Esling2012}.
	
\end{enumerate}
These approaches do not address the problem considered in this paper, namely to find causal sequences of predicates that explain a given consequent. More recently, the focus of learning has been to learn artifacts about time-series that can be expressed in logic and are therefore inherently explainable. These studies are broadly broken down into the following two types:
\begin{itemize}

	\item \textit{Learning properties from templates}: A template property is provided. The learning problem is to choose property parameter values that best represent the time-series.

	\item \textit{Learning property structure and parameter values}: Given a syntax for the property, learn the property that best represents the time-series.

\end{itemize}
We summarize related work in these problems and position our work against them.

\subsection{Learning from Templates}\label{subsec:templateLearning}

A large repository of work exists on mining parameter values for a template property in parametric STL (PSTL)~\cite{Seshia2015,Maler11,Maler18,Fainekos12} with the aim of optimizing property robustness for the given time-series trace. 
%While a property is either satisfied or not satisfied by a time-series, property robusteness~\cite{Donze2010} quantitatively measures how well the time-series satisfies the property.
The work in Asarin et al., 2012, %~\cite{Maler11} 
proposes learning the range of valid parameter values for a PSTL property that a given set of dense-time real-valued system traces satisfy. Given a formula in PSTL, the authors of Asarin et al., 2012, and Bakhirkin et al., 2018, %\cite{Maler11,Maler18}
propose techniques to compute a validity domain for the formula's time and value parameters such that all traces satisfy the formula given these domains. In Yang et al., 2012, %~\cite{Fainekos12}, 
the authors propose a methodology to compute parameter domains for a property in MTL that a given embedded and hybrid system satisfies. 
%The system is modeled in MATLAB and the authors use their property falsification tool S-TALIRO~\cite{staliro} to compute the set of parameter values that robustly satisfy the parameterized MTL property.

In Jin et al., 2015, %~\cite{Seshia2015}, 
the authors propose learning parameter values that satisfy system requirements expressed as template properties in PSTL. 
%They use the framework BREACH~\cite{breach} to compute falsifying traces for a concrete choice of parameter values for the property. 
They iterate on the domain of parameter values until they converge on a combination of values that make the property valid for the given set of traces. 
% The work in~\cite{Seshia2015} requires system traces along with a  model of the system that can be used with their falsification tool.
%The authors of~\cite{Maler18} improved on the work in~\cite{Maler11} by proposing a new methodology to compute the validity domains of parameters of a given PSTL template property by computing bottom-up satisfaction and robustness signals, and by propagating them as a function of time from sub-formulas to formulas.

\change{
Methods for template-based learning require a parameterized property expressed in a formal logic such as MTL or STL. This means that the predicates that influence a given consequent are known, and the parameter values are learned. The learning problem is thereby transformed into a parameter optimization problem. Some of these works also assume the existence of a model of the system.}

\change{
These methods are not applicable to cases where nothing is known about the factors that influence the truth of a given consequent, yet we wish to find the causal sequence of events. Our contribution is in providing a methodology that learns the timed sequence of  predicates that cause the consequent, which involves learning the relevant predicates, as well as the real-time timing between them.
} 

\subsection{Learning Property Structure and Parameter Values}\label{subsec:twoClassClassification}

While in Section~\ref{subsec:templateLearning} a formula structure was provided as input to the learning task, we now summarize property mining studies in which such templates are not-provided.

\change{
The work on Temporal Logic Inference (TLI) in Kong et al., 2014, %~\cite{Calin14} 
aims to classify two labelled sets of time-series by learning distinguishing Boolean combinations of temporal properties of the form $F_{[t_1,t_2)} \psi$ or $G_{[t_1,t_2]} \psi$, where $\psi$ is Boolean\footnote{$F$ and $G$ are respectively the standard \textit{future} and \textit{global} operators of temporal logic}. The work in Bombara et al., 2016, %~\cite{Calin16} 
improves the methodology in Kong et al., 2014, %~\cite{Calin14}
to allow for multiple predicates in the antecedent and properties of the form $F_{[t_1,t_2)} \varphi_g$ or $G_{[t_1,t_2]} \varphi_f$. The predicate and timing constants are learned using local search heuristic algorithms like simulated annealing, while the property structure is learned using standard decision trees. As we have shown in our work, using standard decision trees for learning temporal logic properties can lead to misleading outcomes. Moreover, the methodologies of Kong et al., 2014 and Bombara et al., 2016, %~\cite{Calin14,Calin16} 
do not address our problem of finding a causal sequence for a given consequent.} \newchange{One straightforward way of adapting their work %of~\cite{Calin14,Calin16} 
is to perform a  splitting of traces into two classes of sub-traces of length $n\times k$, one class containing $E$ and the other containing $\neg E$. Note that in both types of traces, it is possible that there exist both, time-points where $E$ is true and others where $\neg E$ is true. Hence, in both sets there would be similar event sequences, with similar delays between events, associated differently with $E$ and $\neg E$. This could result in an empty or misleading outcome. On the other hand, the work proposed here could be adapted to work with classification problems. One way to achieve this is to introduce an variable denoting class type with appropriate values at all time-points in all traces. For instance, consider trace sets $\mathbb{T}_A$ and $\mathbb{T}_B$ containing traces of two classes $A$ and $B$. Introduce a new variable $"class"$, with $class==A$ true at all time-points of traces of class $A$, and false for all time-points of traces of class $B$. We now use the proposed PSI-Miner methodology to mine properties with $E\equiv class==A$ as the target. We successfully used this technique in Example~\ref{ex:disease} to identify key differences between passenger movements labeled {\tt RISKY} and {\tt !NON-RISKY}.}

% In~\cite{Calin14} a property is learned using a directed search over property structures that have a qualilative (language inclusion) and quantitative (robustness~\cite{Donze2010}) partial order among them. 
% The authors use a temporal logic syntax that restricts the types of STL formulae learned to a Boolean combination of short STL primitives, $\varphi_f$ and $\varphi_g$ which are respectively $F_{[t_1,t_2)} \psi$ or $G_{[t_1,t_2]} \psi$, where $F$ and $G$ are respectively the standard \textit{future} and \textit{global} operators of STL; and $t_1$, $t_2$ are time parameters; and $\psi$ is a predicate of the form $x \circ c$, $\circ \in \{\leq,>\}$. 

%The work in ~\cite{Calin16} extends the language and improves on the methodology in~\cite{Calin14} to allow for primitives  $F_{[t_1,t_2)} \varphi_g$ or $G_{[t_1,t_2]} \varphi_f$. The predicate and timing constants are learned using local search heuristic algorithms like simulated annealing, while the property structure is learned using standard decision trees. The property is a mapping of the decision tree into a fragment of STL. As we show later in our work, using standard decision trees for learning temporal logic properties can lead to misleading outcomes.
%Both~\cite{Calin14} and~\cite{Calin16} use the STL robustness measure as the function to optimize the structure and parameters of the property mined. 

The work in Bartocci et al., 2014, %~\cite{Bartocci14}
learns a discriminator property to distinguish between traces generated by two different processes. The method relies on using a statistical abstraction of the data in the traces. 
%and thereafter uses a two-stage approach for identifying a discriminator. First, an optimal formula structure is learned from a library of template structures. This is then followed by a tuning of parameters in the formula so as to maximize its discrimination power. In their methodology 
The property structure and parameters are optimized separately.

%The strategies for mining patterns from system representations can be classified primarily as being analyses for Boolean (digital) or non-Boolean systems. We consider the space of non-Boolean systems to include both programs and hybrid systems. While programs are characterized by discrete behaviours, hybrid systems interleave discrete behaviours with continuous behaviours as the continuous evolution of its real-valued variables.

Past work in Yang \& Evans, 2004; Yang et al., 2006; Gabel \& Su, 2008a, 2008b; and  De-Orio et al., 2009,  %~\cite{Yang2004,PerracottaYang2006,Javert2008,Gabel2008,InfernoDeOrio2009} 
focuses on mining sequences and causal relations for program events. The most recent study examines program traces and learns a finite automaton, over a user-defined alphabet, describing all the ways (up to a discrete event bound on path length) of violating a propositional assertion given in the program~\cite{ChocklerKKS20}. Older studies focus on mining cause-effect relations in programs as LTL properties~\cite{Chang2010,Danese2015-2,Danese2015,Lemieux2015}. The methodology in Cutulencoet al.,  2016, %~\cite{Cutulenco2016} 
mines timed regular expressions from program traces, while in Garg et al.,2016, %~\cite{Garg2016} 
decision trees are used to learn invariants for software programs. In these works, the sequences mined do not preserve timing information between the events. 

%\subsection{Learning Allen and Linear Temporal Logic Patterns}
The tool Goldmine~\cite{goldmine} uses decision trees to mine causal relations from clocked traces of Boolean systems as an ordering of events. The assertions mined are in a subset of LTL limited to bounded safety and liveness. The work in Kauffman \& Fis-chmeister, 2017,  %~\cite{Fischmeister2017} 
mines Allen's interval relations, specifically event intervals from clocked event traces. The proposed methodology mines {\tt nfer} rules (based on Allen's Temporal Logic) from learned \textit{before} relations given a set of clocked event traces. The learned sequences are a series of before relations between events in the traces. These techniques are not applicable to real-time data from dense-time systems.
}

% For instance, from trace data on the NASA's Mars rover, for a single rover activity command, the miner learned the relation \textit{dispatch} \textbf{before} \textit{complete}, which indicates that relation that a command is first dispatched before it is completed.

\change{
To the best of our knowledge, ours is the first work on mining sequences consisting of Boolean combinations of predicates over real variables separated by dense real time delay intervals that causally determine the truth of a given consequent. Our methodology automatically finds the predicate combinations that influence the consequent as well as the real time delay intervals that separate them. Although we start with a property skeleton consisting of bucket positions and template delay intervals, the learning methodology automatically fills the buckets by choosing relevant predicates, merges empty buckets, and tightens the delays between the buckets to return dense time temporal properties for explaining the truth of the consequent.
}

\section{Conclusions}\label{sec:summary}

\change{
In complex dynamical systems, the task of finding the causal trigger of an event is an important and non-trivial task. The AI/ML community seeks data-driven solutions for such problems. Our approach of mining prefix sequences addresses this task. Prefix sequences expressed in logic are easily readable and are easy to explain.

Properties mined using our methodology are useful in many different contexts. This includes the following:

\begin{enumerate}

\item {\em Anomaly Detection.} Anomalies may be viewed as deviations from set patterns in the data. If a property mined from legacy data fails during the execution of the system, then it may be an anomaly. Properties can be readily monitored over simulation and real-time execution, and can thereby be used to detect anomalies.

\item{\em Prediction.} The mined properties can be monitored at runtime to predict future events. For example, when the prefix-sequence at the antecedent of a property matches the runtime behavior, the consequent can be predicted within the corresponding delay interval.

\item{\em Clustering.} Time-series data can be partitioned into clusters on the basis of the truth of the mined properties. Also, for a given consequent we may have mined different prefix-sequences. Clustering the data based on the matches of the prefix-sequence help us to separate out data corresponding to different causes that lead to the same outcome.

\end{enumerate}

There are several interesting offshoots from our work. For example:
\begin{itemize}

\item {\em Incorporating domain knowledge.} If we already know some properties over the variables in the system, then the mining algorithm can be suitably modified to ensure that the mined properties do not contradict the domain knowledge. This is particularly necessary for safety-critical systems, where corner case safety properties are not well represented in the data.

\item {\em Finding separation features.} Separation features are properties that explain the difference between two sets of time-series data sets. To facilitate the readability of the differences, we need to mine properties over similar predicates and having similar structure. This requires considerable modification in the mining algorithm and is one of the future directions being pursued by us.

\item {\em Mining properties with recurrent behaviors.} Suppose a consequent event, $E$, is caused when a predicate $P$ remains true for more than 20 seconds. Such a property cannot be mined using the present approach because the recurrent requirement for predicate $P$ cannot be captured in the present language. This is also an interesting direction of our future research.

\end{itemize}
We believe that there are many other possible directions of research based on the contributions of this paper. The increasing significance of finding causal sequences in data driven learning approaches adds to the impact potential of the methods presented in this paper.
}

\section*{Acknowledgements}
The authors thank Intel corporation CAD SRS funding for partial support of this research.
\bibliographystyle{theapa}
\bibliography{psi-miner}
\end{document}